\documentclass[12pt]{article}
\usepackage{enumerate}
\usepackage[OT1]{fontenc}
\usepackage{amsmath,amssymb}
\usepackage{natbib}
\usepackage[usenames]{color}

\usepackage{smile}
\usepackage{natbib}
\usepackage{multirow}
\usepackage{appendix}

\usepackage[breaklinks,colorlinks,
            linkcolor=red,
            anchorcolor=blue,
            citecolor=blue
            ]{hyperref}

\def\given{\,|\,}

\def\tr{\mathop{\text{tr}}\kern.2ex}

\def\F{\text{F}}

\long\def\comment#1{}

\def\vec{\mathop{\text{vec}}}
\def\tr{\mathop{\text{Tr}}}

\def\cS{{\mathcal{S}}}

\def\cP{{\mathcal{P}}}
\def\cL{{\mathcal{L}}}

\def\cT{{\mathcal{T}}}
\def\cB{{\mathcal{B}}}
\def\cO{\mathcal{O}}
\def\tr{{\text{Tr}}}

\newcommand{\bel}{\begin{eqnarray}\label}
\newcommand{\eel}{\end{eqnarray}}
\newcommand{\bes}{\begin{eqnarray*}}
\newcommand{\ees}{\end{eqnarray*}}

\newcommand{\la}{\langle}
\newcommand{\ra}{\rangle}

\usepackage{mathrsfs}

\usepackage{fullpage}
\def\sep{\;\big|\;}

\usepackage{breakcites}
\def\##1\#{\begin{align}#1\end{align}}
\def\$#1\${\begin{align*}#1\end{align*}}
\usepackage{enumitem}
\begin{document}

\title{\LARGE
	Risk-Sensitive Deep RL: 
	Variance-Constrained Actor-Critic Provably  Finds Globally Optimal Policy}
	\author{Han Zhong\thanks{Peking University. Email: \texttt{hanzhong@stu.pku.edu.cn}} \qquad Xun Deng\thanks{University of Science and Technology of China. Email: \texttt{dx981228@mail.ustc.edu.cn}} \qquad Ethan X. Fang\thanks{Duke University. Email: \texttt{ethan.fang@duke.edu}}  \qquad Zhuoran Yang\thanks{Yale University. Email: \texttt{zhuoran.yang@yale.edu}}\\ Zhaoran Wang\thanks{Northwestern University. Email: \texttt{zhaoranwang@gmail.com}} \qquad  Runze Li\thanks{Pennsylvania State University. Email: \texttt{rzli@psu.edu}}}

\date{}

%\date{February, 2020}

\maketitle

%%%%%%%%%%%%%

%!TEX root =main.tex

\begin{abstract}
While deep reinforcement learning has achieved tremendous successes in various applications, most existing works only focus on maximizing the expected value of total return and thus ignore its inherent stochasticity. Such stochasticity is also known as the aleatoric uncertainty and is closely related to the notion of risk. In this work, we make the first attempt to study risk-sensitive deep reinforcement learning under the average reward setting with the variance risk criteria. In particular, we focus on a variance-constrained policy optimization problem where the goal is to find a policy that maximizes the expected value of the long-run average reward, subject to a constraint that the long-run variance of the average reward is upper bounded by a threshold. Utilizing Lagrangian and Fenchel dualities, we transform the original problem into an unconstrained saddle-point policy optimization problem, and propose an actor-critic algorithm that iteratively and efficiently updates the policy, the Lagrange multiplier, and the Fenchel dual variable. When both the value and policy functions are represented by multi-layer overparameterized neural networks, we prove that our actor-critic algorithm generates a sequence of policies that finds a globally optimal policy at a sublinear rate. Further, We provide numerical studies of the proposed method using two real datasets to back up the theoretical results.
  
 %by imposing a variance constraint to upper bound the variance of the  total return obtained by following a stationary policy. Such a risk-sensitive/variance-constrained approach finds many vital applications in practice. However, it is substantially more challenging than the unconstrained version. By leveraging modern optimization techniques, we provide a new framework to conduct variance-constrained deep reinforcement learning. Despite the nonconvexity  and the high dimensionality, we prove that under mild assumptions, our algorithm achieves the global optimality that the duality gap diminishes at an $\cO(1/\sqrt{K})$ rate, where $K$ is the iteration counter.
\end{abstract}

\section{Introduction}

Reinforcement learning (RL) is a powerful approach to solving multi-stage decision-making problems by interacting with the environment and learning from experiences. 
Thanks to the practical 
efficacy 
of reinforcement learning, it draws substantial attentions from different communities such as operations research \citep{bertsekas1996neuro,mertikopoulos2016learning,wen2017efficient,wang2017stochastic,zeng2018asyncqvi}, computer science \citep{sutton1998introduction} and statistics \citep{menictas2019artificial,clifton2020q}. 
With the advance of deep learning, over the past few years, we have witnessed phenomenal successes of deep reinforcement learning (DRL) in solving extremely challenging problems such as Go \citep{silver2016mastering, silver2017mastering, openai2019}, robotics \citep{kober2013reinforcement, gu2017deep}, and natural language processing \citep{narasimhan2015language}, which were once  regarded too complicated to be solvable by computer programs in the past. 

Despite these empirical successes, 
providing 
theoretical justifications for deep reinforcement learning 
is  rather challenging. 
A  significant challenge  is that 
the optimization problems associated with deep reinforcement learning are usually highly nonconvex, which is due to a combination of  the following two sources. 
 
 First, under the risk-neutral setting, where the goal is to find a policy that  maximizes the (long-run) average reward in expectation within a parametric policy class, 
 the optimization objective is 
 a nonconvex function of the policy parameter. 
 This is true even when the policy admits a tabular or linear parameterization, and the global convergence and optimality of policy optimization algorithms 
 for these cases are only established recently. See, e.g., \cite{agarwal2020optimality,shani2020adaptive, mei2020global, cen2020fast} and the references therein.
 
 Second, when the policy is represented  by a  deep neural network, due to its nonlinearity and complicated structure,   policy optimization  is significantly more challenging. 
Theoretical  guarantees for    deep policy optimization is rather limited.  
 Recently, built upon the theory of neural tangent kernel \citep{jacot2018neural}, 
 \cite{liu2019neural, wang2019neural, fu2020single} prove that various actor-critic algorithms with overparameterized neural networks provably achieve global convergence and optimality.
  
% 
% Earlier works such as trust region policy optimization \citep{schulman2015trust} and proximal policy optimization \citep{schulman2017proximal} are among the earliest and most empirically successful algorithms for solving such problems. However, these algorithms do not guarantee convergence to the globally optimal solution.  Recently, \cite{liu2019neural} show that a variant of these two algorithms converges to the global solution, and theoretically justify the success of deep reinforcement learning in practice.

In this paper,
going beyond the risk-neutral setting, 
 we 
make the first  attempt to study risk-sensitive deep reinforcement learning.
In particular, 
we focus on the variance risk measure 	
\citep{sobel1982variance}
  and aim to  find a neural network  policy that 	 maximizes the expected value of the  long-run average reward under the constraint that   the variance of the long-run average reward  is upper bounded by a certain threshold. 
  Here the variance constraint   incorporates the risk-sensitivity --- the reinforcement learning agent is willing to achieve a possibly smaller expected reward in exchange for a smaller variance. 
 Moreover, such a problem is substantially more challenging than the risk-neutral setting   and   finds important applications. 
 Our goal is to establish an algorithm that provably finds a globally optimal solution to such a risk-sensitive policy optimization problem within the class of deep neural network policies. 
 To the best of our knowledge, this problem has never been considered in existing deep reinforcement learning literature.

%To the best of our knowledge, the notion of risk/aleatoric uncertainty is largely unexplored in the study of  deep reinforcement learning.  To make the first attempt for risk-sensitive deep reinforcement learning, we propose to conduct variance-constrained deep reinforcement learning. In particular, we maximize the expected total reward under the constraint that the variance of total reward is upper bounded. 

\subsection{Motivating Applications}
Imposing the variance constraint is of substantial practical interests. We provide two concrete motivating applications. The first application is in portfolio management. Reinforcement learning/deep reinforcement learning methods have been applied for portfolio optimization \citep{moody1998performance,jiang2017cryptocurrency}, where we dynamically allocate the assets to maximize the total return over time. In such applications, while optimizing the expected total return, it is important to control the volatility/risks of the portfolio. In the celebrated Markowitz model \citep{markowitz1952portfolio},  it is suggested that the risk of a portfolio is based on the variability/variance of returns, and the model is exactly maximizing the expected total return for a given level of the variance of total return.

%
%assets to maximize the total return over time. In such applications, while optimizing the expected total return, it is important to control the portfolio's volatility/risks. In the celebrated Markowitz model \citep{markowitz1952portfolio}, it is suggested that the risk of a portfolio is based on the variability/variance of returns, and the model is to maximize the expected total return for a given level of the variance of total return. 

The second example is in robotics. 
%Deep reinforcement learning achieves phenomenal successes in training robotics \citep{gu2017deep,tai2017virtual}. 
It is known that one of the emerging and promising applications of robotics is senior care/medicine \citep{kohlbacher2015leading,taylor2016medical,tan2020governing}. In these applications, while achieving the maximum expected return, it is extremely important to control the variability of the outcome, as a little change in robotics' operation could lead to devastating outcomes. 
While deep reinforcement learning has achieved phenomenal successes in training robotics \citep{gu2017deep,tai2017virtual} under the risk-neutral setting, 
this example shows that risk-sensitive/variance-constrained deep reinforcement indeed calls for a principled solution.

\subsection{Major Contribution}

Incorporating  a  variance constraint into deep reinforcement learning raises several challenges.
First, this makes the optimization problem a constrained one.
Although there are various algorithms designed for constrained Markov decision process (CMDP)  \citep{altman1999constrained}, 
these methods cannot be directly applied to our variance-constrained problem. 
In particular, 
the constraint in  CMDP is  irrelevant to the reward function in the objective, 
whereas the constraint in our problem is the variance of the long-run average reward.
Thus,  handling such a constraint requires   new algorithms. Second, as we employ deep neural network
policies,  both the  expected value and variance of the long-run average reward  are 
 highly nonconvex functions of the policy parameter.  
 Third, 
 to obtain  the 
 policy update directions, 
 we need to characterize the landscape of the variance of average reward as a functional of the policy.   
  As discussed in \cite{tamar2016learning},
  due to the nonlinearity of  the variance of a random variable in the probability space, 
    this raises a  substantial challenge  even in the simpler linear setting.

 To tackle these challenges, 
 inspired by the celebrated   actor-critic  framework  \citep{konda2000actor}, we  propose a \underline{var}iance-constrained \underline{a}ctor-\underline{c}ritic (VARAC) algorithm, 
 where both the policy (actor) and the value functions (critic) are represented by multi-layer overparameterized neural networks. 
 In specific, to handle  the first challenge, 
 we  transform the constrained problem into  an  unconstrained saddle point problem   via Lagrangian duality. 
 Then, to cope with the third challenge, leveraging Fenchel duality, we further write the variance 
 into the  variational form by introducing a dual variable.
 Thus, the original problem is transformed into a saddle point  problem involving the policy $\pi$, Lagrange multiplier $\lambda$, and dual variable $y$. 
 More importantly, when $\lambda$ and $y$ are fixed, the objective is equal to the long-run average of a transformed  reward, and thus we can characterize its landscape for policy optimization. 
For such a saddle point problem, VARAC updates 
$\pi$, $\lambda$, and $y$ via first-order optimization. 
Specifically, in each iteration, we update the policy $\pi$ via proximal update with the Kullback-Leibler (KL) divergence serving as the Bregman divergence, while $\lambda$ is updated via (projected) gradient method, and $y$-update step admits a closed-form solution.
Moreover, 
the update directions are all based on the  solution to  
the inner problem of the critic, which corresponds to solving two policy evaluation problems determined by current $\lambda$ and $y$ via temporal-difference learning \citep{sutton1988learning} with deep neural networks. 
Our KL-divergence regularized policy update is closely related to the trust-region policy optimization \citep{schulman2015trust} and proximal policy optimization
 \citep{schulman2017proximal}, which have demonstrated great empirical successes.
 Finally, to tackle the second challenge, 
 from a functional perspective, 
 we view the policy update  of VARAC 
   as an instantiation of  infinite-dimensional mirror descent   \citep{beck2003mirror,zhang2018convergence}, 
   which is well approximated by the parameter update of the policy when the neural network is overparameterized. 
   Thus, 
      we show that under mild assumptions, despite  nonconvexity, 
      the policy sequence obtained by VARAC 
      converges to a globally optimal policy  at a sublinear $\cO(1/\sqrt{K})$ rate, where $K$ is the iteration counter. 

In summary, our contribution is two-fold. First, to our best knowledge, we make the first attempt to study risk-sensitive deep reinforcement learning by imposing a  variance-based risk constraint. 
Second, we propose a novel actor-critic  algorithm, dubbed as VARAC,  which provably finds a globally optimal policy of  the variance constrained problem at a sublinear rate.   
 We believe that our work brings a promising future research direction for both optimization and machine learning communities. 

\subsection{Related Work}

Our work extends the field of risk-sensitive optimization. The risks are essentially some measures of the aleatoric uncertainty. In nature, there are two types of uncertainty \citep{clements2019estimating}. The first type is {\it epistemic} uncertainty, which refers to the uncertainty caused by a lack of knowledge
and    can be reduced by acquiring more data.
The second type is {\it aleatoric} uncertainty, which  refers to the notion of inherent randomness.  
That is, 
the uncertainty  due to the stochastic nature of the  environment,   which cannot be reduced even with unlimited data. 
 Optimizing returns while controlling the risk is of great practical importance. 
 Various risk measures are proposed for different applications, which  include variance \citep{rubinstein1973mean}, value at risk (VaR) \citep{pflug2000some},  conditional value at risk (CVaR) \citep{rockafellar2000optimization}, and utility function \citep{browne1995optimal}. The notion of risk is widely studied in the optimization community over  past decades. See, e.g.,  \cite{ruszczynski2006conditional,ruszczynski2006optimization,ruszczynski2010risk,dentcheva2019risk,kose2020risk}, and the  references therein.

Furthermore, our work is 
closely related to the literature on risk-sensitive reinforcement learning
with the  variance  risk measure. 
The study of the variance of the total returns in a Markov decision process (MDP) dates back to \cite{sobel1982variance}.
\cite{filar1989variance} formulates the   variance-regularized MDP 
 as a nonlinear program. 
\cite{mannor2011mean}
proves that finding an exact  optimal policy of  variance-constrained reinforcement learning   is NP-hard, even when the model   is  known. 
More recently, 
 with linear function approximation, 
 \citep{tamar2016learning} proposes a temporal-difference learning algorithm for estimating the variance of the total reward, and 
 \cite{tamar2013variance,prashanth2016variance, prashanth2018risk} propose  actor-critic algorithms for variance-constrained policy optimization. 
 These works all establish asymptotic convergence guarantees via stochastic approximation \citep{borkar2009stochastic}. 
 A more related work is \cite{xie2018block}, which proposes an  actor-critic algorithm via Lagrangian and Fenchel duality, which is shown to converge to a stationary point  at a sublinear rate under the linear setting. In contrast, our work employs deep neural networks, adopts a different KL-divergence regularized policy update, and our algorithm provably finds a globally optimal policy at a sublinear rate.

%In the vast majority of existing reinforcement learning works, 
%theoretical results for the risk-sensitive setting is 

%the goal is to maximize the total reward in expectation. 

%That is, we maximize the reward in a risk-neutral setting. This somewhat ignores some risk/uncertainty in practice. Meanwhile, in many important application domains such as finance and government planning, it is crucial to be risk-averse and control the risks, where such uncertainties cannot be ignored. A few exceptions include \cite{tamar2013variance,tamar2016learning,prashanth2016variance}, where the authors study the variance of the total returns. However, these works only consider  classical settings where the models are linear, and cannot be generalized for deep reinforcement learning. This is because  that we need much more sophisticated algorithms and analysis tools to handle complicated deep neural networks.

\vspace{5pt}

\noindent{\bf Paper Organization}. The rest of this paper is organized as follows. In Section~\ref{sec:2}, we briefly introduce  some background knowledge. In Section~\ref{sec:alg}, we present the VARAC algorithm. In Section~\ref{sec:results}, we provide theoretical guarantees for the VARAC algorithm. To better illustrate our theory, we provide the analysis of VARAC for risk-sensitive RL with linear function approximation in Section~\ref{sec:linear}. In Section~\ref{sec:exp}, we conduct numerical experiments to investigate the empirical performance of our method using two mechanical control environments.
We conclude the paper in Section~\ref{sec:con}. 

\vspace{5pt}

\noindent{\bf Notations}.
For an integer $H$, we denote by $[H]$  the set $\{1, 2, \cdots, H\}$. Meanwhile, for any $x \in \RR$, we define $[x]_+ = \max(x, 0)$. Furthermore, we denote by $\| \cdot \|_2$ the $\ell_2$-norm of a vector or the spectral norm of a matrix, and denote by $\| \cdot \|_\F$ the Frobenius norm of a matrix.  Also, let $\{a_n\}_{n \ge 0}$ and $\{b_n\}_{n \ge 0}$  be two positive sequences. If there exists some positive constant $c$ such that $\limsup_{n \rightarrow \infty} a_n / b_n \le c$, we write $a_n = \cO(b_n)$.
% If $a_n = \cO(b_n\log^k n)$ for some $k>0$, we write $a_n=\tilde\cO(b_n)$. 
 If $\liminf_{n \rightarrow \infty} a_n / b_n \ge c_1$ for some positive constant $c_1$, we write   $a_n = \Omega(b_n)$.

%
%
%\begin{itemize}
%    \item Introduce the epistemic and aleatoric uncertainty of RL. Epistemic uncertainty stems from limited data and aleatoric uncertainty is the inherent randomness of the environment.  \citep{clements2019estimating} 
%    \item  Although there are a lot of works on DRL, most of these works focus on the risk-neutral setting. The aleatoric uncertainty of DRL is rather less explored in theory. 
%    \item In this work, we focus on the mean-variance risk-sensitive RL. We formulate this problem as learning the optimal policy subject to a variance constraint. 
%    \item Challenge of this problem: i) Need to handle the variance constraint. ii) Need to estimate the policy gradient of the variance term, which is a functional of the return. iii) nonconvexity due to having neural policies
%    \item To handle the change, we use i) Lagrange multiplier to handle the constraint and ii) Fenchel duality to estimate the policy gradient of the variance term 
%    \item To the best of our knowledge, we seem to propose the first actor-critic algorithm for risk-sensitive RL that is provably convergent and achieves optimality
%   
%\end{itemize}
%
%Related works: 
%\begin{itemize}
%    \item Risk sensitive RL with mean-variance risk measure
%    \item Policy optimization methods, especially those with function approximation and neural networks
%    \item deep neural networks and neural tangent kernel 
%\end{itemize} 
%

% !TEX root = main.tex

\section{Background} \label{sec:2}
In this section, we briefly review the Markov Decision Process (MDP) under the  average reward setting, the variance-constrained policy optimization problem, and some background of the deep neural network.
\vskip4pt
\noindent{\bf Markov Decision Process.}
We consider the Markov decision process $(\cS, \cA, \cP, r)$, where $\cS$ is a compact state space, $\cA$ is a finite action space, $\cP: \cS\times\cS\times\cA \to \RR$ is the transition kernel, and $r: \cS\times\cA \to \RR$ is the reward function. 
A stationary policy $\pi$ maps each state to  a probability distribution over $\cA$ that $\pi(\cdot|s)\in \cP(\cA)$, where $\cP(\cA)$ is the probability simplex on the action space $\cA$. 
Given a policy $\pi$, the state $\{ s_t\}_{t\geq 0}$ and state-action pair $\{ (s_t, a_t )\}_{t\geq 0}$ are sampled from  the Markov chain over $\cS$  
and $\cS \times \cA$, respectively. 
Throughout this paper, we assume that the Markov chains induced by any stationary policy admit stationary distributions. 
Moreover, 
we denote by $\nu_\pi (s)$ and $\sigma_\pi (s, a)=\pi(a\,|\,s)\cdot \nu_\pi(s)$ the stationary state distribution and the stationary state-action distribution associated with a policy $\pi$, respectively. For ease of presentation, we denote by $\EE_{\sigma_\pi}[\,\cdot\,]$ and $\EE_{\nu_\pi}[\,\cdot\,]$ the expectations $\EE_{(s,a)\sim\sigma_\pi}[\,\cdot\,] = \EE_{a\sim\pi(\cdot\,|\,s), s\sim \nu_\pi(\cdot)}[\,\cdot\,]$ and $\EE_{s\sim\nu_\pi}[\,\cdot\,]$, respectively.

\vskip4pt
\noindent{\bf Average Reward Setting.}
For a given stationary policy $\pi: \cA \times \cS \to \RR$, we measure its performance   using its (long-run) average reward per step, which is defined as 
\#\label{eq:rho}
\rho(\pi)=\lim _{T \rightarrow \infty} \frac{1}{T} \cdot \EE\Bigl[\sum_{t=0}^{T-1} r(s_t, a_t) \,\big|\, \pi\Bigr] = \EE_{(s,a)\sim\sigma_\pi}[r(s,a)].
\#
For all states $s$ in $\cS$ and actions $a$ in $\cA$, the differential action-value function (Q-function) of a policy $\pi$ is defined as
\# \label{eq:def:q}
Q^{\pi}(s,a) = \sum^\infty_{t = 0}\EE\big[ r(s_t, a_t) - \rho(\pi) \, | \, s_0 = s, ~ a_0 = a,~ a_{t} \sim \pi(\cdot\,|\, s_t), ~s_{t+1} \sim \cP(\cdot\,|\, s_t, a_t)\big].
\#
Correspondingly, the differential state-value function of a policy $\pi$ is defined as
\#\label{eq:def:v}
V^\pi(s) = \sum^\infty_{t = 0}\EE\big[ r(s_t, a_t) - \rho(\pi) \, | \, s_0 = s,~ a_{t} \sim \pi(\cdot\,|\, s_t), ~s_{t+1} \sim \cP(\cdot\,|\, s_t, a_t)\big].
\#
In the context of risk-sensitive optimization, one of the most common risk measures is the  long-run variance of reward obtained under policy $\pi$, which is defined as
\$
\Lambda(\pi) = \lim _{T \rightarrow \infty} \frac{1}{T} \cdot \EE\Bigl[\sum_{t=0}^{T-1} {\bigl(r(s_t, a_t) - \rho(\pi)\bigr)}^{2} \big|\, \pi \Bigr] = \EE_{(s,a)\sim\sigma_\pi} \bigl[r(s,a) - \rho(\pi)\bigr]^2.
\$
It is not difficult to show that
\#\label{eq:eta}
\Lambda(\pi) = \eta(\pi) - {\rho(\pi)}^{2},    \quad     \text{where }   \, \eta(\pi) = \EE_{(s,a)\sim\sigma_\pi}[r(s,a)^2].
\#
Let $W^{\pi}$ and $U^{\pi}$ be  the differential action-value and value functions associated with the squared reward of  policy $\pi$ that
\begin{align} 
W^{\pi}(s,a) & = \sum^\infty_{t = 0}\EE[ r(s_t, a_t)^2 - \eta(\pi) \, | \, s_0 = s, ~ a_0 = a,~ a_{t} \sim \pi(\cdot\,|\, s_t), ~s_{t+1} \sim \cP(\cdot\,|\, s_t, a_t)], \label{eq:def:w}\\
U^\pi(s) & = \sum^\infty_{t = 0}\EE[ r(s_t, a_t)^2 - \eta(\pi) \, | \, s_0 = s, a_{t} \sim \pi(\cdot\,|\, s_t), ~s_{t+1} \sim \cP(\cdot\,|\, s_t, a_t)]. \label{eq:def:u}
\end{align}
We denote by $\la \cdot, \cdot \ra$ the inner product over $\cA$, e.g., we have $V^\pi(s) = \EE_{a\sim\pi(\cdot\,|\,s)}[Q^\pi(s,a)] = \la Q^\pi(s,\cdot), \pi(\cdot\,|\,s)\ra$ and $U^\pi(s) = \EE_{a\sim\pi(\cdot\,|\,s)}[W^\pi(s,a)] = \la W^\pi(s,\cdot), \pi(\cdot\,|\,s)\ra$.

Throughout our discussion, we  impose a standard assumption that the reward function is uniformly bounded. In particular, we assume that there exists a constant $M>0$ such that $M = \sup_{(s,a)\in \cS\times\cA}|r(s, a)|$. 
	%Where $\rho(\pi)$ and $\eta(\pi)$ are defined in \eqref{eq:rho} and \eqref{eq:eta}, respectively.
%This bounded reward assumption is standard in literature \citep{liu2019neural}. 
As an immediate consequence, we have  that  for any policy~$\pi$,
 \# \label{eq:bound:rho:eta}
| \rho(\pi) | \le M,  \quad \qquad |\eta(\pi)| \le M^2 .
 \#

\vskip5pt
\noindent{\bf Variance-Constrained Problem.}
We consider the following constrained policy optimization problem to find a policy that maximizes the long-run average reward subject to the constraint that the long-run variance is upper bounded by a certain threshold.  In particular, for a given   $\alpha>0$, we consider the following constrained optimization problem 
\#\label{eq:problem}
\max_{\pi}\rho(\pi)      \text{     \qquad subject to } \Lambda(\pi)\leq\alpha.
\#

\vskip5pt
\noindent{\bf Deep Neural Networks.}
To facilitate our discussion, we briefly review  some basics of deep neural networks (DNNs) \citep{allen2018convergence,gao2019convergence}. Let $x\in\RR^d$ be the input data. Suppose that we have a DNN with $H$ layers of width $m$. We denote by $W_h$ the weight matrix at the $h$-th layer for $ h\in [H]$, where $W_1\in\RR^{d\times m}$ and $W_h\in\RR^{m\times m}$ for $2\le h \le H$. For a DNN with depth $H$, width $m$, and parameter $\theta = \bigl(\vec(W_1)^\top, \cdots, \vec(W_H)^\top\bigr)^\top$, its output $u_\theta(x)$ is  recursively defined~as
\#\label{eq:def-nn-form}
& x^{(0)} = x, \notag \\
& x^{(h)} = \frac{1}{\sqrt{m}}\cdot \sigma( W_h^\top x^{(h-1)} ), \quad \text{ for }h\in[H], \\
&  u_\theta(x) = b^\top x^{(H)}, \notag
\#
where $\sigma(\cdot) = \max\{0, \cdot\}$ is the ReLU activation function, and $b\in\{-1,1\}^{m}$ is the output layer. %, and $\theta = \bigl(\vec(W_1)^\top, \ldots, \vec(W_H)^\top\bigr)^\top$ is the network parameter of $u_\theta(\cdot)$. 
Without loss of generality, we assume that the input $x \in \mathbb{R}^d$ satisfies  $\|x\|_2 = 1$, where $\|\cdot\|_2$ denotes the $\ell_2$-norm. 
In the context of deep reinforcement learning, this can be achieved by having a known embedding function that maps each state-action pair to the unit sphere in $\RR^d$. 
%We initialize the DNN by setting each entry of $W_h$ for $h\in[H]$ following the i.i.d. standard Gaussian distribution $\mathcal N(0,1)$ and each entry of $b$ following the i.i.d. discrete uniform distribution ${\rm Unif}(\{-1,1\})$. 
Besides, we initialize the network parameters randomly by 
%we adopt the following random initialization of the parameters that
\# \label{eq:initialization}
&[W_1]_{i,j} \overset{\rm i.i.d.}{\sim} \mathcal N(0,1)   \text{  for all } (i,j) \in [d] \times [m] , \notag \\
&[W_h]_{i,j} \overset{\rm i.i.d.}{\sim} \mathcal N(0,1)  \text{ for all } (i,j) \in [m]\times[m] \text{ and }  2 \le h \le H ,   \\
&b_i \overset{\rm i.i.d.}{\sim} {\rm Unif}(\{-1,1\})  \text{ for all } i \in [m]. \notag
\#
Without loss of generality, we only update   $\{W_h\}_{h\in[H]}$ throughout the training process,  and fix the output layer~$b$ as its initialization. We denote by $\theta_0 = (\vec(W_1^0)^\top, \cdots, \vec(W_H^0)^\top)^\top$ the initialization of the network parameter.  
In addition, we restrict  network parameter~$\theta$ within a ball centered at $\theta_0$ with radius $R > 0$, which is given by 
\#\label{eq:def-proj-set}
\cB(\theta_0, R) = \bigl\{\theta \in  \RR^{m_{\rm all}} \colon \|W_{h} - W_{h}^0\|_\F \leq R \text{ for $h\in[H]$}  \bigr\},
\#
where  $\{W_{h}\}_{h\in[H]}$ and $\{W_{h}^0\}_{h\in[H]}$ are the weight matrices of  network parameters $\theta$ and $\theta_0$, respectively, and  $\|\cdot\|_\F$  denotes the Frobenius norm. For any fixed depth $H$, width $m$, and radius $R > 0$, the corresponding class of DNNs is 
\#\label{eq:def-dnn-class}
\cU(m, H, R) = \bigl\{u_\theta(\cdot)\colon \theta\in \cB(\theta_0, R)\bigr\}. 
\#

\section{Algorithm}\label{sec:alg}

In this section, we present the Variance-Constrained Actor-Critic with Deep Neural Networks (VARAC) algorithm for solving the  variance-constrained problem \eqref{eq:problem}.

%To solve the variance-constrained problem \eqref{eq:problem}, we propose to work on its Lagrangian dual problem. Moreover, to handle the setting where $\cS$ is large, we adopt DNNs   to parameterize the policies, differential action-value functions and  differential action-value functions associated with the squared reward. Based on such a parameterization, we propose the Variance-Constrained Actor-Critic with Deep Neural Networks (VARAC) algorithm. 
%For notational simplicity, we denote by $\nu_k$  and $\sigma_k$ the stationary state distribution $\nu_{\pi_{\theta_k}}$ and the stationary state-action distribution $\sigma_{\pi_{\theta_k}}$, respectively. 
%Also, we define an auxiliary distribution $\tilde{\sigma}_k$ over $\cS\times \cA$ as $\tilde{\sigma}_k = \nu_k\pi_0$.

\subsection{Problem Formulation} 
As we discussed in the introduction, a major challenge of solving problem \eqref{eq:problem} is that the constraint is difficult to handle. We first transform the problem into an unconstrained saddle point problem by employing the Lagrangian dual formulation that
\# \label{eq:lagrangian}
&\min_{\lambda}\max_{\pi}~\rho(\pi) - \lambda\bigl(\Lambda(\pi) - \alpha\bigr)  =\min_{\lambda}\max_{\pi}~\rho(\pi) - \lambda\eta(\pi) +\lambda\rho(\pi)^2 + \lambda\alpha , 
\#
where the equality follows from \eqref{eq:eta}. As mentioned earlier, the quadratic term $\rho(\pi)^2$ makes the problem nonlinear in the probability distribution, and raises substantial challenges in the computation. Following Lemma 1 of \cite{xie2018block}, we reformulate the problem by leveraging the quadratic term's Fenchel dual. In particular, by the Fenchel duality, we have that $\rho(\pi)^2 =\max_{y\in\RR}(2y\rho(\pi)-y^2)$. Then, the Lagrangian dual is transformed to the following form that
\# \label{eq:problem-formulated}
%\begin{aligned}
&\min_{\lambda}\max_{\pi}~\rho(\pi) - \lambda\eta(\pi) +\lambda\rho(\pi)^2 + \lambda\alpha  \notag\\
&\qquad=\min_{\lambda}\max_{\pi}~\rho(\pi) - \lambda\eta(\pi) + \lambda\max_{y}\bigl(-y^2 + 2y\rho(\pi)\bigr) + \lambda\alpha \notag \\
&\qquad=\min_{\lambda}\max_{\pi}\max_{y}~(1 + 2\lambda y)\rho(\pi) - \lambda\eta(\pi) - \lambda y^2 + \lambda\alpha  .
%\end{aligned}
\#
To facilitate our discussion, we denote the Lagrangian dual function as 
\#\label{eq:def-L}
\cL(\lambda,\pi,y) = (1 + 2\lambda y)\rho(\pi) - \lambda\eta(\pi) - \lambda y^2 + \lambda\alpha .
\#

%\subsection{Parameterize the Policy via Deep Neural Networks}

To handle the potentially complicated functional structures, we propose to use DNNs to represent  the policy $\pi$, differential action-value function $Q$ defined in \eqref{eq:def:q}, and  differential action-value function $W$ associated with the squared reward   defined in \eqref{eq:def:w}. In particular, we consider the energy-based policy $\pi_\theta(a\given s) \propto \exp( \tau^{-1} f_\theta(s,a) )$, where the energy function $f_\theta(s,a)\in\cU(m_{\rm a}, H_{\rm a}, R_{\rm a})$ is parameterized as a DNN with network parameter~$\theta$ \citep{ haarnoja2017reinforcement,wang2019neural}.  Also, we assume that $Q(s,a) = Q_{q}(s, a)$ and $W(s,a) = W_{\omega}(s,a)$ for all $(s,a)\in\cS\times\cA$, where $Q_q(s,a)\in\cU(m_{\rm c}, H_{\rm c}, R_{\rm c})$ and $W_\omega(s,a)\in\cU(m_{\rm b}, H_{\rm b}, R_{\rm b})$ are parameterized as DNNs with network parameters $q$ and $\omega$, respectively.

\subsection{VARAC Algorithm}
We propose the variance-constrained actor-critic with deep neural networks (VARAC) algorithm to solve \eqref{eq:problem}. The algorithm follows the general framework of  the actor-critic method \citep{konda2000actor}. This method solves the unconstrained problem of  maximizing the long-run average reward in \eqref{eq:rho}. At each iteration, in the actor update step, we improve the policy that given the previous estimator for the Q-function, we compute an estimator for the policy  gradient, and conduct a gradient step of the policy. In the critic update step, by plugging the updated policy in, we invoke a policy evaluation algorithm to update the estimator for the Q-function.

 In our setting, due to the variance constraint, the problem is substantially more challenging, and the actor-critic algorithm cannot be directly applied. As discussed in the previous subsection, by employing the Lagrangian and Fenchel dual formulations, we  aim to solve the unconstrained min-max-max problem \eqref{eq:problem-formulated}. Specifically, at each iteration, we first conduct an actor update step, where we update $\lambda$ and $\pi$.  In particular, using solutions  from the previous iteration $k$, we update the Lagrangian multiplier $\lambda$ by  a projected gradient descent step. Next, we update $\pi_\theta$ by the  proximal policy optimization (PPO) algorithm \citep{schulman2017proximal}, where we maximize a KL-penalized objective over $\theta$. To be more specific, in updating $\pi_\theta$, by plugging previous   estimators for the Q-function in \eqref{eq:def:q} and the W-function in \eqref{eq:def:w}, we aim to maximize a linearized version of $\cL(\lambda_{k},\pi_{\theta_{k+1}}, y_k)$ over $\theta_{k+1}$  with a penalty of KL-divergence of $\pi_{\theta_{k+1}}$ and $\pi_{\theta_{k}}$, which is equivalent to 
 $$
 \max_{\theta_{k+1}} \EE_{\nu_{\pi_{\theta_k}}} \bigl[ \bigl\la (1+ 2{\lambda}_{k}{y}_{k})Q_{q_k}(s, \cdot) - {\lambda}_{k}W_{\omega_{k}}(s, \cdot), \pi_{\theta_{k+1}}(\cdot\,|\,s)\bigr\ra- \beta_k \cdot {\rm KL}\bigl(\pi_{\theta_{k+1}}(\cdot\,|\,s)\,\|\, \pi_{\theta_{k}}(\cdot\,|\,s)\bigr) \bigr].
 $$
The key observation is that, by considering energy-based policies, the problem above admits a tractable solution and can be computed efficiently. Finally, we update $y$ by maximizing a quadratic function.
 
 For the critic update step, we update the estimators for the Q-function and W-function by minimizing the Bellman errors. Recall that  we parameterize the Q-function and W-function by deep neural networks with parameters $q$ and $\omega$, respectively. The Bellman error minimization problems become solving
$$
\min_{q \in \cB(q_0,R_c)}\EE_{\sigma_k}[\bigl(Q_q(s,a) - [\cT^{\pi_{\theta_k}}Q_q](s,a)\bigr)^2], \text{ and }\min_{\omega \in \cB(\omega_0,R_{\rm b})}\EE_{\sigma_k}\big[W_\omega(s,a) - [\hat{\cT}^{\pi_{\theta_k}}W_\omega](s,a)\big]^2,
$$
where $\cT^{\pi_{\theta_k}}$ and $\hat{\cT}^{\pi_{\theta_k}}$ are Bellman operators defined later in \eqref{eq:bellman1} and \eqref{eq:bellman2}, respectively. 
We solve these  problems by the temporal difference (TD) method \citep{sutton1988learning}.

We then present the details of the VARAC algorithm.  At the   $(k-1)$-th iteration, we estimate $\rho(\pi_{\theta_{k}})$ and $\eta(\pi_{\theta_{k}})$ by their sample average estimators that 
\# \label{eq:estimate:rho:eta}
\overline{\rho}(\pi_{\theta_{k}}) = \frac{1}{T}\cdot\sum^{T}_{t=1}r(s_t^{k}, a_t^{k}), \qquad \overline{\eta}(\pi_{\theta_{k}}) = \frac{1}{T}\cdot\sum^{T}_{t=1}r(s_t^{k}, a_t^{k})^2 ,
\#
where $T$ is the sample size, and $\{(s_t^{k},a_t^{k})\}_{t=1}^T$ are simulated samples of states and actions following the policy  from the previous iteration. In what follows, with some slight abuse of notation, we write $(s_t^{k},a_t^{k})$ as $(s_t,a_t)$. Note that by the boundedness of the reward, we have 
\#  \label{eq:bound:rho:eta:bar}
| \overline{\rho}(\pi_{\theta_{k}}) | \le M,  \quad \qquad | \overline{\eta}(\pi_{\theta_{k}}) | \le M^2 .
\#
We then present the actor and critic updates at each iteration.

%Then, we present the details of actor update and critic update, respectively. Briefly speaking, we first maximize over $y$, then minimize over $\lambda$, and then maximize over $\pi$.
%Next, we present the actor and critic updates in each iteration. Briefly speaking, in the actor update, we optimize over the decision variables that we update $\lambda$ by the projected gradient method, update $\pi$ by the optimistic proximal policy optimization (OPPO) algorithm \citep{cai2019provably}, and update $y$ by maximizing a quadratic function as discussed later. In the critic update, we evaluate the current solutions by estimating  the differential action-value function $Q^\pi$ and the differential action-value function associated with the squared reward $W^\pi$ via a temporal differences (TD) update \citep{sutton1988learning}.

\vskip4pt
{\noindent\bf Actor Update: (i) $\lambda$-Update Step.}
At the $k$-th iteration, given the solution $(\overline{\lambda}_k,{\pi}_{\theta_k},\overline{y}_k)$ from the $(k-1)$-th iteration, we compute $\lambda_{k+1}$ using the projected gradient method, where we project the solution onto a  bounded region to guarantee the  convergence \citep{prashanth2013actor,prashanth2016variance}.
%The projection area can be justified as long as the saddle point is bounded. 
In particular, we choose a sufficiently large $N > 0$ and update $\lambda_{k+1}$ that
\$
\lambda_{k+1} = \Pi_{[0,N]}\big(\lambda_k - \frac{1}{2\gamma_k}\partial_\lambda\cL(\lambda_k,\pi_{\theta_k},y_{k}) \big) =\Pi_{[0,N]}\Big(\lambda_k - \frac{1}{2\gamma_k}\big(\alpha + 2y_{k}\rho(\pi_{\theta_k}) - \eta(\pi_{\theta_k}) -y_k^2\big) \Big),
\$
where $\gamma_k>0$ is some prespecified stepsize.  
As discussed previously, we do not observe $\rho(\pi_{\theta_k})$ and $\eta(\pi_{\theta_k})$, and as discussed later in $y$-update step, we do not observe the ``ideal" $y_k$. Instead, we estimate them using $\overline{\rho}(\pi_{\theta_k})$, $\overline{\eta}(\pi_{\theta_k})$ in \eqref{eq:estimate:rho:eta} and $\overline{y}_k$ in \eqref{eq:update-y-form}, respectively. We then adopt the following plug-in estimator $\overline{\lambda}_{k+1}$ for the ``ideal" $\lambda_{k+1}$ that
\#\label{eq:update-lambda-form}
\overline{\lambda}_{k+1} =  \Pi_{[0,N]}\Bigl(\overline{\lambda}_k - \frac{1}{2\gamma_k}\bigl(\alpha + 2\overline{y}_k\overline{\rho}(\pi_{\theta_k}) - \overline{\eta}(\pi_{\theta_k}) -\overline{y}_k^2\bigr) \Bigr).
\#
%where $\overline{\lambda}_k$ is the estimators for $\lambda_{k}$.

\vskip4pt
{\noindent\bf (ii) $\pi$-Update Step.}
Note that the policy $\pi$ is parametrized by $\theta$. By the proximal policy optimization method \citep{schulman2017proximal}, we update our policy $\pi_{\theta_{k+1}}$ by maximizing the following $\rm{KL}$-penalized objective over $\theta_{k+1}$,
\#\label{eq:dnn2222}
\begin{aligned}
\max_{\theta_{k+1}}L(\theta_{k+1}) =  \EE_{\nu_{\pi_{\theta_k}}} \bigl[ & \bigl\la (1+ 2\overline{\lambda}_{k}\overline{y}_{k})Q_{q_k}(s, \cdot) - \overline{\lambda}_{k}W_{\omega_{k}}(s, \cdot), \pi_{\theta_{k+1}}(\cdot\,|\,s)\bigr\ra \\
& - \beta_k \cdot {\rm KL}\bigl(\pi_{\theta_{k+1}}(\cdot\,|\,s)\,\|\, \pi_{\theta_{k}}(\cdot\,|\,s)\bigr) \bigr],
\end{aligned}
\#
where $\overline{\lambda}_k$ in \eqref{eq:update-lambda-form} and $\overline{y}_k$ in \eqref{eq:update-y-form} are the estimators for $\lambda_k$ and $y_k$, and $\beta_k>0$ is some prespecified penalty parameter. Note that here we use DNNs $Q_{q_k}$ and $W_{\omega_k}$ to estimate  $Q^{\pi_{\theta_k}}$ and $W^{\pi_{\theta_k}}$, and we provide the theoretical justifications of using DNNs in Section~\ref{sec:ac-error}. 

Solving problem~\eqref{eq:dnn2222} is challenging since the gradient of {\rm KL}-divergence in the objective is difficult to derive. To efficiently and approximately solve the maximization problem~\eqref{eq:dnn2222}, we consider the energy-based policy 
%$\pi(a\,|\,s) \propto \exp\{\tau^{-1}f(s,a)\}$ for the temperature parameter $\tau > 0$ and the energy function $f: \cS\times \cA \to \RR$. We assume 
$\pi_{\theta_{k+1}} \propto \exp(\tau_{k+1}^{-1} f_{\theta_{k+1}})$, where $\tau_{k+1} > 0$ is a temperature parameter, and $f_\theta(s,a)\in\cU(m_{\rm a}, H_{\rm a}, R_{\rm a})$, which is  parameterized by a DNN with network parameter~$\theta$, is an energy function \citep{liu2019neural}. The next proposition shows that problem \eqref{eq:dnn2222} admits a  tractable solution of the oracle infinite-dimensional policy~update.

%See Appendix \ref{appendix:alg-proof} for the detailed proof.

%In the rest of the paper, we denote by $\pi \propto \exp\{\tau^{-1} f\}$ such energy-based policies. 
\begin{proposition}\label{prop:exp-policy}
Let $\pi_{\theta_{k+1}} \propto \exp(\tau_{k+1}^{-1} f_{\theta_{k+1}})$ be an energy-based policy. For any given $\overline{\lambda}_{k}$ and $\overline{y}_{k}$, and given estimators $Q_{q_k}$, $W_{\omega_{k}}$ for $Q^{\pi_{\theta_k}}$ and $W^{\pi_{\theta_k}}$ respectively, we have that $\hat{\pi}_{k+1} = \argmax_{\pi}(\EE_{\nu_k}[\la (1+2\overline{\lambda}_{k}\overline{y}_{k})Q_{q_k}(s, \cdot) - \overline{\lambda}_{k}W_{\omega_{k}}(s, \cdot),\pi(\cdot\,|\,s) \ra - \beta_k \cdot {\rm KL}(\pi(\cdot\,|\,s)\,\|\, \pi_{\theta_k}(\cdot\,|\,s))])$ satisfies
\#\label{eq:pi-new}
\hat{\pi}_{k+1} \propto \exp\big(\beta_k^{-1}(1 + 2\overline{\lambda}_{k}\overline{y}_{k})Q_{q_k} - \beta_k^{-1}\overline{\lambda}_{k}W_{\omega_{k}} + \tau_k^{-1} f_{\theta_k}\big),
\#
%for any $(s,a) \in \cS \times \cA$, 
%{\color{red} Ask Han why $(s,a)$ is put here}
where $\nu_k = \nu_{\pi_{\theta_k}}$ is the stationary state distribution generated by $\pi_{\theta_k}.$
\end{proposition}
\begin{proof}
See Appendix \ref{appendix:alg-proof} for the detailed proof.
\end{proof}

By Proposition \ref{prop:exp-policy}, we update the policy parameter $\theta$ by solving the following problem, 
{\small
\#\label{eq:policy-update}
\theta_{k+1} = \argmin_{\theta \in \cB(\theta_0,R_a)}\EE_{{\sigma}_k} \bigl[\bigl(f_{\theta}(s,a) - \tau_{k+1}\cdot(\beta_k^{-1}(1 + 2\overline{\lambda}_{k}\overline{y}_{k})Q_{q_k} - \beta_k^{-1}\overline{\lambda}_{k}W_{\omega_{k}} + \tau^{-1}_{k}f_{\theta_k}(s,a))\bigr)^2\bigr], 
\#}
where $\sigma_k$ is the stationary state-action distribution of $\pi_{\theta_k}$. That is, to minimize the distance between the output $f_{\theta_{k+1}}$ and the right hand side of \eqref{eq:policy-update}. 
To solve \eqref{eq:policy-update}, we adopt the projected stochastic gradient descent method. Specifically, given an initial $\theta_0$, at the $t$-th iteration, we update
\begin{equation}\label{eq:sgd-update1}
\begin{aligned}
\theta{(t+1)} \leftarrow \Pi_{\cB(\theta_0, R_a)}\Bigl(&\theta{(t)} - \zeta \cdot \bigl(f_{\theta{(t)}}(s,a) - \tau_{k+1}\cdot\big(\beta_k^{-1}(1 + 2\overline{\lambda}_{k}\overline{y}_{k})Q_{q_k} \\
&+ \beta_k^{-1}\overline{\lambda}_{k}W_{\omega_{k}} + \tau^{-1}_{k}f_{\theta_k}(s,a)\big)\bigr)\cdot\nabla_\theta f_{\theta{(t)}}(s,a)\Bigr),
\end{aligned}
\end{equation}
where the operator $\Pi_{\cB(\theta_0, R_a)}(\cdot)$ projects the solution onto the set $\cB(\theta_0, R_a)$ defined in \eqref{eq:def-proj-set},  the state-action pair $(s, a)$ is sampled from $\sigma_k = \sigma_{\pi_{\theta_k}}$, and $\zeta>0$ is the stepsize. See Algorithm~\ref{alg:ac-update} in Appendix \ref{appendix:alg} for a pseudocode.

\vskip4pt
{\noindent\bf (iii) $y$-Update Step.}
Given  $\overline{\lambda}_{k+1}$ and $\theta_{k+1}$, we update $y_{k+1}= \argmax_y \cL(\overline{\lambda}_{k+1},\pi_{\theta_{k+1}},y)$. By the property of the quadratic function, it is easy to see that  $y_{k+1}=\rho(\pi_{\theta_{k+1}})$. However, since $\rho(\pi_{\theta_{k+1}})$ is unknown,
we adopt $\overline{\rho}(\pi_{\theta_{k+1}})$ defined in \eqref{eq:estimate:rho:eta} as an estimator for $y_{k+1}$ that
\#\label{eq:update-y-form}
%\overline{y}_{k+1} = \Pi_{[0,M]}\bigl(\overline{\rho}(\pi_{\theta_k})\bigr) ,
\overline{y}_{k+1} = \overline{\rho}(\pi_{\theta_{k+1}}).
\#
%where $\overline{\rho}(\pi_{\theta_k})$ defined in \eqref{eq:estimate:rho:eta} is an estimator for  $\rho(\pi_{\theta_k}) $.

\vskip4pt
{\noindent\bf Critic Update: (i) $q$-Update Step.} %\label{sec:td}
In the critic update, we evaluate the current solution $(\overline{\lambda}_k,\pi_{\theta_k},\overline{y}_k)$ by estimating the corresponding value functions. We first  consider the differential action-value function (Q-function), and derive an estimator $Q_{q_k}$ for $Q^{\pi_{\theta_k}}$ in \eqref{eq:dnn2222}, where $Q_{q_k}$ is parametrized as a DNN with network parameter $q_k$. To obtain $q_k$, we  solve the following least-squares problem 
\#\label{eq:mspbe1}
q_{k} = \argmin_{q \in \cB(q_0,R_c)}\EE_{\sigma_k}[\bigl(Q_q(s,a) - [\cT^{\pi_{\theta_k}}Q_q](s,a)\bigr)^2],
\#
where $\sigma_k$ is the stationary state-action distribution of the policy $\pi_{\theta_k}$. 
Here the Bellman %evaluation 
operator $\cT^\pi$ of a policy $\pi$ is% defined as
\#\label{eq:bellman1}
[\cT^\pi Q](s, a) = \EE [r(s,a) -\rho(\pi) + Q(s',a') \,\big|\, s'\sim\cP(\cdot\,|\,s, a),~a'\sim\pi(\cdot\,|\,s')].
\#
Recall that  $Q_{q} \in {\cU} (m_c, H_c, R_c)$  is defined through a deep neural network in \eqref{eq:def-dnn-class}, where $q$ is the network parameter, $H_c$ is the depth, $m_c$ is the width, and $R_c$ is the projection radii. To solve \eqref{eq:mspbe1}, given an initial $q_0$, we use the iterative TD-update that at the $t$-th iteration,  we let
\begin{equation}\label{eq:td-update1}
\begin{aligned}
q{(t+1)} \leftarrow \Pi_{\cB{(q_0,R_{\rm c})}}\Bigl(&q{(t)} - \delta\cdot \bigl(Q_{q{(t)}}(s,a) - r(s, a) \\
&+ \overline{\rho}(\pi_{\theta_{k}}) - Q_{q{(t)}}(s', a')\bigr)\cdot\nabla_q Q_{q{(t)}}(s,a)\Bigr),
\end{aligned}
\end{equation}
{\noindent where $(s, a) \sim \sigma_k$, $s'\sim\cP(\cdot\,|\,s,a)$, $a' \sim \pi_{\theta_k}(\cdot\,|\,s')$, and $\delta$ is the stepsize. See Algorithm \ref{alg:td} in Appendix \ref{appendix:alg} for a pseudocode.} 

\vskip4pt
{\noindent\bf (ii) $\omega$-Update Step.}
Next, we derive an estimator $W_{\omega_k}$ for $W^{\pi_{\theta_k}}$ in \eqref{eq:dnn2222}. The procedure is similar to the previous step. We  solve the following least-squares problem to obtain~$\omega_{k}$,
\#\label{eq:mspbe2}
\omega_{k} = \argmin_{\omega \in \cB(\omega_0,R_{\rm b})}\EE_{\sigma_k}\big[W_\omega(s,a) - [\hat{\cT}^{\pi_{\theta_k}}W_\omega](s,a)\big]^2,
\#
where the operator $\hat{\cT}^\pi$ of a policy $\pi$ is defined as
\#\label{eq:bellman2}
[\hat{\cT}^\pi W](s, a) = \EE[r(s,a)^2 -\eta(\pi) + W(s',a') \,\big|\, s'\sim\cP(\cdot\,|\,s, a),~a'\sim\pi(\cdot\,|\,s')].
\#
As we discussed earlier, we parameterize $W$ using a DNN that we let $W_{\omega} \in {\cU} (m_b, H_b, R_b)$ defined in \eqref{eq:def-dnn-class}, where $\omega$ is the network parameter, $H_b$ is the depth,  $m_b$ is the width, and $R_b$ is the projection radii. To solve \eqref{eq:mspbe2}, given an initial $\omega_0$, we use the TD update that at the $t$-th iteration, we let
\begin{equation}\label{eq:td-update2}
\begin{aligned}
\omega{(t+1)} \leftarrow \Pi_{\cB{(\omega_0,R_b)}}\Bigl(& \omega{(t)} - \delta\cdot \bigl(W_{\omega(t)}(s,a) - r(s, a)^2 \\
& + \overline{\eta}(\pi_{\theta_{k}}) - W_{\omega{(t)}}(s', a')\bigr)\cdot\nabla_\omega W_{\omega{(t)}}(s,a)\Bigr),
\end{aligned}
\end{equation}
where $(s, a) \sim \sigma_k$, $s'\sim\cP(\cdot\,|\,s,a)$, $a' \sim \pi_{\theta_k}(\cdot\,|\,s')$, and $\delta$ is the stepsize. See Algorithm \ref{alg:td2} in Appendix \ref{appendix:alg} for a pseudocode.

Putting the actor and critic updates together, we present the pseudocode of the VARAC algorithm in  Algorithm~\ref{alg:risac}.

\begin{algorithm}%[H]
\caption{Variance-Constrained Actor-Critic with Deep Neural Networks}
\begin{algorithmic}[1]\label{alg:risac}
\REQUIRE MDP $(\cS, \cA, \cP, r)$, penalty parameter $\beta$, widths $m_a$, $m_b$ and $m_c$, depths $H_a$, $H_b$ and $H_c$, projection radii $R_a$, $R_b$ and $R_c$, number of SGD and TD iterations $T$ and number of VARAC iterations $K$
\STATE Initialize with uniform policy: $\tau_0\leftarrow 1$, $f_{\theta_0} \leftarrow 0$, $ \pi_{\theta_0} \leftarrow \pi_0 \propto \exp(\tau_0^{-1}f_{\theta_0})$
\STATE Sample $\{(s_t, a_t, a^0_t, s_t', a_t')\}^{T}_{t = 1}$ with $(s_t, a_t) \sim \sigma_0$, $a^0_t\sim \pi_0(\cdot\,|\,s_t)$, $s_t'\sim\cP(\cdot\,|\,s_t, a_t)$ and $a_t' \sim \pi_{\theta_0}(\cdot\,|\,s_t')$
\STATE Estimate $\rho(\pi_{\theta_{0}})$ and $\eta(\pi_{\theta_{0}})$ by $\overline{\rho}(\pi_{\theta_{0}}) = \frac{1}{T}\cdot\sum^{T}_{t=1}r(s_t, a_t)$ and $\overline{\eta}(\pi_{\theta_{0}}) = \frac{1}{T}\cdot\sum^{T}_{t=1}r(s_t, a_t)^2$
\FOR{$k = 0, \dots, K-1$}

\STATE Set temperature parameter $\tau_{k+1} \leftarrow \beta\sqrt{K}/(k+1)$ and penalty parameter $\beta_k \leftarrow \beta\sqrt{K}$

\STATE Solve $Q_{q_{k}}(s, a) \in \cU(m_{\rm c}, H_{\rm c}, R_{\rm c})$ in \eqref{eq:mspbe1} using the TD update in \eqref{eq:td-update1} (Algorithm \ref{alg:td})\label{line:td:q}

\STATE Solve $W_{\omega_{k}}(s, a) \in \cU(m_{\rm b}, H_{\rm b}, R_{\rm b})$ in \eqref{eq:mspbe2} using the TD update in \eqref{eq:td-update2} (Algorithm~\ref{alg:td2})\label{line:td:w}

\STATE Update $\lambda$ : $\overline{\lambda}_{k+1} =  \Pi_{[0,N]}\bigl(\overline{\lambda}_k - \frac{1}{2\gamma_k}(\alpha + 2\overline{y}_k\overline{\rho}(\pi_{\theta_k}) - \overline{\eta}(\pi_{\theta_k}) -\overline{y}_k^2) \bigr)$

\STATE Solve $f_{\theta_{k+1}} \in \cU(m_{\rm a}, H_{\rm a}, R_{\rm a})$ in \eqref{eq:policy-update} using the SGD update in \eqref{eq:sgd-update1} (Algorithm~\ref{alg:ac-update})\label{line:sgd}

\STATE Update policy: $\pi_{\theta_{k+1}} \propto \exp(\tau_{k+1}^{-1}f_{\theta_{k+1}})$ \label{line:c}
\STATE Sample $\{(s_t, a_t, a^0_t, s_t', a_t')\}^{T}_{t = 1}$ with $(s_t, a_t) \sim \sigma_{k+1}$, $a^0_t\sim \pi_0(\cdot\,|\,s_t)$, $s_t'\sim\cP(\cdot\,|\,s_t, a_t)$ and $a_t' \sim \pi_{\theta_{k+1}}(\cdot\,|\,s_t')$
\STATE Estimate $\rho(\pi_{\theta_{k+1}})$ and $\eta(\pi_{\theta_{k+1}})$ by $\overline{\rho}(\pi_{\theta_{k+1}}) = \frac{1}{T}\cdot\sum^{T}_{t=1}r(s_t, a_t)$ and $\overline{\eta}(\pi_{\theta_{k+1}}) = \frac{1}{T}\cdot\sum^{T}_{t=1}r(s_t, a_t)^2$\label{line:estimate}

\STATE Update $y$ : $\overline{y}_{k+1} = \overline{\rho}(\pi_{\theta_{k+1}}) $

\ENDFOR
\end{algorithmic}
\end{algorithm}

%%%%%%%%%%%%%%%

%%%%%%%%%%%%%
% !TEX root = main.tex

\section{Theoretical Results}\label{sec:results}
In this section, we  establish the  convergence of the proposed VARAC algorithm by analyzing the estimation and computation errors, and we show that the solution converges to a globally optimal solution at an $\cO(1/\sqrt{K})$ rate. Further, we show that under the Slater condition, we have both optimality and feasibility gaps diminish at $\cO(1/\sqrt{K})$ rates. Before going further, we first impose some mild assumptions.
%The estimation errors are introduced in estimating $\rho(\pi)$ and $\eta(\pi)$, and computation errors come from actor and critic updates in \eqref{eq:policy-update}, \eqref{eq:mspbe1} and \eqref{eq:mspbe2}. 

%Throughout our discussion, we impose the following assumption.

\begin{assumption}[]\label{assumption:unique-solution}
	There exists a saddle point $(\lambda^*, \pi^*, y^*)$, which is a solution of the saddle point optimization problem \eqref{eq:problem-formulated}.
\end{assumption}

%\begin{assumption}[Bounded Reward]\label{assumption:bounded-reward}
%There exists a constant $M>0$ such that $R_{\rm max} = \sup_{(s,a)\in \cS\times\cA}|r(s, a)|$, which implies there exists constants $M>0$ and $F>0$ such that $\rho(\pi)\le M$ and $\eta(\pi) \le F$ for any policy $\pi$.
%\end{assumption}

\begin{assumption}[]\label{assumption:closed-q-w}
For any $Q_q \in \cU(m_{\rm c},H_{\rm c},R_{\rm c})$, $W_\omega \in \cU(m_{\rm b},H_{\rm b},R_{\rm b})$, and policy $\pi$, we have $\cT^\pi Q_q \in \cU(m_{\rm c},H_{\rm c},R_{\rm c})$ and $\hat{\cT}^\pi W_\omega \in \cU(m_{\rm b},H_{\rm b},R_{\rm b})$.
\end{assumption}

Assumption \ref{assumption:unique-solution} assumes the existence of a solution. Assumption \ref{assumption:closed-q-w} assumes that the class of DNNs $\cU(m_{\rm c},H_{\rm c},R_{\rm c})$  in \eqref{eq:def-dnn-class} is closed under the Bellman  operator $\cT^\pi$ defined in \eqref{eq:bellman1}, and the class $\cU(m_{\rm b},H_{\rm b},R_{\rm b})$  is closed under the operator $\hat{\cT}^\pi$ defined in \eqref{eq:bellman2}. Such an assumption is standard in literature for all classes of policies  \citep{munos2008finite,antos2008fitted,tosatto2017boosted,yang2019theoretical,liu2019neural}.

Furthermore, to guarantee the convergence of TD updates \eqref{eq:td-update1} and \eqref{eq:td-update2}, we need an additional contraction condition, which is common in reinforcement learning literature \citep{van1998learning}.  
In particular,  suppose $s \in \cR^d$, and  define a Hilbert space $L_2(\cR^d, \cB(\cR^d), \pi)$, which is endowed with an inner product $\la J_1,J_2\ra_\pi = \int J_1(s)J_2(s)\pi(ds)$ for any real-valued functions $J_1,J_2$ on the Hilbert space. Also, for any policy $\pi$, we denote by $P^\pi$ an operator given by $(P^\pi J)(s) = \EE_{\pi}[J(s_1) \,|\, s_0 = s]$. The contraction assumption assumes the contraction property of the operator $P^\pi$ as follows.

\begin{assumption}\label{assumption:contraction}
For any policy $\pi$, there exists a constant $\beta_\pi \in [0,1)$ such that $\|P^\pi J\|_\pi \le \beta_\pi \| J \|_\pi$, where $\|J\|_\pi = \langle J, J\rangle _\pi$, for all $J:L_2(\cR^d, \cB(\cR^d), \pi)\rightarrow \RR$ that are orthogonal to $e=(1,1,...,1,1)^\top$.  
\end{assumption}

%\begin{assumption}[Bounded optimal $\lambda$ and $y$]\label{assumption:bounded-lambda-y}
%There exist two constant $N > 0$ and $M >0$ such that $|\lambda^*| \leq N$ and $|y^*| \leq M$, where $\lambda^*$ and $y^*$ are optimal $\lambda$ and $y$, respectively.  (seems unnecessary ?????)
%\end{assumption}

%\begin{assumption}[Bounded gradient]\label{assumption:bounded-gradient}
%	There exist two constant $L_\lambda > 0$ and $L_y >0$ such that $|\partial_\lambda\cL| \leq L_\lambda$ and $|\partial_y\cL| \leq L_y$.
%\end{assumption}
%%%%%%%%%%%%%%%%%%%%%%%%%%%%%%%%%%%%%%%%%%%%%%%%%%%%%%%%%%%%%%%%%%%%%%%%%%%%%%%%%
% !TEX root = main.tex

\subsection{Estimation Errors}\label{sec:estimate-error}
We first bound the estimation errors, where we provide the rates of convergence of the estimators $\overline{\rho}(\pi_{\theta_k})$ and $\overline{\eta}(\pi_{\theta_k})$ towards $\rho(\pi)$ and $\eta(\pi)$.

\begin{lemma}[Estimation Errors]\label{lem:average-error} 
	%For each $k\in[K]$, and for any $p\in (0,1)$, 
	For any $p\in (0,1)$, and for all $k\in[K]$, the estimators $\overline{\rho}(\pi_{\theta_k})$ and $\overline{\eta}(\pi_{\theta_k})$ in \eqref{eq:estimate:rho:eta} satisfy, with probability at least $1-p$,  
	\$
	 |\rho(\pi_{\theta_k}) - \overline{\rho}(\pi_{\theta_k}) |   \le \cO\bigl(T^{-1/2}\log(4K/p)^{1/2}\bigr),    \quad  |\eta(\pi_{\theta_k}) - \overline{\eta}(\pi_{\theta_k}) |  \le \cO\bigl(T^{-1/2}\log(4K/p)^{1/2}\bigr),
	\$
where $T$ is the simulated sample size. 
\end{lemma}
\begin{proof}
	Fix $k \in [K]$, by the bounded reward assumption and Azuma-Hoeffding inequality, 
	%we have that 
	%for any $\varepsilon >0$,
	%\$
	%p = \PP\bigl( | \rho(\pi_{\theta_k}) - \overline{\rho}(\pi_{\theta_k}\bigr) | \geq \varepsilon) \geq e^{-T\varepsilon^2/2M^2}.
	%\$
	%This is equivalent to 
	%\$
	%\varepsilon \leq M \cdot\bigl(T^{-1/2}\log(2/p)^{1/2}\bigr).
	%\$
	%Then it holds that, with probability at least $1-p/(2K)$,
	it holds with probability at least  $1-p/(2K)$ that
	\$
	| \rho(\pi_{\theta_k}) - \overline{\rho}(\pi_{\theta_k}) | \leq \cO\bigl(T^{-1/2}\log(4K/p)^{1/2}\bigr).
	\$
	Similarly, with probability at least $1-p/(2K)$, it holds that 
	\$
	| \eta(\pi_{\theta_k}) - \overline{\eta}(\pi_{\theta_k}) | \leq \cO\bigl(T^{-1/2}\log(4K/p)^{1/2}\bigr) .
	\$
	Together with the union bound argument, we complete the proof.
\end{proof}

By this lemma, in what follows, without loss of generality, we assume that the  errors satisfy that, for some $c_k,d_k >0$,
\#\label{eq:average error}
|\rho(\pi_{\theta_k}) - \overline{\rho}(\pi_{\theta_k}) | \le c_k,    \quad\qquad\qquad  |\eta(\pi_{\theta_k}) - \overline{\eta}(\pi_{\theta_k}) |  \le d_k.
\#

\subsection{Computation Errors}\label{sec:ac-error}

In this subsection, we bound the approximation errors of deep neural networks. First, in the following lemma, we characterize the error in the actor update step, which is induced by solving  subproblem  \eqref{eq:policy-update} using the SGD method in \eqref{eq:sgd-update1}. 
%See Line \ref{line:b} of Algorithm \ref{alg:risac} and Algorithm \ref{alg:ac-update} for a detailed algorithm. 

\begin{lemma}[$\pi$-Update Error]\label{thm:ac-error}
Suppose that Assumption \ref{assumption:closed-q-w} holds. Let  $\zeta = T^{-1/2}$, $H_{\rm a} = \cO (T^{1/4})$, $R_{\rm a} = \cO (m_{\rm a}^{1/2} H_{\rm a}^{-6}(\log m_{\rm a})^{-3})$ and $m_{\rm a} = \Omega(d^{3/2}R_{\rm a}^{-1} H_{\rm a}^{-3/2} \log^{3/2}(m_{\rm a}^{1/2}/R_{\rm a}))$. Then, at the $k$-th iteration of Algorithm \ref{alg:risac}, with probability at least $1 - \exp( - \Omega(R_{\rm a}^{2/3} m_{\rm a}^{2/3}H_a))$, the output $f_{\overline{\theta}}$ of Algorithm~\ref{alg:ac-update} satisfies
\$
 &\EE \bigl[ \bigl(f_{\overline\theta}(s,a) - \tau_{k+1}\cdot(\beta_k^{-1}(1 + 2\overline{\lambda}_{k}\overline{y}_{k})Q_{q_k} - \beta_k^{-1}\overline{\lambda}_{k}W_{\omega_{k}} + \tau^{-1}_{k}f_{\theta_k}(s,a))\bigr)^2 \bigr] \\
 &\quad= \cO ( R_{\rm a}^2 T^{-1/2} + R_{\rm a}^{8/3} m_{\rm a}^{-1/6} H_{\rm a}^{7} \log m_{\rm a} ),
\$
where the expectation is taken over $\overline \theta$ and $(s,a)\sim \sigma_{\pi_{\theta_k}}$, and $T$ is the iteration counter for the SGD method. 
\end{lemma}
\begin{proof}
	 See the proof of Proposition B.3 in \cite{fu2020single} for the detailed proof.
\end{proof}

Similarly, we characterize the computation errors in the critic update step, which are induced in $q$- and $\omega$-update steps in solving  subproblems in \eqref{eq:mspbe1} and \eqref{eq:mspbe2} using the TD updates in \eqref{eq:td-update1} and \eqref{eq:td-update2}. 
%See Line \ref{line:a} and \ref{line:b} of Algorithm \ref{alg:risac}, Algorithm \ref{alg:td} and Algorithm \ref{alg:td2} for details.

\begin{lemma}[$q$-Update Error]\label{thm:td}
Suppose that Assumptions \ref{assumption:closed-q-w} and \ref{assumption:contraction} hold. Let the parameters be that $\delta = T^{-1/2}$, $H_{\rm c} = \cO (T^{1/4})$, $R_{\rm c} = \cO (m_{\rm c}^{1/2} H_{\rm c}^{-6}(\log m_{\rm c})^{-3})$ and $m_{\rm c} = \Omega(d^{3/2}R_{\rm c}^{-1} H_{\rm c}^{-3/2} \log^{3/2}(m_{\rm c}^{1/2}/R_{\rm c}))$. Then, at the $k$-th iteration of Algorithm \ref{alg:risac}, with probability at least $1 - \exp( - \Omega(R_{\rm c}^{2/3} m_{\rm c}^{2/3}H_c))$, the output $Q_{\overline{q}}$ of Algorithm~\ref{alg:td} satisfies
\$
\EE\bigl[ \bigl(Q_{\overline{q}}(s,a) - Q^{\pi_{\theta_k}}(s,a)\bigr)^2 \bigr] = \cO ( R_{\rm c}^2 T^{-1/2} + R_{\rm c}^{8/3} m_{\rm c}^{-1/6} H_{\rm c}^{7} \log m_{\rm c} ),
\$
where the expectation is taken over $\overline q$ and $(s,a)\sim \sigma_{\pi_{\theta_k}}$, and $T$ is the iteration counter for the TD method. 
\end{lemma}
\begin{proof}
 See Appendix \ref{appendix:td} for the detailed proof.
\end{proof}

\begin{lemma}[$\omega$-Update Error]\label{thm:td2}
Suppose that Assumptions \ref{assumption:closed-q-w} and \ref{assumption:contraction} hold. Let  the parameters be that $\delta = T^{-1/2}$, $H_{\rm b} = \cO (T^{1/4})$, $R_{\rm b} = \cO (m_{\rm b}^{1/2} H_{\rm b}^{-6}(\log m_{\rm b})^{-3})$ and $m_{\rm b} = \Omega(d^{3/2}R_{\rm b}^{-1} H_{\rm b}^{-3/2} \log^{3/2}(m_{\rm b}^{1/2}/R_{\rm b}))$. Then, at the $k$-th iteration of Algorithm \ref{alg:risac}, with probability at least $1 - \exp( - \Omega(R_{\rm b}^{2/3} m_{\rm b}^{2/3}H_b ))$, the output $W_{\overline{\omega}}$ of Algorithm~\ref{alg:td2} satisfies
\$
\EE \bigl[ \bigl(W_{\overline{\omega}}(s,a) - W^{\pi_{\theta_k}}(s,a)\bigr)^2 \bigr] = \cO ( R_{\rm b}^2 T^{-1/2} + R_{\rm b}^{8/3} m_{\rm b}^{-1/6} H_{\rm b}^{7} \log m_{\rm b} ),
\$
where the expectation is taken over $\overline \omega$ and $(s,a)\sim \sigma_{\pi_{\theta_k}}$, and $T$ is the iteration counter for the TD method. 
\end{lemma}
\begin{proof}
  This proof is similar to the proof of Lemma \ref{thm:td}, and we omit it to avoid repetition.
\end{proof}

  %If the widths $m_{\rm a}$, $m_{\rm c}$ and $m_{\rm_b}$ of the DNNs $f_{\theta}$, $Q_{q}$ and $W_\omega$ are sufficiently large,  then the errors established in Lemmas \ref{thm:ac-error}, \ref{thm:td} and \ref{thm:td2} converges to zero at rates of $O (1/\sqrt{T})$. 
  
 Essentially, putting Lemmas \ref{thm:ac-error}, \ref{thm:td} and \ref{thm:td2} together, we establish that the computation errors incurred by fitting the DNNs  diminish at  rates of $\cO(T^{-1/2})$ if the network widths $m_{\rm a}$, $m_{\rm c}$ and $m_{\rm_b}$ of the DNNs $f_{\theta}$, $Q_{q}$ and $W_\omega$ are sufficiently large.

%  converge to zero at the rate of $\cO(T^{-1/2})$ if the widths $m_{\rm a}$, $m_{\rm c}$ and $m_{\rm_b}$ of the DNNs $f_{\theta}$, $Q_{q}$ and $W_\omega$ are sufficiently large.

%%%%%%%%%%%%%%%%%%%%%%%%%%%%%%%%%%%%%%%%%%%%%%%%%%%%%%%%%%%%%%%%%%%%%%%%%%%%%%%%%

\subsection{Error Propagation}\label{sec:error}
We then bound the policy error propagation at each iteration by analyzing the difference between our  policy update $\pi_{\theta_{k+1}}$  in \eqref{eq:policy-update} and an ideal policy update $\pi_{k+1}$ defined below in~\eqref{eq:pi-new-true}.  Recall that, as defined in \eqref{eq:pi-new}, $\hat{\pi}_{k+1}$ is a policy update based on $\overline{\lambda}_k$, $\overline{y}_k$, $Q_{q_k}$ and $W_{\omega_k}$, which are the estimators for the true $\lambda_k$, $y_k$, $Q^{\pi_{\theta_k}}$ and $W^{\pi_{\theta_k}}$, respectively. Correspondingly, we define the ideal policy update based on $\overline{\lambda}_k$, $\overline{y}_k$, $Q^{\pi_{\theta_k}}$ and $W^{\pi_{\theta_k}}$ as
\begin{equation}
\begin{aligned}\label{eq:pi-new-true}
\pi_{k+1} = \argmax_\pi\EE_{\nu_k}& [\la (1+2\overline{\lambda}_k\overline{y}_k)Q^{\pi_{\theta_k}}(s, \cdot) - \overline{\lambda}_kW^{\pi_{\theta_k}}(s, \cdot), \pi(\cdot, s)\ra \\
&- \beta_k \cdot {\rm KL}\bigl(\pi(\cdot\,|\,s)\,\|\, \pi_{\theta_k}(\cdot\,|\,s)\bigr)\bigr].
\end{aligned}
\end{equation}
By  Proposition \ref{prop:exp-policy}, we have a closed-form solution of $\pi_{k+1}$ that
\$
\pi_{k+1}\propto \exp\bigl(\beta_k^{-1}(1+2\overline{\lambda}_k\overline{y}_k)  Q^{\pi_{\theta_k}} - \beta_k^{-1}\overline{\lambda}_kW^{\pi_{\theta_k}}+ \tau_k^{-1} f_{\theta_k}\bigr).
\$

For ease of presentation, we adopt the following notations to denote density ratios of policies and stationary distributions,
\#\label{eq:w0w}
 \phi^*_{k} = \EE_{{\sigma}_k}[|{\ud \pi^*}/{\ud \pi_0} - {\ud \pi_{\theta_k}}/{\ud \pi_0}|^2]^{1/2},\quad
 \psi^*_k =  \EE_{\sigma_k}[|{\ud \sigma^*}/{\ud \sigma_k} - \ud \nu^*/\ud \nu_k|^2]^{1/2},
\#
where $\ud \pi^*/\ud \pi_0$, $\ud \pi_{\theta_k}/\ud \pi_0$, $\ud \sigma^*/\ud \sigma_k$, and $\ud \nu^*/\ud \nu_k$ are the Radon-Nikodym derivatives, and recall that we denote  the optimal policy as $\pi^*$,  its stationary state distribution as $\nu^*$, and  its stationary state-action distribution as  $\sigma^*$.

We then prove an important lemma for the error propagation, which essentially quantifies how the  errors of  policy update $\hat{\pi}_{k+1}$ in \eqref{eq:policy-update} and the  policy evaluation propagate into the infinite-dimensional policy space.

\begin{lemma}[Error Propagation]\label{lem:error-kl}
Suppose that the policy improvement error in Line \ref{line:sgd} of Algorithm \ref{alg:risac} satisfies
\#\label{eq:error-kl1}
\EE_{{\sigma}_k}\bigl[\bigl(f_{\theta_{k+1}}(s, a) - \tau_{k+1}\cdot(\beta_k^{-1}Q_{\omega_k}(s, a) - \tau^{-1}_kf_{\theta_k}(s, a))\bigr)^2 \bigr] \leq \epsilon_{k+1},
\#
and the policy evaluation error of Q-function in Line \ref{line:td:q} of Algorithm \ref{alg:risac} satisfies
\#\label{eq:error-kl2}
\EE_{\sigma_k}\bigl[\bigl(Q_{q_k}(s, a) - Q^{\pi_{\theta_k}}(s, a)\bigr)^2\bigr] \leq \epsilon'_{k},
\#
and the policy evaluation error of W-function in Line \ref{line:td:w} of Algorithm \ref{alg:risac} satisfies
\#\label{eq:error-kl3}
\EE_{\sigma_k}\bigl[\bigl(W_{\omega_k}(s, a) - W^{\pi_{\theta_k}}(s, a)\bigr)^2\bigr] \leq \epsilon_{k}''.
\#
For $\pi_{k+1}$ defined in \eqref{eq:pi-new-true} and $\pi_{\theta_{k+1}}$ obtained in Line \ref{line:sgd} of Algorithm \ref{alg:risac}, we have
\#\label{812759}
\bigl| \EE_{\nu^*}\big[ \big\la \log\big(\pi_{\theta_{k+1}}(\cdot\,|\,s)/\pi_{k+1}(\cdot\,|\,s)\big), \pi^*(\cdot\,|\,s)-\pi_{\theta_k}(\cdot\,|\,s) \big\ra]\bigr|\le \varepsilon_k,
\#
where $\varepsilon_k= \tau_{k+1}^{-1}\epsilon_{k+1}\cdot\phi^*_{k}+(1+2MN)\cdot\beta_k^{-1}\epsilon_{k}'\cdot\psi^*_{k} +N\cdot\beta_k^{-1}\epsilon_{k}''\cdot\psi^*_{k}.$
\end{lemma}
\begin{proof}See Appendix \ref{812354} for the detailed proof. \end{proof}

Recall that we consider energy-based policies, where the energy function $f_\theta$ is parametrized as a DNN. The next lemma characterizes the stepwise energy difference by quantifying the difference between 
$f_{\theta_{k+1}}$ and $f_{\theta_k}$.

\begin{lemma}[Stepwise Energy Difference]\label{lem:stepwise-energy}
Under the same assumptions of Lemma \ref{lem:error-kl}, we~have
\$
\EE_{\nu^*}[\| \tau_{k+1}^{-1} f_{\theta_{k+1}}(s,\cdot) - \tau_k^{-1}f_{\theta_k}(s,\cdot)  \|_{\infty}^2] \le 2\varepsilon_k'+ 2\beta_k^{-2}\hat{M},
\$
where $\varepsilon_k'=|\cA|\cdot \tau_{k+1}^{-2}\epsilon_{k+1}^2$ and $\hat{M}=4(1+2MN)^2 \cdot \EE_{\nu^*}[ \max_{a\in\cA} (Q_{q_0}(s,a))^2 + R_{\rm c}^2]+ 4N^2 \cdot \EE_{\nu^*}[ \max_{a\in\cA} (W_{\omega_0}(s,a))^2 + R_{\rm b}^2 ] .$
\end{lemma}
\begin{proof}See Appendix \ref{812354} for the detailed proof. \end{proof}

%The errors characterized in Lemmas \ref{lem:error-kl} and \ref{lem:stepwise-energy} play key roles in establishing the global convergence of VARAC.

\subsection{Global Convergence of VARAC}\label{sec:main}
In this subsection, we establish the global convergence of the VARAC algorithm. In particular, we  derive the convergence of the solution path, and  then show that, despite the nonconvexity of our problem, the solution path converges to a globally optimal solution.

%and the duality gap.
%In this paper, we measure the quality of the algorithm’s output via two ways. First, we consider the Lagrangian function’s value defined in \eqref{eq:def-L}. Then we also investigate the duality gap of the output of our algorithm, which is a standard metric commonly used in optimization literatures. 
%In both cases, we show that our algorithm converges to the saddle point at a rate of $1/\sqrt{K}$, where $K$ is the iterations count.

%We track the progress of VARAC in Algorithm \ref{alg:risac} using 
%$\cL(\lambda, \pi, y)$ defined in \eqref{eq:def-L}.

We first prove the convergence of the solution path  by showing that the objective of 
the Lagrangian function \eqref{eq:def-L} of the solution path  converges to the corresponding objective of a saddle point. 
Specifically, the following theorem characterizes the convergence of $\cL(\overline{\lambda}_k, \pi_{\theta_k}, \overline{y}_k)$ towards $\cL(\lambda^*, \pi^*, y^*)$. 
%Recall that $T_a$, $T_c$ and $T_b$ are the numbers of SGD and TD iterations in Lines \ref{line:c}, \ref{line:b} and \ref{line:a} of Algorithm \ref{alg:risac}.

\begin{theorem}[Approximate Saddle Point]\label{thm:main}
Suppose that Assumptions~\ref{assumption:unique-solution}, \ref{assumption:closed-q-w}, and \ref{assumption:contraction} hold. For the sequences $\{\overline{\lambda}_k\}^{K}_{k=1}$, $\{\pi_{\theta_k}\}^{K}_{k = 1}$ and $\{\overline{y}_k\}^{K}_{k=1}$ generated by the VARAC algorithm (Alg.~\ref{alg:risac}), we have  
\begin{equation}\label{eqn:m1}
\begin{aligned}
 %& - \sum_{k=0}^{K-1} (c_k + d_k)^2\cdot \cO(1/K\sqrt{K}) - \sum_{k=0}^{K-1} (c_k + d_k)\cdot \cO(1/K)  - \cO(1/K)\\
 & - \sum_{k=0}^{K-1} (c_k + d_k)\cdot \cO(1/K)  - \cO(1/\sqrt{K})\\
 &\qquad \leq \frac{1}{K}\sum_{k=0}^{K-1} \bigl(\cL(\lambda^*,\pi^*, y^*) - \cL(\overline{\lambda}_k, \pi_{\theta_k}, \overline{y}_k) \bigr)\\
 &\qquad \leq \bigl( \sum_{k=0}^{K-1}c_k\bigr) \cdot \cO(1/K) +   \sum_{k=0}^{K-1}(  \varepsilon_k + \varepsilon_k') \cdot  \cO(1/\sqrt{K}) + \cO(1/\sqrt{K}),
 \end{aligned}
\end{equation}
where $c_k$ and $d_k$ are estimation errors defined in \eqref{eq:average error}. Here $\varepsilon_k= \tau_{k+1}^{-1}\epsilon_{k+1}\cdot\phi^*_{k}+(1+2MN)\cdot\beta_k^{-1}\epsilon_{k}'\cdot\psi^*_{k} +N\cdot\beta_k^{-1}\epsilon_{k}''\cdot\psi^*_{k}$ and $\varepsilon_k'=|\cA|\cdot \tau_{k+1}^{-2}\epsilon_{k+1}^2$, where 
\$
&\epsilon_{k+1} = \cO ( R_{\rm a}^2 T^{-1/2} + R_{\rm a}^{8/3} m_{\rm a}^{-1/6} H_{\rm a}^{7} \log m_{\rm a} ), \qquad 
\epsilon'_{k} = \cO ( R_{\rm c}^2 T^{-1/2} + R_{\rm c}^{8/3} m_{\rm c}^{-1/6} H_{\rm c}^{7} \log m_{\rm c} ), \\
&\epsilon''_{k} = \cO ( R_{\rm b}^2 T^{-1/2} + R_{\rm b}^{8/3} m_{\rm b}^{-1/6} H_{\rm b}^{7} \log m_{\rm b} ).
\$
\end{theorem}

In what follows, we prove Theorem \ref{thm:main} through  a few lemmas. 
%Note that by Assumption \ref{assumption:bounded-reward}, we have 
%\# \label{eq:inequality:1}
%\rho(\pi_{\theta_{k}}) \le M, \quad \overline{\rho}(\pi_{\theta_{k}}) \le M, \quad  \eta(\pi_{\theta_{k}}) \le F, \quad  \overline{\eta}(\pi_{\theta_{k}}) \le F
%\#
%Meanwhile, by the updating rules \eqref{eq:update-lambda-form} and \eqref{eq:update-y-form}, we have 
%\# \label{eq:inequality:2}
 %\overline{\lambda}_k \le N, \qquad \overline{y}_k \le M .
%\#
 %The inequalities in \eqref{eq:inequality:1} and \eqref{eq:inequality:2} are frequently used in the following analysis. 
 We first present the performance difference lemma, which evaluates the difference in the values of the Lagrangian function \eqref{eq:def-L} for different policies.
\begin{lemma}[Performance Difference]\label{lem:v-diff}
	For $\cL(\lambda, \pi, y)$ defined in \eqref{eq:def-L}, we have
	\$
	\cL(\lambda, \pi^*, y) - \cL(\lambda, \pi, y) = \EE_{\nu^*}\big[\big\la (1 + 2\lambda y)Q^{\pi}(s,\cdot) - \lambda W^{\pi}(s, \cdot), \pi^*(\cdot\,|\,s) - \pi(\cdot\,|\,s)\big\ra\big],
	\$
	where $\nu^*$ is the stationary state distribution of the optimal policy $\pi^*$. 
\end{lemma}
\begin{proof} See Appendix \ref{appendix:main} for the detailed proof.\end{proof}

In the next two lemmas, we establish the one-step descent of the Lagrangian multiplier $\lambda$- and policy $\pi$-update steps, respectively. The key idea of the proof follows from the analysis of the mirror descent algorithm \citep{beck2003mirror,nesterov2013introductory}. 

\begin{lemma}[One-Step Descent of $\lambda$]\label{lem:main-descent-lambda}
	At the $k$-th iteration of Algorithm~\ref{alg:risac}, we have that 
	$\overline{\lambda}_k$ in \eqref{eq:update-lambda-form} and  the optimal solution $\lambda^*$ satisfy
	\#\label{eq:descent-lambda}
	&\| \lambda^* - \overline{\lambda}_k\|^2 - \|\lambda^* - \overline{\lambda}_{k+1} \|^2   \\
	&\qquad\geq -\frac{1}{\gamma_k} \cdot \bigl(\cL(\lambda^*, \pi_{\theta_k},\overline{y}_k) - \cL(\overline{\lambda}_k, \pi_{\theta_k},\overline{y}_k) \bigr)  - \frac{1}{\gamma_k}\cdot2N\cdot(d_k + 2Mc_k) - \frac{1}{4\gamma_k^2}\cdot(\alpha+ 4M^2)^2. \notag
	\#
\end{lemma}
   \begin{proof}
   	By the updating rule of $\lambda$ in \eqref{eq:update-lambda-form}, we have
   	\begin{equation}\label{eq:lambda110}
	\begin{aligned}
   	&\|\lambda^* - \overline{\lambda}_k\|^2 - \|\lambda^* - \overline{\lambda}_{k+1} \|^2    \\
   	&\qquad = - 2\la \lambda^* - \overline{\lambda}_{k+1}, \overline{\lambda}_k - \overline{\lambda}_{k+1} \ra +\|\overline{\lambda}_{k+1} - \overline{\lambda}_{k}\|^2   \\
   	& \qquad  \geq  - \bigl\la \lambda^* - \overline{\lambda}_{k+1}, \gamma_k^{-1}\cdot \bigl(\alpha - \overline{\eta}(\pi_{\theta_k}) -\overline{y}_k^2 + 2\overline{y}_k\overline{\rho}(\pi_{\theta_k})   \bigr) \bigr\ra +\|\overline{\lambda}_{k+1}  - \overline{\lambda}_{k}\| ^2  \\
   	& \qquad = -\frac{1}{\gamma_k}\cdot\la \lambda^* - \overline{\lambda}_{k}, \alpha - \overline{\eta}(\pi_{\theta_k}) -\overline{y}_k^2 +2\overline{y}_k\overline{\rho}(\pi_{\theta_k})    \ra   \\
   	&\qquad\qquad+\frac{1}{\gamma_k}\cdot\la \overline{\lambda}_{k+1} - \overline{\lambda}_{k},\alpha - \overline{\eta}(\pi_{\theta_k}) -\overline{y}_k^2 +2\overline{y}_k\overline{\rho}(\pi_{\theta_k})    \ra  
   	+\|\overline{\lambda}_{k+1}  - \overline{\lambda}_{k}\|^2, 
   	\end{aligned}
	\end{equation}
   	where the inequality follows from the non-expansiveness of the projection in \eqref{eq:update-lambda-form}. 	
	By the definition of $\cL(\lambda, \pi, y)$ in \eqref{eq:def-L}, we obtain
   	\begin{equation}\label{eq:lambda111}
	\begin{aligned}
   	&  -\frac{1}{\gamma_k}\cdot\la \lambda^* - \overline{\lambda}_{k}, \alpha - \overline{\eta}(\pi_{\theta_k}) -\overline{y}_k^2 +2\overline{y}_k\overline{\rho}(\pi_{\theta_k})    \ra   \\
   	& \qquad = -\frac{1}{\gamma_k}\cdot\la \lambda^* - \overline{\lambda}_{k}, \alpha - \eta(\pi_{\theta_k}) -\overline{y}_k^2 +2\overline{y}_k\rho(\pi_{\theta_k})   \ra   \\
   	&\qquad\qquad - \frac{1}{\gamma_k}\cdot \bigl\la \lambda^* - \overline{\lambda}_{k}, \eta(\pi_{\theta_k}) - \overline{\eta}(\pi_{\theta_k}) +2\overline{y}_k\bigl(\overline{\rho}(\pi_{\theta_k}) - \rho(\pi_{\theta_k})\bigr) \bigr\ra \\
   	& \qquad \geq -\frac{1}{\gamma_k} \cdot \bigl(\cL(\lambda^*, \pi_{\theta_k},\overline{y}_k) - \cL(\overline{\lambda}_k, \pi_{\theta_k},\overline{y}_k) \bigr)  - \frac{1}{\gamma_k}\cdot2N\cdot(d_k + 2Mc_k),
   	\end{aligned}
	\end{equation}
   	where the last inequality is obtained by \eqref{eq:bound:rho:eta}, \eqref{eq:bound:rho:eta:bar}, \eqref{eq:update-y-form}, \eqref{eq:average error} and the assumption $\lambda_k \le N$. Meanwhile, by \eqref{eq:bound:rho:eta:bar}, \eqref{eq:update-y-form}, and the inequality $2xy \geq -x^2 - y^2$, we have
   	\begin{equation}\label{eq:lambda112}
	\begin{aligned} 
   	& \bigl\la \overline{\lambda}_{k+1} - \overline{\lambda}_{k}, {\gamma_k}^{-1}\cdot \bigl(\alpha - \overline{\eta}(\pi_{\theta_k}) -\overline{y}_k^2 +2\overline{y}_k\overline{\rho}(\pi_{\theta_k}) \bigr) \bigr\ra   \\
   	& \qquad \geq -\|\overline{\lambda}_{k+1}  - \overline{\lambda}_{k}\|^2 - \frac{1}{4\gamma_k^2}\cdot \bigl(\alpha - \overline{\eta}(\pi_{\theta_k}) -\overline{y}_k^2 +2\overline{y}_k\overline{\rho}(\pi_{\theta_k})\bigr)^2 \\
   	& \qquad \geq -\|\overline{\lambda}_{k+1}  - \overline{\lambda}_{k}\|^2 - \frac{1}{4\gamma_k^2}\cdot(\alpha+ 4M^2)^2.
   	\end{aligned}
	\end{equation}
   	Plugging \eqref{eq:lambda111} and \eqref{eq:lambda112} into \eqref{eq:lambda110}, we have
   	\$ 
   	&\| \lambda^* - \overline{\lambda}_k\|^2 - \|\lambda^* - \overline{\lambda}_{k+1} \|^2   \\
   	&\qquad\geq \frac{1}{\gamma_k} \cdot \bigl(\cL(\lambda^*, \pi_{\theta_k},\overline{y}_k) - \cL(\overline{\lambda}_k, \pi_{\theta_k},\overline{y}_k) \bigr)   - \frac{1}{\gamma_k}\cdot2N\cdot(d_k + 2Mc_k) - \frac{1}{4\gamma_k^2}\cdot(\alpha+ 4M^2 )^2,
   	\$
   	which concludes the proof.   \end{proof}

\begin{lemma}[One-Step Descent of $\pi$]\label{lem:main-descent}
	For the oracle improved policy $\pi_{k+1}$ defined in \eqref{eq:pi-new-true} and the  policy $\pi_{\theta_{k}}$ generated by Algorithm~\ref{alg:risac}, we have that, for any $s \in \cS$,
	\$
	&{\rm KL}\bigl(\pi^*(\cdot\,|\,s)\,\|\,\pi_{\theta_{k+1}}(\cdot\,|\,s)\bigr) - {\rm KL}\bigl(\pi^*(\cdot\,|\,s)\,\|\,\pi_{\theta_{k}}(\cdot\,|\,s)\bigr)\\
	& \qquad\leq \la \log(\pi_{\theta_{k+1}}(\cdot\,|\,s)/\pi_{k+1}(\cdot\,|\,s)), \pi_{\theta_k}(\cdot\,|\,s)-\pi^*(\cdot\,|\,s) \ra \\
	&\qquad\qquad- \beta_k^{-1}\cdot\la(1+2y^*\overline{\lambda}_k)Q^{\pi_{\theta_k}}(s, \cdot) - \overline{\lambda}_kW^{\pi_{\theta_k}}(s, \cdot), \pi^*(\cdot\,|\,s) - \pi_{\theta_{k}}(\cdot\,|\,s)\ra\notag\\
	&\qquad\qquad + \beta_k^{-1}\cdot\la2(y^* - \overline{y}_k)\overline{\lambda}_kQ^{\pi_{\theta_k}}(s, \cdot), \pi^*(\cdot\,|\,s) - \pi_{\theta_{k}}(\cdot\,|\,s)\ra\notag\\
	&\qquad\qquad - \la \tau_{k+1}^{-1}f_{\theta_{k+1}}(s, \cdot) - \tau_k^{-1}f_{\theta_k}(s, \cdot), \pi_{\theta_{k}}(\cdot\,|\,s) - \pi_{\theta_{k+1}}(\cdot\,|\,s) \ra \\
	&\qquad\qquad - 1/2\cdot\|\pi_{\theta_{k+1}}(\cdot\,|\,s) - \pi_{\theta_{k}}(\cdot\,|\,s)\|_1^2.
	\$
\end{lemma}
\begin{proof}
	The proof is similar to the proof of Lemma \ref{lem:main-descent-lambda}, and we defer the details to~Appendix~\ref{appendix:main}.
\end{proof} 

Next, we derive an upper bound  of $\cL(\overline{\lambda}_k, \pi_{\theta_k}, y^*) - \cL(\overline{\lambda}_k, \pi_{\theta_k}, \overline{y}_k)$.
 \begin{lemma}\label{lem:main-descent-y}
	For the optimal solution $y^*$ and $\overline{y}_k$ obtained in \eqref{eq:update-y-form}, we have 
	\#\label{eq:812335}
	\cL(\overline{\lambda}_k, \pi_{\theta_k}, y^*) - \cL(\overline{\lambda}_k, \pi_{\theta_k}, \overline{y}_k)  \le 4MNc_k.
	\#
\end{lemma}
\begin{proof}
	By the definition of $\cL(\lambda,\pi,y)$ in \eqref{eq:def-L}, we have
	\$
	&\cL(\overline{\lambda}_k, \pi_{\theta_k}, y^*) - \cL(\overline{\lambda}_k, \pi_{\theta_k}, \overline{y}_k)   \\
	& \qquad=\overline{\lambda}_k\cdot \la y^*-  \overline{y}_k, 2\rho(\pi_{\theta_k}) - y^* -  \overline{y}_k \ra \notag\\
	&\qquad=2\overline{\lambda}_k \cdot\la y^*-  \overline{y}_k, \rho(\pi_{\theta_k}) -  \overline{y}_k \ra -\overline{\lambda}_k \cdot (y^*-  \overline{y}_k)^2 .
	\$
	By \eqref{eq:update-y-form} and \eqref{eq:average error}, we have $|\rho(\pi_{\theta_k}) -  \overline{y}_k| = |\rho(\pi_{\theta_k}) - \overline{\rho}(\pi_{\theta_k})| \le c_k$. Combined with   \eqref{eq:update-lambda-form}, \eqref{eq:update-y-form} and the fact that $(y^*-  \overline{y}_k)^2$ is nonnegative, we  have
	\#\label{eq:l09}
	\cL(\overline{\lambda}_k, \pi_{\theta_k}, y^*) - \cL(\overline{\lambda}_k, \pi_{\theta_k}, \overline{y}_k)  \le 4MNc_k,
	\#
	which concludes the proof.
\end{proof}

Combining  Lemma \ref{lem:main-descent} and Lemma \ref{lem:main-descent-y},  we derive an upper bound of $\cL(\overline{\lambda}_k,\pi^*, y^*) - \cL(\overline{\lambda}_k,\pi_{\theta_k}, \overline{y}_k)$ in the next lemma.

\begin{lemma}\label{lem:descent:lambda:pi} 
	For the sequences $\{\overline{\lambda}_k\}^{K}_{k=1}$, $\{\pi_{\theta_k}\}^{K}_{k = 1}$, and $\{\overline{y}_k\}^{K}_{k=1}$ generated by the VARAC algorithm, we have  
	\$
	&\beta_k^{-1}\cdot \bigl(\cL(\overline{\lambda}_k,\pi^*, y^*) - \cL(\overline{\lambda}_k,\pi_{\theta_k}, \overline{y}_k)\bigr) \\
	&\qquad \le 
	\EE_{\nu^*} \bigl[{\rm KL}\bigl(\pi^*(\cdot\,|\,s)\,\|\,\pi_{\theta_{k}}(\cdot\,|\,s)\bigr)\bigr] - \EE_{\nu^*}\bigl[{\rm KL}\bigl(\pi^*(\cdot\,|\,s)\,\|\,\pi_{\theta_{k+1}}(\cdot\,|\,s)\bigr)\bigr] \notag\\
	&\qquad\qquad +\beta_k^{-2} \hat{M} +\beta_k^{-1}\cdot8MNc_k+ \varepsilon_k + \varepsilon_k'.  \notag
	\$
 \end{lemma}
\begin{proof}
   	Taking expectation of ${\rm KL}(\pi^*(\cdot\,|\,s)\,\|\,\pi_{\theta_{k+1}}(\cdot\,|\,s)) - {\rm KL}(\pi^*(\cdot\,|\,s)\,\|\,\pi_{\theta_{k}}(\cdot\,|\,s))$ with respect to $s\sim\nu^*$, and by Lemma \ref{lem:error-kl} and Lemma \ref{lem:main-descent}, we have
   \$
   &\EE_{\nu^*}\bigl[{\rm KL}\bigl(\pi^*(\cdot\,|\,s)\,\|\,\pi_{\theta_{k+1}}(\cdot\,|\,s)\bigr)\bigr] - \EE_{\nu^*}\bigl[{\rm KL}\bigl(\pi^*(\cdot\,|\,s)\,\|\,\pi_{\theta_{k}}(\cdot\,|\,s)\bigr)\bigr] \notag\\
   &\qquad\le
   \varepsilon_{k} - \beta_k^{-1}\cdot \EE_{\nu^*}[\la(1+2y^*\overline{\lambda}_k)Q^{\pi_{\theta_k}}(s, \cdot) - \overline{\lambda}_kW^{\pi_{\theta_k}}(s, \cdot), \pi^*(\cdot\,|\,s) - \pi_{\theta_{k}}(\cdot\,|\,s)\ra] \notag\\
   %&\quad\qquad \EE_{\nu^*}[\la(2(\overline{y}_k - y^*)\overline{\lambda}_k)Q^{\pi_{\theta_k}}(s, \cdot), \pi^*(\cdot\,|\,s) - \pi_{\theta_{k}}(\cdot\,|\,s)\ra] \norag\\
   &\qquad\qquad + \beta_k^{-1}\cdot \EE_{\nu^*}[\la2(y^*-\overline{y}_k)\overline{\lambda}_kQ^{\pi_{\theta_k}}(s, \cdot), \pi^*(\cdot\,|\,s) - \pi_{\theta_{k}}(\cdot\,|\,s)\ra] \notag\\
   &\qquad\qquad
   - \EE_{\nu^*}[\la \tau_{k+1}^{-1}f_{\theta_{k+1}}(s, \cdot) - \tau_k^{-1}f_{\theta_k}(s, \cdot),  \pi_{\theta_k}(\cdot\,|\,s)- \pi_{\theta_{k+1}}(\cdot\,|\,s)  \ra] \notag\\
   &\qquad\qquad - 1/2\cdot\EE_{\nu^*}[ \|\pi_{\theta_{k+1}}(\cdot\,|\,s) - \pi_{\theta_{k}}(\cdot\,|\,s)\|^2_1 ],
   \$
   where $\varepsilon_k$ is defined in  Lemma \ref{lem:error-kl}.
   
   By Lemma \ref{lem:v-diff} and the H{\"o}lder's inequality, we further have
   \#\label{812209}
   &\EE_{\nu^*}\bigl[{\rm KL}\bigl(\pi^*(\cdot\,|\,s)\,\|\,\pi_{\theta_{k+1}}(\cdot\,|\,s)\bigr)\bigr] - \EE_{\nu^*}\bigl[{\rm KL}\bigl(\pi^*(\cdot\,|\,s)\,\|\,\pi_{\theta_{k}}(\cdot\,|\,s)\bigr)\bigr] \notag\\
   &\qquad\le
   \varepsilon_{k} - \beta_k^{-1}\cdot\bigl(\cL(\overline{\lambda}_k,\pi^*, y^*) - \cL(\overline{\lambda}_k,\pi_{\theta_k}, y^*)\bigr)   +\beta_k^{-1}\cdot\bigl(2\overline{\lambda}_k(\overline{y}_k-y^*)\bigr)\cdot\bigl(\rho(\pi^*) - \rho(\pi_{\theta_{k}})\bigr)  \notag\\
   &\qquad\qquad  +\EE_{\nu^*}[\| \tau_{k+1}^{-1}f_{\theta_{k+1}}(s,\cdot)- \tau_{k}^{-1}f_{\theta_{k}}(s,\cdot)\|_{\infty}\cdot \|  \pi_{\theta_k}(\cdot\,|\,s)- \pi_{\theta_{k+1}}(\cdot\,|\,s) \|_1]\notag\\
   &\qquad\qquad - 1/2\cdot\EE_{\nu^*}[ \|\pi_{\theta_{k+1}}(\cdot\,|\,s) - \pi_{\theta_{k}}(\cdot\,|\,s)\|^2_1 ] \notag\\
   &\qquad\le
   \varepsilon_{k} - \beta_k^{-1}\cdot\bigl(\cL(\overline{\lambda}_k,\pi^*, y^*) - \cL(\overline{\lambda}_k,\pi_{\theta_k}, y^*)\bigr)   +\beta_k^{-1}\cdot\bigl(2\overline{\lambda}_k(\overline{y}_k-y^*)\bigr)\cdot\bigl(\rho(\pi^*) - \rho(\pi_{\theta_{k}})\bigr)  \notag\\
   &\qquad\qquad + 1/2 \cdot \EE_{\nu^*}[ \| \tau_{k+1}^{-1}f_{\theta_{k+1}}(s,\cdot)- \tau_{k}^{-1}f_{\theta_{k}}(s,\cdot)\|_{\infty}^2 ]\notag\\
   &\qquad\le
   \varepsilon_{k} - \beta_k^{-1}\cdot\bigl(\cL(\overline{\lambda}_k,\pi^*, y^*) - \cL(\overline{\lambda}_k,\pi_{\theta_k}, y^*)\bigr)   \notag\\
   &\qquad\qquad +\beta_k^{-1}\cdot\bigl(2\overline{\lambda}_k(\overline{y}_k-y^*)\bigr)\cdot\bigl(\rho(\pi^*) - \rho(\pi_{\theta_{k}})\bigr) + (\varepsilon_k'+\beta_k^{-2}\hat{M}), \#
   where  the second inequality holds by the fact that $2xy - y^2 \leq x^2$ for any $x,y\in\RR$,  and the last inequality holds by  Lemma~\ref{lem:stepwise-energy}. Rearranging the terms in \eqref{812209}, we have
   \begin{equation}\label{812334}
   \begin{aligned}
   &\beta_k^{-1}\cdot \bigl(\cL(\overline{\lambda}_k,\pi^*, y^*) - \cL(\overline{\lambda}_k,\pi_{\theta_k}, y^*)\bigr) \\
   &\qquad \le 
   \EE_{\nu^*}\bigl[{\rm KL}\bigl(\pi^*(\cdot\,|\,s)\,\|\,\pi_{\theta_{k}}(\cdot\,|\,s)\bigr)\bigr] - \EE_{\nu^*}\bigl[{\rm KL}\bigl(\pi^*(\cdot\,|\,s)\,\|\,\pi_{\theta_{k+1}}(\cdot\,|\,s)\bigr)\bigr] \\
   &\qquad\qquad + 2\beta_k^{-1}\cdot\overline{\lambda}_k\cdot(\overline{y}_k-y^*)\cdot\bigl(\rho(\pi^*) - \rho(\pi_{\theta_{k}})\bigr)
   +\beta_k^{-2} \hat{M} + \varepsilon_k + \varepsilon_k'.
   \end{aligned}
   \end{equation}
   %   By Lemma \ref{lem:main-descent-y}, we have
%   \#\label{eq:812335}
%   \cL(\overline{\lambda}_k, \pi_{\theta_k}, y^*) - \cL(\overline{\lambda}_k, \pi_{\theta_k}, \overline{y}_k)  \le 4MNc_k .
%   \#
Furthermore, by the definition that $\overline{y}_k = \overline{\rho}(\pi_{\theta_k})$ and $y^*=\rho(\pi^*)$, we have 
   \$
   &\overline{\lambda}_k\cdot (\overline{y}_k-y^*)\cdot\bigl(\rho(\pi^*) - \rho(\pi_{\theta_{k}})\bigr)  \\
   &\qquad=\overline{\lambda}_k\cdot  \bigl[  (\overline{y}_k-y^*)\cdot\bigl(  \overline{\rho}(\pi_{\theta_{k}})- \rho(\pi_{\theta_{k}})\bigr)
   -\bigl(\rho(\pi^*) - \overline{\rho}(\pi_{\theta_{k}})\bigr)^2   \bigr]  \notag \\
   &\qquad \le \overline{\lambda}_k \cdot (\overline{y}_k-y^*)\cdot\bigl(  \overline{\rho}(\pi_{\theta_{k}})- \rho(\pi_{\theta_{k}})\bigr),  \notag 
   \$
   where the inequality holds by the fact that $  \overline{\lambda}_k(\overline{y}_k-y^*)^2$ is nonnegative. By  \eqref{eq:update-lambda-form}, \eqref{eq:update-y-form} and \eqref{eq:average error}, we further have $\overline{\lambda}_k \le N$, $ | \overline{y}_k-y^* | \le 2M$, and $ | \overline{\rho}(\pi_{\theta_{k}})- \rho(\pi_{\theta_{k}}) | \le c_k$. Hence, we obtain
   \#\label{eq:812444}
   \overline{\lambda}_k \cdot (\overline{y}_k-y^*)\cdot\bigl(\rho(\pi^*) - \rho(\pi_{\theta_{k}})\bigr) \le 2MNc_k, 
   \# 
   where the inequality holds by \eqref{eq:update-y-form}, \eqref{eq:average error}, and the fact that $  \overline{\lambda}_k(\overline{y}_k-y^*)^2$ is nonnegative.
   Plugging \eqref{eq:812335} and \eqref{eq:812444}  into \eqref{812334}, we obtain
   \begin{equation}\label{eq:812336}
   \begin{aligned}
   &\beta_k^{-1}\cdot \bigl(\cL(\overline{\lambda}_k,\pi^*, y^*) - \cL(\overline{\lambda}_k,\pi_{\theta_k}, \overline{y}_k)\bigr) \\
   &\qquad \le 
   \EE_{\nu^*}\bigl[{\rm KL}\bigl(\pi^*(\cdot\,|\,s)\,\|\,\pi_{\theta_{k}}(\cdot\,|\,s)\bigr)\bigr] - \EE_{\nu^*}\bigl[{\rm KL}\bigl(\pi^*(\cdot\,|\,s)\,\|\,\pi_{\theta_{k+1}}(\cdot\,|\,s)\bigr)\bigr] \\
   &\qquad\qquad +\beta_k^{-2} \hat{M} +\beta_k^{-1}\cdot 10 MNc_k+ \varepsilon_k + \varepsilon_k',  
   \end{aligned}
   \end{equation}
   which concludes the proof.
\end{proof}

Now, we are ready to prove Theorem \ref{thm:main} by casting the VARAC algorithm as an infinite-dimensional mirror descent with primal and dual errors.

\begin{proof}[Proof of Theorem \ref{thm:main}] 
	We show the convergence in two steps by showing the first and second inequalities in \eqref{eqn:m1}, respectively. \\
	\noindent\textbf{Part 1.}
	%By telescoping \eqref{eq:descent-lambda} for $k+1\in [K]$, we obtain 
	%\# \label{eq:lambda111}
	%&\frac{1}{K}\sum_{k=0}^{K-1} ( (\cL(\lambda^*, \pi_{\theta_k},\overline{y}_k) - \cL(\overline{\lambda}_k, \pi_{\theta_k},\overline{y}_k)) \\
	%&\quad \geq\frac{\|\lambda^* - \overline{\lambda}_K\|^2 - \|\lambda^* - \overline{\lambda}_0\|^2 - \sum_{k=0}^{K-1}\gamma_k^{-1}\cdot2N\cdot(d_k + 2Mc_k)   }{\sum_{k=0}^{K-1}\gamma_k^{-1}}     \notag\\
	%&\quad\qquad - \frac{\sum_{k=0}^{K-1}\gamma_k^{-2}(\alpha+ F + 3M^2 + 2Mc_k + d_k)^2 }{4\sum_{k=0}^{K-1}\gamma_k^{-1}} .\notag
	%\#
	Letting $\gamma_k =\gamma\sqrt{K}$ and telescoping \eqref{eq:descent-lambda} for $k + 1 \in [K]$, we have
	\begin{equation}\label{eq:result-part1}
	\begin{aligned}
	&\frac{1}{K}\sum_{k=0}^{K-1} \bigl( (\cL(\lambda^*, \pi_{\theta_k},\overline{y}_k) - \cL(\overline{\lambda}_k, \pi_{\theta_k},\overline{y}_k)\bigr) \\
	&\qquad\geq \gamma \cdot \frac{\|\lambda^* - \overline{\lambda}_K\|^2 - \|\lambda^* - \overline{\lambda}_0\|^2}{\sqrt{K}} - \frac{ 2N\sum_{k=0}^{K-1} (d_k + 2Mc_k)}{K}  - \frac{(\alpha+  4M^2)^2  }{4\gamma  \sqrt{K}}  \\
	&\qquad \geq  -\frac{\gamma \cdot   \|\lambda^* - \overline{\lambda}_0\|^2}{\sqrt{K}} - \frac{ 2N\sum_{k=0}^{K-1} (d_k + 2Mc_k)}{K}  - \frac{(\alpha+  4M^2)^2  }{4\gamma  \sqrt{K}}  \\
	&\qquad =  - \sum_{k=0}^{K-1} (c_k + d_k)\cdot \cO(1/K)  - \cO(1/\sqrt{K}), 
	\end{aligned}
	\end{equation}
	where the second inequality holds by the fact that $\|\lambda^* - \overline{\lambda}_K\|^2$ is nonnegative. By the definition of saddle-point  that $\cL(\lambda^*, \pi_{\theta_k},\overline{y}_k) \leq \cL(\lambda^*, \pi^*, y^*) $, we complete the proof of the first part of Theorem~\ref{thm:main}.

	\vskip5pt
	\noindent\textbf{Part 2.}
	By telescoping \eqref{eq:812336} for $k + 1\in [K]$, we obtain
	\$
	&\sum_{k=0}^{K-1} \beta_k^{-1}\cdot\bigl(\cL(\overline{\lambda}_k,\pi^*, y^*) - \cL(\overline{\lambda}_k, \pi_{\theta_k}, \overline{y}_k)\bigr)\notag\\
	&\qquad \le 
	\EE_{\nu^*}\bigl[{\rm KL}\bigl(\pi^*(\cdot\,|\,s)\,\|\,\pi_{\theta_{0}}(\cdot\,|\,s)\bigr)\bigr] - \EE_{\nu^*}\bigl[{\rm KL}\bigl(\pi^*(\cdot\,|\,s)\,\|\,\pi_{\theta_{K}}(\cdot\,|\,s)\bigr)\bigr] \notag\\
	&\qquad\qquad + \sum_{k=0}^{K-1}(\beta_k^{-2} \hat{M} +\beta_k^{-1}\cdot10MNc_k+ \varepsilon_k + \varepsilon_k')  .
	\$
	Note that we have (i) $\EE_{\nu^*}[{\rm KL}(\pi^*(\cdot\,|\,s)\,\|\,\pi_{\theta_0}(\cdot\,|\,s))] \leq \log|\cA|$ due to the uniform initialization of policy, and  (ii) the KL-divergence is nonnegative. Setting  $\beta_k = \beta\sqrt{K},$  we have
	\begin{equation} \label{eq:result-part2}
	\begin{aligned}
	&\frac{1}{K}\sum_{k=0}^{K-1} \bigl(\cL(\overline{\lambda}_k,\pi^*, y^*) - \cL(\overline{\lambda}_k, \pi_{\theta_k}, \overline{y}_k)\bigr) \\
	&\qquad\le \frac{10MN\sum_{k=0}^{K-1}c_k}{K} +
	\frac{\beta\log|\cA| +  \beta^{-1} \hat{M} +\sum_{k=0}^{K-1}(  \varepsilon_k + \varepsilon_k')}{\sqrt{K}}  \\
	&\qquad   =\sum_{k=0}^{K-1}c_k \cdot \cO(1/K) +   \sum_{k=0}^{K-1}(  \varepsilon_k + \varepsilon_k') \cdot  \cO(1/\sqrt{K}) + \cO(1\sqrt{K})                   .
	\end{aligned}
	\end{equation}
By the definition of saddle point that $\cL(\overline{\lambda}_k,\pi^*, y^*) \geq \cL(\lambda^*, \pi^*, y^*) $, we conclude the proof.
\end{proof}
%\begin{proof}
%See Section \ref{sec:sketch} for the detailed proof.
%\end{proof}

%To further elaborate Theorem \ref{thm:main}, in the following corollary we choose proper parameters to ensure the $\c\cO(1/\sqrt{K})$ rate of convergence.
By optimizing the input parameters, we obtain the $\cO(1/\sqrt{K})$ rate of convergence in the following corollary.

\begin{corollary}\label{coro:main-bound}
Suppose that Assumptions \ref{assumption:unique-solution}, \ref{assumption:closed-q-w}, and \ref{assumption:contraction} hold. Let $R_{\rm a}  = R_{\rm b} = R_{\rm c} = \cO(m_{\rm a}^{1/2} H_{\rm a}^{-6}(\log m_{\rm a})^{-3}),$ $T = \Omega( K^{3} (\phi^*_{k}+\psi^*_{k})^2 |\cA| R_{\rm a}^4H_am_a^{2/3} )$, $m_{\rm a} = m_{\rm b} = m_{\rm c} = \Omega( d^{3/2}K^{9}(\phi^*_{k}+\psi^*_{k})^6|\cA|^3 R_{\rm a}^{16} H_{\rm a}^{42} \log^6 m_{\rm a} )$ and $p=\exp( - \Omega(R_{\rm a}^{2/3} m_{\rm a}^{2/3}H_a ))$ for any $0\le k \le K$. With probability at least $1 - 4\exp( - \Omega(R_{\rm a}^{2/3} m_{\rm a}^{2/3}H_a ))$, we have
\$
\frac{1}{K} \bigg| \sum_{k=0}^{K-1}  \bigl(\cL(\lambda^*,\pi^*, y^*) - \cL(\overline{\lambda}_k, \pi_{\theta_k}, \overline{y}_k)\bigr) \bigg| \leq \cO(1/\sqrt{K}).
\$
\end{corollary}
\begin{proof}
See Appendix \ref{appendix:main} for the  detailed proof. 
\end{proof}

Finally, we show in the next theorem about the convergence of the solution path to a globally optimal solution at an $\cO(1/\sqrt{K})$ rate despite the nonconvexity of problem~\eqref{eq:problem}. This shows that  the VARAC algorithm converges to a globally optimal solution.

\begin{theorem}[Global Convergence]\label{thm:duality-gap}
	Suppose that Assumptions~\ref{assumption:unique-solution}, \ref{assumption:closed-q-w}, and \ref{assumption:contraction} hold. For the sequences $\{\overline{\lambda}_k\}^{K}_{k=1}$, $\{\pi_{\theta_k}\}^{K}_{k = 1}$ and $\{\overline{y}_k\}^{K}_{k=1}$ generated by the VARAC algorithm, we have  
	\$
	 0& \leq \frac{1}{K}\sum_{k=0}^{K-1} \bigl( \cL(\overline{\lambda}_k,\pi^*, y^*) - \cL(\lambda^*, \pi_{\theta_k}, \overline{y}_k) \bigr)\\
	& \leq  \sum_{k=0}^{K-1}(c_k + d_k)\cdot \cO(1/K) + \sum_{k = 0}^{K - 1}(\varepsilon_k + \varepsilon_k') \cdot \cO(1/\sqrt{K}) + \cO(1/\sqrt{K}).
	\$
	Moreover, if we set the input parameters same as Corollary \ref{coro:main-bound}, it holds that, with probability at least $1 - 4\exp( - \Omega(R_{\rm a}^{2/3} m_{\rm a}^{2/3}H_a ))$,
	\$
	\frac{1}{K} \bigg| \sum_{k=0}^{K-1} \bigl( \cL(\overline{\lambda}_k,\pi^*, y^*) - \cL(\lambda^*, \pi_{\theta_k}, \overline{y}_k)\bigr) \bigg| \le \cO(1/\sqrt{K}).
	\$
\end{theorem}
\begin{proof}
	See Appendix \ref{appendix:main} for the detailed proof. 
\end{proof}
%%%%%%%%%%%%%%%%%%%%%%%%

\subsection{Stronger Results Under Slater Condition}
We then establish a stronger result that under the Slater condition that problem \eqref{eq:problem} is strictly feasible, the optimality and feasibility gaps both diminish  $\cO(1/\sqrt{K})$ rates.

\begin{assumption}[Slater Condition] \label{assumption:slater}
	There exists $\xi > 0$ and $\bar{\pi}$ such that $\alpha - \Lambda(\bar{\pi}) \geq \xi$.
\end{assumption}

The Slater condition in Assumption \ref{assumption:slater} is mild in practice and commonly adopted in the previous literature on constrained optimization \citep{bertsekas2014constrained} and constrained RL \citep{altman1999constrained,paternain2019safe,paternain2019constrained,efroni2020exploration,ding2020natural,ding2021provably,chen2021primal}. With Assumption \ref{assumption:slater}, we can characterize the boundedness of the optimal Lagrangian dual variable $\lambda^*$ as follows. 

\begin{lemma}[Boundedness of $\lambda^*$] \label{lem:boundedness:lambda}
	Suppose Assumption \ref{assumption:slater} holds, then the optimal Lagrangian dual variable $\lambda^*$ satisfies that $0 \le \lambda^* \le (\rho(\pi^*) - \rho(\bar{\pi}))/\xi$. 
\end{lemma}
  
\begin{proof}
	See \citet{paternain2019safe,paternain2019constrained} for a detailed proof.
\end{proof}

Together with \eqref{eq:bound:rho:eta}, Lemma \ref{lem:boundedness:lambda} shows that $\lambda^* \in [0, M/\xi]$. Inspired by this, we  choose $N = 2M/\xi$ in \eqref{eq:update-lambda-form}. With the Slater condition (Assumption \ref{assumption:slater}), we derive the convergence rates of optimality and feasibility gaps in 
the following theorem.

\begin{theorem}[Constraint Violation] \label{thm:main:slater}
	Suppose that Assumptions~\ref{assumption:unique-solution}, \ref{assumption:closed-q-w}, \ref{assumption:contraction}, and \ref{assumption:slater} hold. Let $N = 2M/\xi$ in \eqref{eq:update-lambda-form}. For the sequences $\{\overline{\lambda}_k\}^{K}_{k=1}$, $\{\pi_{\theta_k}\}^{K}_{k = 1}$ and $\{\overline{y}_k\}^{K}_{k=1}$ generated by the VARAC algorithm (Alg.~\ref{alg:risac}), we have  
%\begin{equation}\label{eqn:m2}
%\begin{aligned}
	\$
	\rho(\pi^*) - \frac{1}{K}\sum_{k=0}^{K-1}  \rho(\pi_{\theta_k})  &\le  \sum_{k=0}^{K-1}(c_k + d_k) \cdot \cO(1/K) +   \sum_{k=0}^{K-1}(  \varepsilon_k + \varepsilon_k') \cdot  \cO(1/\sqrt{K}) + \cO(1/\sqrt{K}), \\
	 \Bigl[ \frac{1}{K}\sum_{k = 0}^{K - 1}\Lambda(\pi_{\theta_k}) - \alpha  \Bigr]_+ &\le \sum_{k=0}^{K-1}(c_k + d_k) \cdot \cO(1/K) +   \sum_{k=0}^{K-1}(  \varepsilon_k + \varepsilon_k') \cdot  \cO(1/\sqrt{K}) + \cO(1/\sqrt{K}).
	\$            
 %\end{aligned}
%\end{equation}
where $c_k$ and $d_k$ are estimation errors defined in \eqref{eq:average error}. Here $\varepsilon_k= \tau_{k+1}^{-1}\epsilon_{k+1}\cdot\phi^*_{k}+(1+2MN)\cdot\beta_k^{-1}\epsilon_{k}'\cdot\psi^*_{k} +N\cdot\beta_k^{-1}\epsilon_{k}''\cdot\psi^*_{k}$ and $\varepsilon_k'=|\cA|\cdot \tau_{k+1}^{-2}\epsilon_{k+1}^2$, where 
\$
&\epsilon_{k+1} = \cO ( R_{\rm a}^2 T^{-1/2} + R_{\rm a}^{8/3} m_{\rm a}^{-1/6} H_{\rm a}^{7} \log m_{\rm a} ), \qquad 
\epsilon'_{k} = \cO ( R_{\rm c}^2 T^{-1/2} + R_{\rm c}^{8/3} m_{\rm c}^{-1/6} H_{\rm c}^{7} \log m_{\rm c} ), \\
&\epsilon''_{k} = \cO ( R_{\rm b}^2 T^{-1/2} + R_{\rm b}^{8/3} m_{\rm b}^{-1/6} H_{\rm b}^{7} \log m_{\rm b} ).
\$
Moreover, if we set the input parameters same as Corollary \ref{coro:main-bound}, it holds that, with probability at least $1 - 4\exp( - \Omega(R_{\rm a}^{2/3} m_{\rm a}^{2/3}H_a ))$,
	\$
	\rho(\pi^*) - \frac{1}{K}\sum_{k=0}^{K-1}  \rho(\pi_{\theta_k})  \le   \cO(1/\sqrt{K}), \qquad \Bigl[ \frac{1}{K}\sum_{k = 0}^{K - 1}\Lambda(\pi_{\theta_k}) - \alpha  \Bigr]_+ \le  \cO(1/\sqrt{K}). 
	\$
\end{theorem}

\begin{proof}[Proof of Theorem \ref{thm:main:slater}]
	Recall that $\cL(\pi, \lambda, y)$ takes the following form 
	\$
	\cL(\lambda,\pi,y) = (1 + 2\lambda y)\rho(\pi) - \lambda\eta(\pi) - \lambda y^2 + \lambda\alpha .
	\$
	With slight abuse of notation, we define 
	\# \label{eq:52001}
	\cL(\lambda, \pi) = \rho(\pi) - \lambda\bigl(\Lambda(\pi) - \alpha\bigr)
	\#
	Together with the fact that $y^* = \rho(\pi^*)$, for any $k \in [K]$, we have
	\# \label{eq:52002}
	\cL(\overline{\lambda}_k,\pi^*, y^*) = \rho(\pi^*) - \overline{\lambda}_k\bigl(\Lambda(\pi^*) - \alpha\bigr) = \cL(\overline{\lambda}_k, \pi^*).
	\#
	Moreover, for any $k \in [K]$, we have 
	\# \label{eq:52003}
	|\cL(\overline{\lambda}_k, \pi_{\theta_k}, \overline{y}_k) - \cL(\overline{\lambda}_k, \pi_{\theta_k})| &= |2\overline{\lambda}_k \bigl( \rho(\pi_{\theta_k}) - \overline{y}_k \bigr) - \overline{\lambda}_k \bigl( \rho(\pi_{\theta_k}) + \overline{y}_k \bigr) \bigl( \rho(\pi_{\theta_k}) - \overline{y}_k \bigr) | \notag\\
	& \le  2\overline{\lambda}_k | \rho(\pi_{\theta_k}) - \overline{y}_k | + \overline{\lambda}_k \bigl( \rho(\pi_{\theta_k}) + \overline{y}_k \bigr) | \rho(\pi_{\theta_k}) - \overline{y}_k |  \notag\\
	& \le 2N (1 + M) c_k,
	\#
	where the first inequality follows from triangle inequality and the last inequality uses the definition of $c_k$ in \eqref{eq:average error}.
	Plugging \eqref{eq:52002} and \eqref{eq:52003} into \eqref{eq:result-part2}, we obtain
	\#  \label{eq:52004}
	%& - \sum_{k=0}^{K-1} (c_k + d_k)\cdot \cO(1/K)  - \cO(1/\sqrt{K})\\
	& \frac{1}{K}\sum_{k=0}^{K-1} \bigl(\cL(\overline{\lambda}_k,\pi^*) - \cL(\overline{\lambda}_k, \pi_{\theta_k}) \bigr) \notag\\
    %&\qquad \leq \frac{1}{K}\sum_{k=0}^{K-1} \bigl(\cL(\lambda^*,\pi^*) - \cL(\overline{\lambda}_k, \pi_{\theta_k}) \bigr)\\
    &\qquad \leq \bigl( \sum_{k=0}^{K-1}c_k\bigr) \cdot \cO(N/K) +   \sum_{k=0}^{K-1}(  \varepsilon_k + \varepsilon_k') \cdot  \cO(1/\sqrt{K}) + \cO(1/\sqrt{K}).
	\#
	By the definition of $\cL(\lambda, \pi)$ in \eqref{eq:52001}, \eqref{eq:52004} yields that 
	\# \label{eq:52005}
	& \frac{1}{K}\sum_{k=0}^{K-1} \bigl(\rho(\pi^*) - \rho(\pi_{\theta_k})  \bigr) - \frac{1}{K}\sum_{k = 0}^{K - 1}\overline{\lambda}_{k}\bigl( \Lambda(\pi^*) - \Lambda(\pi_{\theta_k}) \bigr) \notag\\
	& \qquad \le \bigl( \sum_{k=0}^{K-1}c_k\bigr) \cdot \cO(1/K) +   \sum_{k=0}^{K-1}(  \varepsilon_k + \varepsilon_k') \cdot  \cO(1/\sqrt{K}) + \cO(1/\sqrt{K}).
	\#
	For any fixed $\lambda' \in [0, N]$, we have 
	\$
	0 &\le (\overline{\lambda}_K - \lambda')^2 \\
	&= \sum_{k = 0}^{K - 1} \bigl((\overline{\lambda}_{k + 1} - \lambda')^2 - (\overline{\lambda}_k - \lambda')^2\bigr) + (\overline{\lambda}_0 - \lambda')^2\\
	& = \sum_{k = 0}^{K - 1} \biggl(\Bigl(\Pi_{[0,N]}\bigl(\overline{\lambda}_k - \frac{1}{2\gamma_k}\bigl(\alpha + 2\overline{y}_k\overline{\rho}(\pi_{\theta_k}) - \overline{\eta}(\pi_{\theta_k}) -\overline{y}_k^2\bigr) \bigr) - \lambda' \Bigr)^2 - (\overline{\lambda}_k - \lambda')^2\biggr) + (\overline{\lambda}_0 - \lambda')^2\\
	& \le \sum_{k = 0}^{K - 1} \biggl(\Bigl(\overline{\lambda}_k - \lambda' - \frac{1}{2\gamma_k}\bigl(\alpha + 2\overline{y}_k\overline{\rho}(\pi_{\theta_k}) - \overline{\eta}(\pi_{\theta_k}) -\overline{y}_k^2\bigr) \Bigr)^2 - (\overline{\lambda}_k - \lambda')^2\biggr) + (\overline{\lambda}_0 - \lambda')^2, 
	\$
	where the second equality uses the definition of $\overline{\lambda}_{k + 1}$ in \eqref{eq:update-lambda-form} and the last inequality follows from the property of projection. Combining with the fact that $\overline{\lambda}_0, \lambda' \in [0, N]$, we further have 
	\# \label{eq:52006}
	0 &\le \sum_{k = 0}^{K - 1}\frac{\overline{\lambda}_k - \lambda'}{\gamma_k} \bigl(\overline{\eta}(\pi_{\theta_k}) + \overline{y}_k^2 - 2\overline{y}_k\overline{\rho}(\pi_{\theta_k}) - \alpha \bigr) \notag\\
	& \qquad + \sum_{k = 0}^{K - 1}\frac{1}{4\gamma_k^2}\bigl(\alpha + 2\overline{y}_k\overline{\rho}(\pi_{\theta_k}) - \overline{\eta}(\pi_{\theta_k}) -\overline{y}_k^2\bigr)^2 + N^2\notag\\
	& = \sum_{k = 0}^{K - 1}\frac{\overline{\lambda}_k - \lambda'}{\gamma_k} \bigl(\Lambda(\pi_{\theta_k}) - \alpha \bigr) + \sum_{k = 0}^{K - 1}\frac{\overline{\lambda}_k - \lambda'}{\gamma_k} \bigl(\overline{\eta}(\pi_{\theta_k}) - \overline{\rho}(\pi_{\theta_k})^2 - \Lambda(\pi_{\theta_k}) \bigr) \notag\\
	& \qquad + \sum_{k = 0}^{K - 1}\frac{1}{4\gamma_k^2}\bigl(\alpha + \overline{\rho}(\pi_{\theta_k})^2 - \overline{\eta}(\pi_{\theta_k})\bigr)^2 + N^2, 
	\#
	where the equality uses the fact that $\overline{y}_k = \overline{\rho}(\pi_{\theta_k})$. 
	Meanwhile, by the definitions of $c_k$ and $d_k$ in \eqref{eq:average error}, we further obtain 
	\# \label{eq:52007}
	|\overline{\eta}(\pi_{\theta_k}) - \overline{\rho}(\pi_{\theta_k})^2 - \Lambda(\pi_{\theta_k}) | &\le |\overline{\eta}(\pi_{\theta_k}) - \eta(\pi_{\theta_k})| + | \overline{\rho}(\pi_{\theta_k})^2 - \rho(\pi_{\theta_k})^2| \notag\\
	& = |\overline{\eta}(\pi_{\theta_k}) - \eta(\pi_{\theta_k})| + | \overline{\rho}(\pi_{\theta_k}) - \rho(\pi_{\theta_k})| \cdot | \overline{\rho}(\pi_{\theta_k}) + \rho(\pi_{\theta_k})| \notag\\
	& \le d_k + 2M c_k.
	\#
	Combining \eqref{eq:52006}, \eqref{eq:52007}, and the facts that $(\alpha + \overline{\rho}(\pi_{\theta_k})^2 - \overline{\eta}(\pi_{\theta_k}))^2 \le (\alpha + 2M^2)^2$ and $\gamma_k = \gamma \sqrt{K}$, we further obtain that 
	\# \label{eq:52008}
	\frac{1}{K}\sum_{k = 0}^{K - 1}(\overline{\lambda}_k - \lambda') \bigl(\alpha - \Lambda(\pi_{\theta_k}) \bigr) \le \sum_{k = 0}^{K - 1}(c_k + d_k) \cdot \cO(1/K) + \cO(1/\sqrt{K}).
	\#
	Adding \eqref{eq:52008} to \eqref{eq:52005}, together with the fact that $\Lambda(\pi^*) \le \alpha$, we have 
	\$
	& \frac{1}{K}\sum_{k=0}^{K-1} \bigl(\rho(\pi^*) - \rho(\pi_{\theta_k})  \bigr) + \frac{\lambda'}{K}\sum_{k = 0}^{K - 1}\bigl( \Lambda(\pi_{\theta_k}) - \alpha  \bigr)\\
	&\qquad \le  \sum_{k=0}^{K-1}(c_k + d_k) \cdot \cO(1/K) +   \sum_{k=0}^{K-1}(  \varepsilon_k + \varepsilon_k') \cdot  \cO(1/\sqrt{K}) + \cO(1/\sqrt{K}).
	\$
	We choose $\lambda' = N$ when $\sum_{k = 0}^{K - 1}( \Lambda(\pi_{\theta_k}) - \alpha  ) \ge 0$, otherwise we take $\lambda' = 0$. Thus, we obtain 
	\$
	&\rho(\pi^*) - \frac{1}{K}\sum_{k=0}^{K-1}  \rho(\pi_{\theta_k})  + N \cdot \Bigl[ \frac{1}{K}\sum_{k = 0}^{K - 1}\Lambda(\pi_{\theta_k}) - \alpha  \Bigr]_+ \\
	&\qquad \le  \sum_{k=0}^{K-1}(c_k + d_k) \cdot \cO(1/K) +   \sum_{k=0}^{K-1}(  \varepsilon_k + \varepsilon_k') \cdot  \cO(1/\sqrt{K}) + \cO(1/\sqrt{K}).
	\$
    Note that $N \ge 2\lambda^*$, together with Lemma \ref{lem:constraint:violation}, we have
	\$
	\rho(\pi^*) - \frac{1}{K}\sum_{k=0}^{K-1}  \rho(\pi_{\theta_k})  &\le  \sum_{k=0}^{K-1}(c_k + d_k) \cdot \cO(1/K) +   \sum_{k=0}^{K-1}(  \varepsilon_k + \varepsilon_k') \cdot  \cO(1/\sqrt{K}) + \cO(1/\sqrt{K}), \\
	 \Bigl[ \frac{1}{K}\sum_{k = 0}^{K - 1}\Lambda(\pi_{\theta_k}) - \alpha  \Bigr]_+ &\le \sum_{k=0}^{K-1}(c_k + d_k) \cdot \cO(1/K) + \sum_{k=0}^{K-1}(  \varepsilon_k + \varepsilon_k') \cdot  \cO(1/\sqrt{K}) + \cO(1/\sqrt{K}).
	\$
	Therefore, we conclude the proof of Theorem \ref{thm:main:slater}.
\end{proof}

\section{Risk-Sensitive RL with Linear Function Approximation} \label{sec:linear}
In this section, we consider the setting where we approximate the $Q$-function in \eqref{eq:def:q}, the $W$-function in \eqref{eq:def:w}, and the energy function $f$ (corresponding to the energy-based policy $\pi \propto \exp(\tau^{-1} f)$) by linear functions, which are computationally more efficient than neural networks, and derive the theoretical results under our proposed algorithmic framework.
Specifically, we assume that 
$Q_q(s, a) = q^\top \varphi(s, a)$, $W_\omega(s, a) = \omega^\top \varphi(s, a)$, $f_\theta(s, a) = \theta^\top \varphi(s, a)$. Here $\varphi : \cS \times \cA \rightarrow \RR^d$ is a $d$-dimensional feature map. Without loss of generality, we further assume that $\|\varphi(s, a)\|_2 \le 1$ for any $(s, a) \in \cS \times \cA$.
\subsection{Algorithm}
In this subsection, we present VARAC with linear function approximation. In the sequel, we describe the actor and critic update rules at each iteration.
\vskip 4pt
\noindent{\bf Actor Update: (i) $\lambda$-Update Step.}
Similar to \eqref{eq:update-lambda-form}, we update $\lambda$ by 
\#
\overline{\lambda}_{k+1} =  \Pi_{[0,N]}\Bigl(\overline{\lambda}_k - \frac{1}{2\gamma_k}\bigl(\alpha + 2\overline{y}_k\overline{\rho}(\pi_{\theta_k}) - \overline{\eta}(\pi_{\theta_k}) -\overline{y}_k^2\bigr) \Bigr),
\#
where $\gamma_k>0$ is some prespecified stepsize.  

\vskip 4pt
\noindent{\bf (ii) $\pi$-Update Step.}
Under the linear function approximation setting, by Proposition \ref{prop:exp-policy}, the solution of \eqref{eq:dnn2222} admits a closed-form solution that
\# \label{eq:update:theta2}
\theta_{k + 1} = \tau_{k + 1} \cdot \big(\beta_k^{-1}(1 + 2\overline{\lambda}_{k}\overline{y}_{k})q_k - \beta_k^{-1}\overline{\lambda}_{k}\omega_{k} + \tau_k^{-1} {\theta_k}\big).
\#
\vskip 4pt
\noindent{\bf (iii) $y$-Update Step.}
Similar to \eqref{eq:update-y-form}, we update $y$ by 
\#
%\overline{y}_{k+1} = \Pi_{[0,M]}\bigl(\overline{\rho}(\pi_{\theta_k})\bigr) ,
\overline{y}_{k+1} = \overline{\rho}(\pi_{\theta_{k+1}}).
\#

\vskip 4pt
\noindent{\bf Critic Update: (i) $q$-Update Step.}
We  solve the least-squares problem in~\eqref{eq:mspbe1}, which can be solved by TD learning. Specifically, given an initial radii $R$, we use the iterative TD-update that at the $t$-th iteration,  we let
\begin{equation}\label{eq:td-update3}
\begin{aligned}
q{(t+1)} \leftarrow \Pi_{\cB{(0,R)}}\Bigl(&q{(t)} - \delta\cdot \bigl(Q_{q{(t)}}(s,a) - r(s, a) \\
&+ \overline{\rho}(\pi_{\theta_{k}}) - Q_{q{(t)}}(s', a')\bigr)\cdot\varphi(s,a)\Bigr),
\end{aligned}
\end{equation}
{\noindent where $(s, a) \sim \sigma_k$, $s'\sim\cP(\cdot\,|\,s,a)$, $a' \sim \pi_{\theta_k}(\cdot\,|\,s')$, and $\delta$ is the stepsize. See Algorithm \ref{alg:td3} in Appendix \ref{appendix:alg:linear} for a pseudocode.}

\vskip 4pt
\noindent{\bf (ii) $\omega$-Update Step.} Similar to the $q$-update step, we update $\omega$ by 
\begin{equation}\label{eq:td-update4}
	\begin{aligned}
	\omega{(t+1)} \leftarrow \Pi_{\cB{(0,R)}}\Bigl(& \omega{(t)} - \delta\cdot \bigl(W_{\omega(t)}(s,a) - r(s, a)^2 \\
	& + \overline{\eta}(\pi_{\theta_{k}}) - W_{\omega{(t)}}(s', a')\bigr)\cdot \varphi(s,a)\Bigr),
	\end{aligned}
	\end{equation}
	{\noindent where $(s, a) \sim \sigma_k$, $s'\sim\cP(\cdot\,|\,s,a)$, $a' \sim \pi_{\theta_k}(\cdot\,|\,s')$, and $\delta$ is the stepsize. See Algorithm \ref{alg:td4} in Appendix \ref{appendix:alg:linear} for a pseudocode.}

Putting the above update rules together, we obtain the VARAC with linear function approximation. The pseudocode is summarized in  Algorithm \ref{alg:risac2} in Appendix \ref{appendix:alg:linear}.

\subsection{Theoretical Results}

In this subsection, we provide theoretical guarantees for VARAC with linear function approximation. First, we impose the following assumption, which parallels to  Assumption~\ref{assumption:closed-q-w} for the DNN setting.

\begin{assumption} \label{assumption:closed:q:w2}
	%For any $\omega, \theta \in \cB(0, R)$, it holds that 
	%\$
	%\inf_{\omega' \in \cB(0, R)}\EE_{\sigma_{\pi_\theta}} \bigl[ \bigl( (\cT^{\pi_\theta}Q_{\omega} - w' \varphi)(s, a) \bigr)^2 \bigr] = 0,
	%\$
	%where $\cT^{\pi_\theta}$ is defined in \eqref{eq:bellman1}.
	For any $R > 0$, $Q_q \in \{q^\top\varphi : q \in \cB(0, R)\}$, $W_\omega \in \{\omega^\top\varphi : \omega \in \cB(0, R)\}$, and policy $\pi$, we have $\cT^\pi Q_q \in \{q^\top\varphi : q \in \cB(0, R)\}$ and $\hat{\cT}^\pi W_\omega \in \{\omega^\top\varphi : \omega \in \cB(0, R)\}$.
\end{assumption}

\iffalse
\begin{assumption}[Well-Conditioned Feature]
	The minimum singular value of the matrix $\EE_{\sigma_k}[\varphi(s, a)\varphi(s, a)^\top]$
	is uniformly lower bounded by a positive absolute constant $\sigma^*$ for any $k \ge 1$.
\end{assumption}
\fi 

In the theoretical analysis of VARAC with DNN,  we characterize the estimation and computational errors, respectively. Here we only need to bound the computational errors since the estimation errors can be bounded by similar arguments as in Section \ref{sec:estimate-error}. As stated in Section \ref{sec:ac-error}, the computational errors are incurred by: (i) the SGD update (Lemma~\ref{thm:ac-error}), when we update policy $\pi$, and (ii) the TD update (Lemmas~\ref{thm:td} and \ref{thm:td2}), when we evaluate Q-function and W-function. Here in the framework of linear approximation, instead of SGD updates, we update policy $\pi$ by a closed-form solution in~\eqref{eq:update:theta2}.
Hence, we only need to characterize the TD errors, which is achieved in the following two lemmas.

\begin{lemma}[$q$-Update Error] \label{lemma:q:error:linear}
	Suppose that Assumptions \ref{assumption:contraction} and \ref{assumption:closed:q:w2} hold. Let $\delta = T^{-1/2}$. Then, at the $k$-th iteration of Algorithm \ref{alg:risac2}, the output $Q_{\overline{q}}$ of Algorithm \ref{alg:td3} satisfies
	\$
	\EE\bigl[ \bigl(Q_{\overline{q}}(s,a) - Q^{\pi_{\theta_k}}(s,a)\bigr)^2 \bigr] \le \cO(R^2T^{-1/2}),
	\$
	where the expectation is taken over $\overline q$ and $(s,a)\sim \sigma_{\pi_{\theta_k}}$, and $T$ is the iteration counter. 
\end{lemma}

\begin{proof}
	See Appendix \ref{appendix:td:linear} for a detailed proof.
\end{proof}

\begin{lemma}[$\omega$-Update Error] \label{lemma:w:error:linear}
	Suppose that Assumptions \ref{assumption:contraction} and \ref{assumption:closed:q:w2} hold. Let $\delta = T^{-1/2}$. Then, at the $k$-th iteration of Algorithm \ref{alg:risac2}, the output $W_{\overline{\omega}}$ of Algorithm \ref{alg:td4} satisfies
    \$
    \EE \bigl[ \bigl(W_{\overline{\omega}}(s,a) - W^{\pi_{\theta_k}}(s,a)\bigr)^2 \bigr] \le \cO(R^2T^{-1/2}),
    \$
    where the expectation is taken over $\overline \omega$ and $(s,a)\sim \sigma_{\pi_{\theta_k}}$, and $T$ is the iteration counter. 
\end{lemma}

\begin{proof}
	The proof is similar to the proof of Lemma \ref{lemma:q:error:linear}, and we omit it to avoid repetition.
\end{proof}

%Similar to the Lemmas \ref{thm:td} and \ref{thm:td2}, we characterize TD errors in the following two lemmas.

Then, following the arguments in Section \ref{sec:error}, we analyze the  errors. Specifically, under the same notations in Lemmas~\ref{lem:error-kl} and~\ref{lem:stepwise-energy}, we have 
\$
\varepsilon_k= (1+2MN)\cdot\beta_k^{-1}\epsilon_{k}'\cdot\psi^*_{k} +N\cdot\beta_k^{-1}\epsilon_{k}''\cdot\psi^*_{k}, \quad \epsilon'_k = \epsilon''_k = \cO(R^2T^{-1/2}), \quad \varepsilon'_k = 0.
\$
The derivation is the same as that of Lemmas~\ref{lem:error-kl} and~\ref{lem:stepwise-energy}, and thus we omit the details  for simplicity.
Plugging these errors into Theorem \ref{thm:main:slater}, we have the following theorem. 
\begin{theorem}[Constrained Violation] \label{thm:main:linear}
	Suppose that Assumptions~\ref{assumption:unique-solution}, \ref{assumption:contraction}, \ref{assumption:slater}, and \ref{assumption:closed:q:w2} hold. Let $N = 2M/\xi$ in~\eqref{eq:update-lambda-form}. For the sequences $\{\overline{\lambda}_k\}^{K}_{k=1}$, $\{\pi_{\theta_k}\}^{K}_{k = 1}$ and $\{\overline{y}_k\}^{K}_{k=1}$ generated by the VARAC algorithm (Alg.~\ref{alg:risac}), we have  
%\begin{equation}\label{eqn:m2}
%\begin{aligned}
	\$
	\rho(\pi^*) - \frac{1}{K}\sum_{k=0}^{K-1}  \rho(\pi_{\theta_k})  &\le  \sum_{k=0}^{K-1}(c_k + d_k) \cdot \cO(1/K) +   \sum_{k=0}^{K-1} \varepsilon_k  \cdot  \cO(1/\sqrt{K}) + \cO(1/\sqrt{K}), \\
	 \Bigl[ \frac{1}{K}\sum_{k = 0}^{K - 1}\Lambda(\pi_{\theta_k}) - \alpha  \Bigr]_+ &\le \sum_{k=0}^{K-1}(c_k + d_k) \cdot \cO(1/K) +   \sum_{k=0}^{K-1}  \varepsilon_k  \cdot  \cO(1/\sqrt{K}) + \cO(1/\sqrt{K}).
	\$            
 %\end{aligned}
%\end{equation}
where $c_k$ and $d_k$ are estimation errors defined in \eqref{eq:average error}. Here $\varepsilon_k= (1+2MN)\cdot\beta_k^{-1}\epsilon_{k}'\cdot\psi^*_{k} +N\cdot\beta_k^{-1}\epsilon_{k}''\cdot\psi^*_{k}$, where 
$\epsilon'_{k} = R^2T^{-1/2}$ and $\epsilon''_{k} = R^2T^{-1/2}.$
\end{theorem}

\begin{proof}
	The proof is same as the proof of Theorem \ref{thm:main:slater}, and we omit it to avoid repetition.
\end{proof}

By Theorem \ref{thm:main:linear}, we have that, under the linear function approximation setting, VARAC (Algorithm~\ref{alg:risac2}) also achieves the $\cO(1/\sqrt{K})$ convergence rate/constraint violation.

\iffalse
\section{Special Example: Two-Layer Neural Network}

In this section, we specialize

\vskip 4pt
\noindent{\bf Two-Layer Neural Network.}
We introduce the two-layer neural network \citep{cai2019neural,liu2019neural,wang2020neural}, which is a special case of DNN defined in Section~\ref{sec:2}. Specifically, for a two-layer neural network $f(x, W, b)$ with width $m$ and input $x$, its output takes the form of 
\$
f(x, W, b) = \frac{1}{\sqrt{m}}\sum_{i = 1}^m b_i \cdot \sigma(x^\top [W]_i),
\$
where $\sigma(\cdot) = \max\{0, \cdot\}$ is the ReLU activation function, $b \in \{-1, 1\}^m$ is the output layer, and $W = ([W]_1^\top, [W]_2^\top, \cdots, [W]_m^\top) \in \RR^{md}$ are input weights. Moreover, we consider the initialization 
\$
[W]_i \overset{\rm i.i.d.}{\sim} \mathcal N(0,1), \quad  [b]_i \overset{\rm i.i.d.}{\sim} \text{Unif}(\{-1, 1\}),  \quad \text{for all } i \in [m].
\$

\$
\cB(W^0, R) = \{W : \|W - W^0\|_2 \le R \}
\$
\fi

% !TEX root = main.tex

\section{Experiment} \label{sec:exp}
To evaluate the efficacy of our newly proposed VARAC algorithm, we conducted experiments using two publicly available mechanical control environments: {\sf Pendulum-v0} and {\sf BipedalWalkerHardcore-v3} from OpenAI gym \citep{1606.01540}. 
Various reinforcement learning algorithms are extensively employed in diverse automation control scenarios to instruct machines in executing different tasks \citep{chen2022system,qiu2022safe}. Ensuring control stability, which means maintaining stable algorithm performance even when minor environmental variations occur, is crucial for the practical usefulness of the algorithm, which is exactly what the proposed VARAC algorithm aims to achieve.

\subsection{Experiment Setting}
We let the classical TD3 algorithm \citep{fujimoto2018addressing}, known for its effectiveness and robustness in continuous control environments, be the baseline algorithm. We run the TD3 and VARAC algorithms for $2 \times 10^5$ steps on the {\sf Pendulum-v0} environment and $3 \times 10^6$ steps on {\sf BipedalWalkerHardcore-v3}, each with ten different random seeds. The learned policies are evaluated based on 40 episodes, and we revord the average performance at each checkpoint. We select the best policy from each run to compute the risk-sensitive metric to ensure fair comparisons. Details of the other hyperparameters are given in Section~\ref{app:hyper} in Appendix.

\subsection{Implementation of VARAC}

The updates of $\lambda$ and $y$ follow \eqref{eq:update-lambda-form} and \eqref{eq:update-y-form}, respectively. 
Regarding the policy update stage of VARAC, we approximate the solution to \eqref{eq:dnn2222} by
\begin{align*}
L(\theta_{k+1}) =  \EE_{\nu_{\pi_{\theta_k}}} \bigl[ & \bigl\la (1+ 2\overline{\lambda}_{k}\overline{y}_{k})Q_{q_k}(s, \cdot) - \overline{\lambda}_{k}W_{\omega_{k}}(s, \cdot), \pi_{\theta_{k+1}}(\cdot\,|\,s)\bigr\ra \\ 
& - \beta_k \cdot {\rm KL}\bigl(\pi_{\theta_{k+1}}(\cdot\,|\,s)\,\|\, \pi_{\theta_{k}}(\cdot\,|\,s)\bigr) \bigr]
\\
\approx\EE_{\nu_{\pi_{\theta_k}}} \bigl[ & \bigl\la \tilde{Q}_{\mu_k}(s, \cdot), \pi_{\theta_{k+1}}(\cdot\,|\,s)\bigr\ra  - \beta_k \cdot {\rm KL}\bigl(\pi_{\theta_{k+1}}(\cdot\,|\,s)\,\|\, \pi_{\theta_{k}}(\cdot\,|\,s)\bigr) \bigr],
\end{align*}
where $\tilde{Q}_{\mu_k}(s,a)$ is the function approximation of $\sum_{t = 0}^\infty\EE[\tilde r(s_t,a_t)-\tilde \rho(\pi)\mid s_t=s,a_t = a]$, $\tilde r(s,a)= (1 + 2\overline{\lambda}_{k}\overline{y}_{k})r(s,a)- \overline{\lambda}_{k} r^2(s,a)$, and  $\tilde \rho(s,a)= (1+2\overline{\lambda}_{k}\overline{y}_{k})\rho(s,a) - \overline{\lambda}_{k} \eta(s,a)$. This suggests that we only need to solve a new MDP problem by replacing the original $r$ by $\tilde{r} = (1 + 2\overline{\lambda}_{k}\overline{y}_{k})r- \overline{\lambda}_{k} r^2$. We solve this new MDP problem by TD3 for a fair comparison.

%Here we incorporate TD3 as the policy optimization framework in the practical implementation of VARAC for consistency. 

\subsection{Empirical Performance}

We depict the reward of TD3 and VARAC under two environments ({\sf Pendulum-v0} and {\sf BipedalWalkerHardcore-v3}) in Figure~\ref{fig:main_fig}. Additionally, we report the mean and variance of TD3 and VARAC in Table~\ref{tab:eval}. From the figure, we can observe that VARAC exhibits a slower convergence  compared with TD3, but ultimately reaches a similar level of performance. The table shows that VARAC achieves slightly lower mean performance, but significantly reduces the variance. This demonstrates the empirical power of VARAC in the risk-sensitive setting.

\begin{figure}[t]
    % \vspace{-18pt}
    % \setlength{\abovecaptionskip}{0cm}
    % \setlength{\belowcaptionskip}{0.1cm}
    \centering
    \mbox{
            \hspace{-0.10in}
		\subfigure[]{
			\label{fig:pendulum}
			\includegraphics[width=0.5\textwidth]{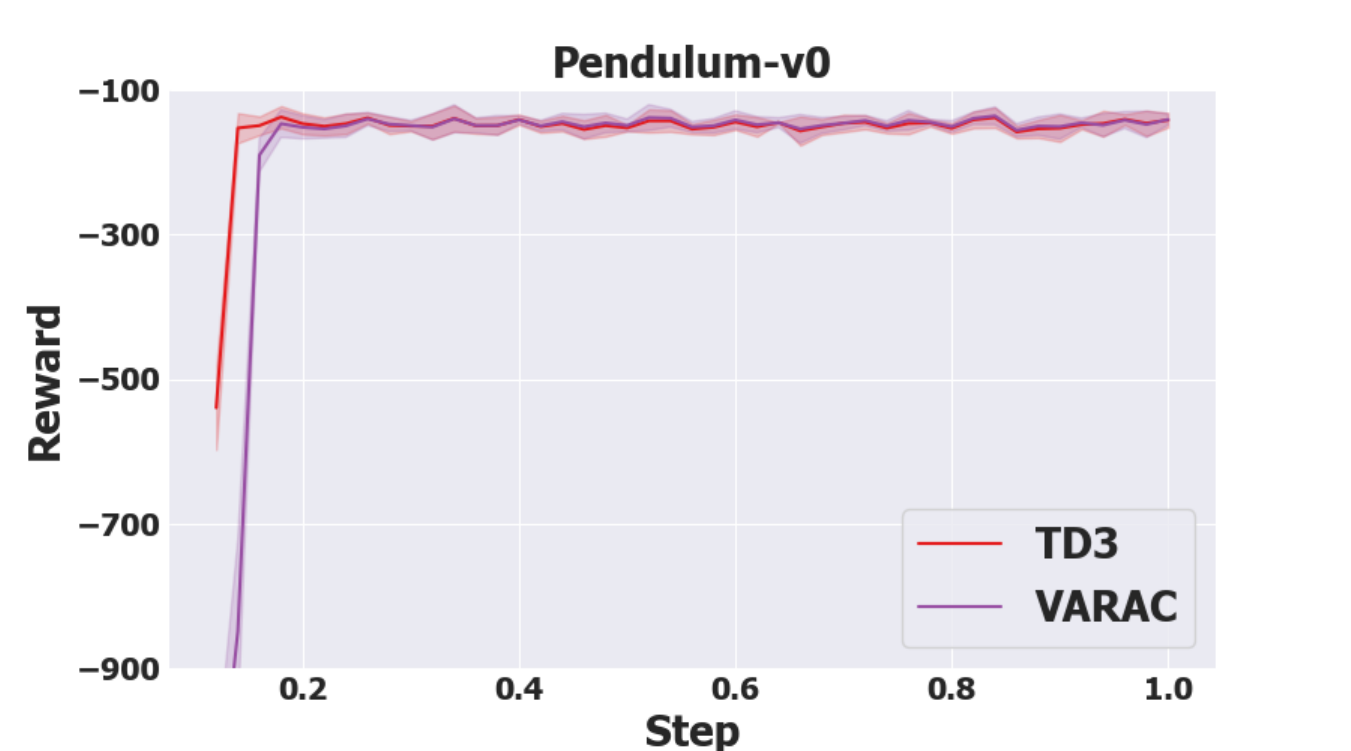}
		}
		\hspace{-0.15in}
		\subfigure[]{
			\label{fig:bipedal}
			\includegraphics[width=0.5\textwidth]{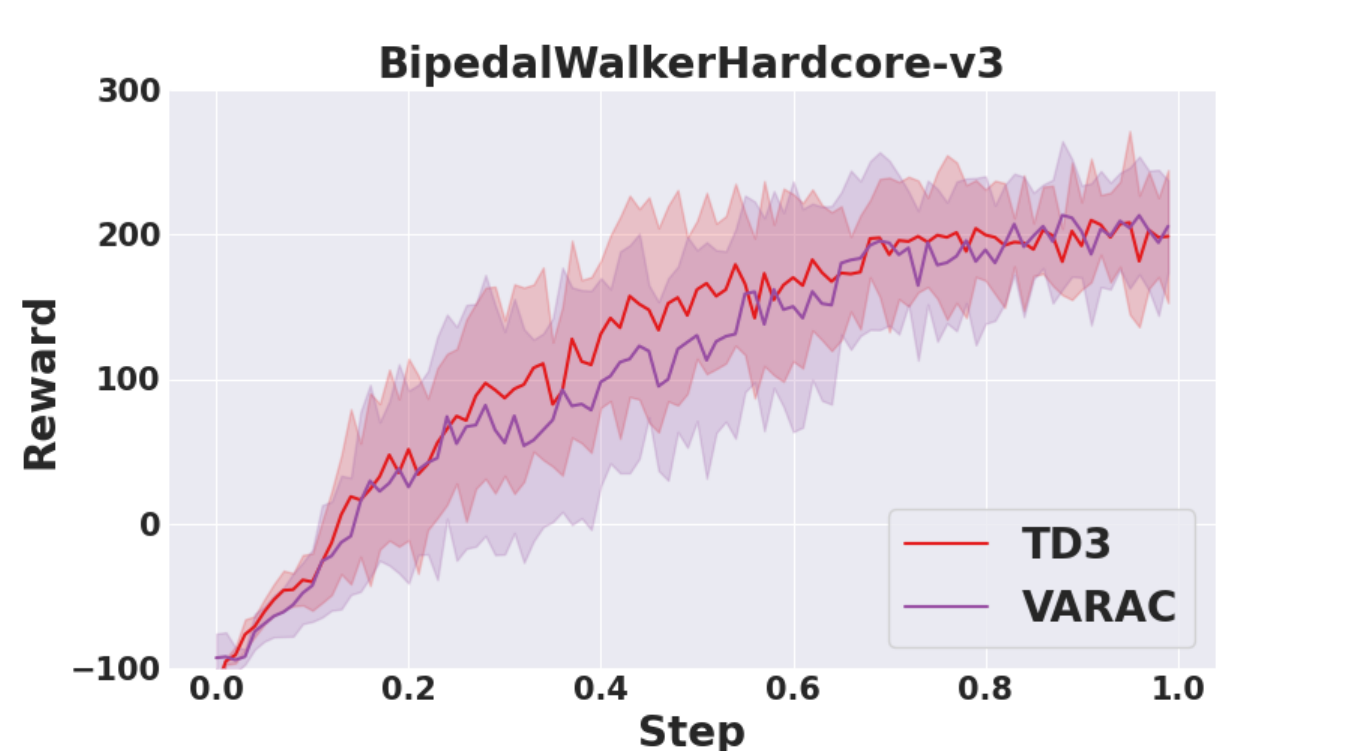}
		}
	}
    \caption{We present the training progress of TD3 and VARAC algorithms in the figure. In the graph, the $y$-axis represents the average reward value obtained after evaluating each checkpoint, while the $x$-axis represents the ratio of training steps to the total number of steps, indicating the extent of training progress. The curves are averaged over ten independent runs with shaded regions indicating standard deviations.}
    \label{fig:main_fig}
\end{figure}

\begin{table}[t] \small 
	\centering
	\resizebox{0.7\textwidth}{!}{
		\begin{tabular}{lcccc}
			\hline
                \multirow{2}{*}{Algorithm} & \multicolumn{2}{c}{\textsf{Pendulum-v0}} & \multicolumn{2}{c}{\textsf{BipedalWalkerHardcore-v3}} \\
                
			   & Mean & Variance & Mean & Variance \\
                \hline
                TD3 & -122 & 6903 & 234 & 9348 \\
                VARAC & -126 & \textbf{4826} &  221 &  \textbf{6090} \\
			\hline
	\end{tabular}}
        \caption{The mean and variance of the policy learned by TD3 and VARAC algorithms under two gym environments.} \label{tab:eval}
	
\end{table}

%%%%%%%%%%%%%%%%%%%%%%%%
\section{Conclusion} \label{sec:con}
To conclude, to the best of our knowledge, we make the first attempt to study risk-sensitive deep reinforcement learning, where we consider the variance constrained deep reinforcement learning. We propose an efficient and theoretically sound VARAC algorithm to solve the problem. Under mild assumptions, despite the overparametrization and nonconvexity, we show that our algorithm achieves an $\cO(1/\sqrt{K})$ convergence rate to a saddle point, and that our solution converges to a globally optimal solution at a same rate.
For future work, we plan to extend the risk constraints to other coherent risk measures such as the conditional value at risk. 

%{\textcolor{red}{Under mild assumptions, despite the overparametrization and nonconvexity, both convergence to a saddle point and diminishing of the duality gap have reached a rate of $\cO(1/\sqrt{K})$ in our algorithm.}
%%%%%%%%%%%%%%
% !TEX root = main.tex

%%%%%%%%%%%%%

%\newpage

\bibliographystyle{ims}
\bibliography{graphbib}

\newpage
\begin{appendix}

\noindent{\LARGE{\bf Appendix}}

\section{Algorithms in Section \ref{sec:alg}}\label{appendix:alg}
We present the algorithms for solving the subproblems of policy improvement and policy evaluation in Section \ref{sec:alg}. 

%by solving the subproblem in \eqref{eq:mspbe} using the TD update in \eqref{eq:td-update1}

\begin{algorithm}[H]
	\caption{Update $\theta$ via SGD}
	\begin{algorithmic}[1]\label{alg:ac-update}
		\STATE {\bf Require:} MDP $(\cS, \cA, \cP, r, \gamma)$, current energy function $f_{\theta_k}$, initial actor parameter $\theta_0$, number of iterations $T$, sample $\{(s_t, a^0_t)\}^{T}_{t = 1}$
		\STATE {\bf Initialization:} $\theta(0) \leftarrow \theta_0$
		\STATE Set stepsize $\zeta \leftarrow {T}^{-1/2}$ 
		\FOR{$t = 0, \dots, T - 1$}
		\STATE Sample $(s, a) \leftarrow (s_{t+1}, a^0_{t+1})$
		\STATE $\theta(t+1) \leftarrow \Pi_{\cB(\theta_0, R_a)}\bigl(\theta(t) - \zeta\cdot \bigl(f_{\theta(t)}(s,a) - \tau_{k+1}\cdot(\beta_k^{-1}(1 + 2\overline{\lambda}_{k}\overline{y}_{k})Q_{q_k}+ \beta_k^{-1}\overline{\lambda}_{k}W_{\omega_{k}} + \tau^{-1}_{k}f_{\theta_k}(s,a))\bigr)\cdot\nabla_\theta f_{\theta(t)}(s,a)\bigr)
		$
		\ENDFOR
		\STATE Average over path $\overline{\theta} \leftarrow 1/T\cdot\sum^{T-1}_{t = 0}\theta(t)$
		\STATE {\bf Output:} $f_{\overline{\theta}}$
	\end{algorithmic}
\end{algorithm}

\begin{algorithm}[H]
\caption{Update $q$ via TD(0)}
\begin{algorithmic}[1]\label{alg:td}
\STATE {\bf Require:} MDP $(\cS, \cA, \cP, r)$, initial critic parameter $q_0$, number of iterations $T$, sample $\{(s_t, a_t, s_t', a_t')\}^{T}_{t = 1}$
\STATE {\bf Initialization:} $q(0) \leftarrow q_0$
\STATE Set stepsize $\delta \leftarrow T^{-1/2}$
\FOR{$t = 0, \dots,T - 1$}
\STATE Sample $(s, a, s', a')$ $\leftarrow (s_{t+1}, a_{t+1}, s'_{t+1}, a'_{t+1})$
\STATE $q(t+1) \leftarrow \Pi_{\cB(q_0, R_c)}\bigl( q(t) - \delta\cdot \bigl(Q_{q(t)}(s,a) - r(s, a) + \overline{\rho}(\pi_{\theta_{k}}) - Q_{q(t)}(s', a')\bigr)\cdot\nabla_q Q_{q(t)}(s,a)\bigr)$
\ENDFOR
\STATE Average over path $\overline{q} \leftarrow 1/T \cdot \sum^{T - 1}_{t = 0}q(t)$
\STATE {\bf Output:} $Q_{\overline{q}}$
\end{algorithmic}
\end{algorithm}

\begin{algorithm}[H]
	\caption{Update $\omega$ via TD(0)}
	\begin{algorithmic}[1]\label{alg:td2}
		\STATE {\bf Require:} MDP $(\cS, \cA, \cP, r, \gamma)$, initial critic parameter $\omega_0$, number of iterations $T$, sample $\{(s_t, a_t, s_t', a_t')\}^{T}_{t = 1}$
		\STATE {\bf Initialization:} $\omega(0) \leftarrow \omega_0$
		\STATE Set stepsize $\delta \leftarrow T^{-1/2}$
		\FOR{$t = 0, \dots,T-1$}
		\STATE Sample $(s, a, s', a') \leftarrow (s_{t+1}, a_{t+1}, s'_{t+1}, a'_{t+1})$
		\STATE $\omega(t+ 1) \leftarrow \Pi_{\cB(\omega_0, R_b)}\bigl(\omega(t) - \delta\cdot \bigl(W_{\omega(t)}(s,a) - r(s, a) + \overline{\eta}(\pi_{\theta_{k}}) - W_{\omega(t)}(s', a')\bigr)\cdot\nabla_\omega Q_{\omega(t)}(s,a)\bigr)$
		\ENDFOR
		\STATE Average over path $\overline{\omega} \leftarrow 1/T \cdot \sum^{T-1}_{t = 0}\omega(t)$
		\STATE {\bf Output:} $W_{\overline{\omega}}$
	\end{algorithmic}
\end{algorithm}

\section{Proof of Proposition \ref{prop:exp-policy}}\label{appendix:alg-proof}
\begin{proof}
The subproblem of policy improvement for solving $\hat{\pi}_{k+1}$ takes the form
\$
&\max_{\pi} ~\EE_{\nu_k} \bigl[\la \pi(\cdot\,|\,s), (1 + 2\overline{\lambda}_k\overline{y}_k)Q_{q_k}(s,a) - \overline{\lambda}_kW_{\omega_k}(s,a) \ra - \beta_k \cdot {\rm KL}\bigl(\pi(\cdot\,|\,s)\,\|\,\pi_{\theta_k}(\cdot\,|\,s)\bigr) \bigr] \\
&~\text{subject to }\sum_{a\in \cA}\pi(a\,|\,s) = 1,~~\text{for any}~ s \in \cS.
\$
We consider the Lagrangian dual function of the above maximization problem that
\$
 &\int_{s\in\cS}\bigl[\la \pi(\cdot\,|\,s), (1 + 2\overline{\lambda}_k\overline{y}_k)Q_{q_k}(s,a) - \overline{\lambda}_kW_{\omega_k}(s,a)\ra - \beta_k\cdot {\rm KL}\bigl(\pi(\cdot\,|\,s)\,\|\,\pi_{\theta_k}(\cdot\,|\,s)\bigr)\bigr]\nu_{k}(\ud s)  \\
 &\qquad+ \int_{s\in \cS} \biggl(\sum_{a\in \cA}\pi(a\,|\, s) - 1\biggr)\lambda(\ud s).  \notag
\$
Recall that we restrict the solution be an energy-based policy that $\pi_{\theta_{k}} \propto \exp(\tau_{k}^{-1} f_{\theta_{k}})$.
Plugging $\pi_{\theta_{k}}(s,a) = \exp(\tau_{k}^{-1} f_{\theta_{k}}(s,a))/\sum_{a'\in\cA}\exp(\tau_{k}^{-1} f_{\theta_{k}}(s,a'))$ into the above function and taking the derivative, we obtain the optimality condition
\$
&(1 + 2\overline{\lambda}_k\overline{y}_k)Q_{q_k}(s,a) - \overline{\lambda}_kW_{\omega_k}(s,a) + \beta_k\tau_k^{-1} f_{\theta_k}(s,a) \\
&\qquad- \beta_k\cdot\bigg[\log\biggl(\sum_{a'\in\cA}\exp\bigl(\tau_k^{-1}f_{\theta_k}(s,a')\bigr)\biggr) + \log\pi(a\,|s) + 1\biggr] + \frac{\lambda(s)}{\nu_k(s)} = 0,
\$
for any $a\in \cA$ and $s\in \cS$. Note that $\log(\sum_{a'\in\cA}\exp(\tau_k^{-1} f_{\theta_k}(s,a')))$ is determined by the state~$s$ only. Thus, for any $(s,a)\in \cS\times \cA$, we have 
\$
\hat{\pi}_{k+1}(a\,|\,s) \propto \exp\bigl(\beta_k^{-1}(1 + 2\overline{\lambda}_k\overline{y}_k)Q_{q_k}(s,a) - \beta_k^{-1}\overline{\lambda}_kW_{\omega_k}(s,a) + \tau_k^{-1} f_{\theta_k}(s,a)\bigr), 
\$
which completes the proof.
\end{proof}

% !TEX root = main.tex

\section{Proofs for Section \ref{sec:ac-error}}\label{appendix:td}
\begin{proof}[Proof of Lemma \ref{thm:td}]
		Let the local linearization of $Q_{q}$ be
		\#\label{eq:lin-critic}
		\overline Q_{q} = Q_{q_0} + (q - q_0)^\top \nabla_{q_0} Q_{q}. 
		\#
		We denote by 
		\#\label{eq:def-gs-c}
		& g_t = \bigl( Q_{q(t)}(s,a) -  Q_{q(t)}(s', a') -  r_0  + \overline{\rho}(\pi_\theta) \bigr)\cdot \nabla_q Q_{q(t)}(s,a), && g_t^e = \EE_{\pi_\theta}[ g_n ], \notag\\
		& \overline g_t = \bigl( \overline Q_{q_(t)}(s,a) -  \overline Q_{q(t)}(s',a') -  r_0 + \overline{\rho}(\pi_\theta)   \bigr)\cdot \nabla_q Q_{q_0}(s,a), && \overline g_t^e = \EE_{\pi_\theta}[ \overline g_n ],\notag\\
		& g_* = \bigl( Q_{q_*}(s,a) -  Q_{q_*}(s',a') - r_0 + \overline{\rho}(\pi_\theta)  \bigr)\cdot \nabla_q Q_{q_*}(s,a), && g_*^e = \EE_{\pi_\theta}[ g_* ],\notag\\
		& \overline g_* = \bigl( \overline Q_{q_*}(s,a) -  \overline Q_{q_*}(s',a') - r_0  + \overline{\rho}(\pi_\theta) \bigr)\cdot \nabla_q Q_{q_0}(s,a), && \overline g_*^e = \EE_{\pi_\theta}[ \overline g_* ],
		\#
		where $q_*$ satisfies that
		\$
		q_* = \Pi_{\cB(q_0, R_{\rm c} )}(q_* - \delta \cdot \overline g_*^e ). 
		\$
		Here the expectation $\EE_{\pi_\theta}[\cdot]$ is taken following $(s,a)\sim \rho_{\pi_\theta}(\cdot)$, $s'\sim P(\cdot\given s,a)$, $a'\sim \pi_{\theta}(\cdot\given s_1)$, and $r_0 = \cR(s,a)$.  By Algorithm \ref{alg:td}, we have that
		\$
		q(t+1) = \Pi_{\cB(q_0, R_{\rm c} )}\big(q(t) - \delta \cdot  g_t \big). 
		\$
		Then, we have
		\#\label{eq:ac-bd1-c}
		& \EE_{\pi_\theta}\bigl[ \| q(t+1) - q_* \|_2^2 \given q(t) \bigr]\notag\\
		& \qquad = \EE_{\pi_\theta}\bigl[ \| \Pi_{\cB(q_0, R_{\rm c})}(q(t) - \delta \cdot  g_t ) - \Pi_{\cB(q_0, R_{\rm c})}(q_* - \delta \cdot \overline g_*^e ) \|_2^2 \given q(t) \bigr]\notag\\
		& \qquad \leq \EE_{\pi_\theta}\bigl[ \| (q(t) - \delta \cdot  g_t ) - (q_* - \delta \cdot \overline g_*^e ) \|_2^2 \given q(t) \bigr]\notag\\
		& \qquad = \|q(t) - q_*\|_2^2  + 2\delta \cdot  \la q_* - q(t), g_t^e - \overline g_*^e \ra + \delta^2 \cdot \EE_{\pi_\theta} \bigl[\|g_t - \overline g_*^e\|_2^2\given q(t)\bigr]. 
		\#
		We upper bound the second term on the right hand side of \eqref{eq:ac-bd1-c} in the sequel. By H\"older's inequality, it holds that
		\#\label{eq:ac-bd-t1-a-c}
		& \la q_* - q(t), g_t^e - \overline g_*^e \ra   \notag\\
		& \qquad = \la  q_* - q(t), g_t^e - \overline g_t^e \ra + \la  q_* - q(t), \overline g_t^e - \overline g_*^e \ra \notag\\
		& \qquad \leq \| q_* - q(t)\|_2\cdot \| g_t^e - \overline g_t^e \|_2 + \la  q_* - q(t), \overline g_t^e - \overline g_*^e  \ra\notag\\
		& \qquad \leq 2R_{\rm c} \cdot \|g_t^e - \overline g_t^e\|_2 + \la  q_* - q(t), \overline g_t^e - \overline g_*^e \ra,
		\#
		where the last inequality is obtained by the fact that $q(t), q_* \in \cB(q_0, R_{\rm c})$. By the definitions in \eqref{eq:def-gs-c}, we further obtain
		\#\label{eq:ac-bd-t1-b-c}
		& \la q_* - q(t), \overline g_t^e - \overline g_*^e \ra \notag\\
		& \qquad = \EE_{\pi_\theta} \bigl[ \bigl((\overline Q_{q(t)}(s,a) - \overline Q_{q_*}(s,a)) - (\overline Q_{q(t)}(s',a') - \overline Q_{q_*}(s',a')) \bigr) \cdot \la q_* - q(t), \nabla_q Q_{q_0}(s,a) \ra \bigr]\notag\\ 
		& \qquad = \EE_{\pi_\theta} \bigl[ \bigl((\overline Q_{q(t)}(s,a) - \overline Q_{q_*}(s,a)) - (\overline Q_{q(t)}(s',a') - \overline Q_{q_*}(s',a')) \bigr) \cdot (\overline Q_{q_*}(s,a) - \overline Q_{q(t)}(s,a)) \bigr]\notag\\
		& \qquad = \EE_{\pi_\theta} \bigl[ (\overline Q_{q(t)}(s,a) - \overline Q_{q_*}(s,a)) \cdot (\overline Q_{q(t)}(s',a') - \overline Q_{q_*}(s',a')) \bigr] \notag \\
		&\qquad\qquad - \EE_{\pi_\theta} \bigl[ (\overline Q_{q(t)}(s,a) - \overline Q_{q_*}(s,a))^2 \bigr],
		\#
		where the second equality is obtained by \eqref{eq:lin-critic}. By Cauchy-Schwartz inequality, we further have 
		\# \label{eq:ac-bd-t1-bb}
		&\EE_{\pi_\theta} \bigl[ (\overline Q_{q(t)}(s,a) - \overline Q_{q_*}(s,a)) \cdot (\overline Q_{q(t)}(s',a') - \overline Q_{q_*}(s',a')) \bigr]  \notag \\
		&\qquad \le  [\EE_{\pi_\theta} \bigl[ (\overline Q_{q(t)}(s,a) - \overline Q_{q_*}(s,a))^2 \bigr] ]^{\frac{1}{2}} \cdot  [ \EE_{\pi_\theta} \bigl[ (\overline Q_{q(t)}(s',a') - \overline Q_{q_*}(s',a'))^2 \bigr] ]^{\frac{1}{2}} \notag \\
		& \qquad \le \beta_{\pi_\theta} \cdot \EE_{\pi_\theta} \bigl[ (\overline Q_{q(t)}(s,a) - \overline Q_{q_*}(s,a))^2 \bigr],
		\#
		where the last inequality holds by Assumption \ref{assumption:contraction}. Combining \eqref{eq:ac-bd-t1-a-c}, \eqref{eq:ac-bd-t1-b-c} and \eqref{eq:ac-bd-t1-bb}, we have
		\$
		\la q_* - q(t), \overline g_t^e - \overline g_*^e \ra \le 2R_{\rm c} \cdot \|g_t^e - \overline g_t^e\|_2 - (1-\beta_{\pi_\theta}) \EE_{\rho_{\pi_\theta}} \bigl[ (\overline Q_{q(t)} - \overline Q_{q_*})^2 \bigr]. 
		\$
		The remaining proof follows \cite{fu2020single}.	
\end{proof}

\section{Proofs for Section \ref{sec:error}}\label{appendix:error} \label{812354}
\begin{proof}[Proof of Lemma \ref{lem:error-kl}]
	We first have by \eqref{eq:pi-new-true}, and recall that we restrict  $\pi_{\theta_{k}} \propto \exp(\tau_{k}^{-1} f_{\theta_{k}})$,
	\$
	&\pi_{k+1}(a\,|\,s)=\exp\bigl(\beta_k^{-1}(1 + 2\overline{\lambda}_k\overline{y}_k)Q^{\pi_{\theta_k}}(s,a) - \beta_k^{-1}\overline{\lambda}_kW^{\pi_{\theta_k}}(s,a) +\tau_k^{-1}f_{\theta_k}(s,a)\bigr)/Z_{k+1}(s),
	\$
	and
	\$
	&\pi_{\theta_{k+1}}(a\,|\,s)=\exp\bigl(\tau_{k+1}^{-1}f_{\theta_{k+1}}(s,a)\bigr)/Z_{\theta_{k+1}}(s),
	\$
	where $Z_{k+1}(s), Z_{\theta_{k+1}}(s)\in \RR$ are normalization factors, which are defined as
	\#\label{eq:normalization}
	Z_{k+1}(s)& =\sum_{a'\in\cA}\exp\bigl(\beta_k^{-1}(1 + 2\overline{\lambda}_k\overline{y}_k)Q^{\pi_{\theta_k}}(s,a') - \beta_k^{-1}\overline{\lambda}_kW^{\pi_{\theta_k}}(s,a') +\tau_k^{-1}f_{\theta_k}(s,a')\bigr),\notag\\
	Z_{\theta_{k+1}}(s) & =\sum_{a'\in\cA} \exp\bigl(\tau_{k+1}^{-1}f_{\theta_{k+1}}(s,a')\bigr),
	\#
	respectively.
	 Then, we reformulate the inner product in \eqref{812759} as 
	\# \label{eq:l04}
	&\la \log\pi_{\theta_{k+1}}(\cdot\,|\,s)-\log\pi_{k+1}(\cdot\,|\,s), \pi^*(\cdot\,|\,s)-\pi_{\theta_k}(\cdot\,|\,s) \ra \notag\\
	&\qquad=\big\la  \tau_{k+1}^{-1}f_{\theta_{k+1}}(s,\cdot)- \big(\beta_k^{-1}(1 + 2\overline{\lambda}_k\overline{y}_k)Q^{\pi_{\theta_k}}(s,\cdot) - \beta_k^{-1}\overline{\lambda}_kW^{\pi_{\theta_k}}(s,\cdot)  \notag \\
	&\qquad\qquad+\tau_k^{-1}f_{\theta_k}(s,\cdot)\big), \pi^*(\cdot\,|\,s)-\pi_{\theta_k}(\cdot\,|\,s) \big\ra,
	\#
	where we use the fact that
	\$
	&\la \log Z_{k+1}(s)-\log Z_{\theta_{k+1}}(s),  \pi^*(\cdot\,|\,s)-\pi_{\theta_k}(\cdot\,|\,s)\ra \notag\\
	&\qquad
	= (\log Z_{k+1}(s)-\log Z_{\theta_{k+1}}(s)) \sum_{a'\in \cA} \bigl( \pi^*(a'\,|\,s)-\pi_{\theta_k}(a'\,|\,s)\bigr) = 0.
	\$
	Thus, it remains to upper bound the right-hand side of \eqref{eq:l04}. We first decompose it to three terms, namely the error from learning the Q-function and the error from fitting the improved policy,  that is,
	{\small
	\#\label{812951}
	&\la  \tau_{k+1}^{-1}f_{\theta_{k+1}}(s,\cdot)- (\beta_k^{-1}(1 + 2\overline{\lambda}_k\overline{y}_k)Q^{\pi_{\theta_k}}(s,\cdot) - \beta_k^{-1}\overline{\lambda}_kW^{\pi_{\theta_k}}(s,\cdot) +\tau_k^{-1}f_{\theta_k}(s,\cdot)), \pi^*(\cdot\,|\,s)-\pi_{\theta_k}(\cdot\,|\,s) \ra \notag\\
	&\qquad=
	\underbrace{\la \tau_{k+1}^{-1}f_{\theta_{k+1}}(s,\cdot)- (\beta_k^{-1}(1 + 2\overline{\lambda}_k\overline{y}_k)Q_{q_k}(s,\cdot) - \beta_k^{-1}\overline{\lambda}_kW_{\omega_k}(s,\cdot)+\tau_k^{-1}f_{\theta_k}(s,\cdot)), \pi^*(\cdot\,|\,s)-\pi_{\theta_k}(\cdot\,|\,s)  \ra}_{\displaystyle{\rm (i)}} \notag\\
	&\qquad\qquad
	+\underbrace{(1 + 2\overline{\lambda}_k\overline{y}_k)\cdot\la \beta_k^{-1}Q_{q_k}(s,\cdot)-\beta_k^{-1}Q^{\pi_{\theta_k}}(s,\cdot), \pi^*(\cdot\,|\,s)-\pi_{\theta_k}(\cdot\,|\,s) \ra}_{\displaystyle{\rm (ii)}} \notag\\
	&\qquad\qquad
	+\underbrace{\overline{\lambda}_k\cdot\la \beta_k^{-1}W^{\pi_{\theta_k}}(s,\cdot) - \beta_k^{-1}W_{\omega_k}(s,\cdot), \pi^*(\cdot\,|\,s)-\pi_{\theta_k}(\cdot\,|\,s) \ra}_{\displaystyle{\rm (iii)}}.
	\#}
	\vskip4pt
	\noindent{\bf Upper Bounding (i):} We have
	{\small
	\#\label{eq:l01}
	& \la  \tau_{k+1}^{-1}f_{\theta_{k+1}}(s,\cdot)- (\beta_k^{-1}(1 + 2\overline{\lambda}_k\overline{y}_k)Q_{q_k}(s,\cdot) - \beta_k^{-1}\overline{\lambda}_kW_{\omega_k}(s,\cdot) +\tau_k^{-1}f_{\theta_k}(s,\cdot)), \pi^*(\cdot\,|\,s)-\pi_{\theta_k}(\cdot\,|\,s) \ra  \notag\\
	&\qquad =
	\biggl\la \tau_{k+1}^{-1}f_{\theta_{k+1}}(s,\cdot)- (\beta_k^{-1}(1 + 2\overline{\lambda}_k\overline{y}_k)Q_{q_k}(s,\cdot) - \beta_k^{-1}\overline{\lambda}_kW_{\omega_k}(s,\cdot) \notag\\
	&\qquad\qquad +\tau_k^{-1}f_{\theta_k}(s,\cdot)),\pi_{0}(\cdot\,|\,s)\cdot\biggl(\frac{\pi^*(\cdot\,|\,s)}{\pi_{0}(\cdot\,|\,s)}-\frac{\pi_{\theta_k}(\cdot\,|\,s)}{\pi_{0}(\cdot\,|\,s)}\biggr) \biggr\ra .
	\#}
	Taking expectation with respect to $s\sim\nu^*$ on the both sides of \eqref{eq:l01}, we obtain
	{\small\$ 
	&\bigl|\EE_{\nu^*}[\la  \tau_{k+1}^{-1}f_{\theta_{k+1}}(s,\cdot)- (\beta_k^{-1}(1 + 2\overline{\lambda}_k\overline{y}_k)Q_{q_k}(s,\cdot) - \beta_k^{-1}\overline{\lambda}_kW_{\omega_k}(s,\cdot)  +\tau_k^{-1}f_{\theta_k}(s,\cdot)),\pi^*(\cdot\,|\,s)-\pi_{\theta_k}(\cdot\,|\,s) \ra ]\bigr|\notag\\
	&\qquad = \biggl|\int_\cS\biggl\la \tau_{k+1}^{-1}f_{\theta_{k+1}}(s,\cdot)- (\beta_k^{-1}(1 + 2\overline{\lambda}_k\overline{y}_k)Q_{q_k}(s,\cdot) - \beta_k^{-1}\overline{\lambda}_kW_{\omega_k}(s,\cdot) +\tau_k^{-1}f_{\theta_k}(s,\cdot)), \notag\\
	&\qquad\qquad\pi_{0}(\cdot\,|\,s)\cdot\biggr(\frac{\pi^*(\cdot\,|\,s)}{\pi_{0}(\cdot\,|\,s)}-\frac{\pi_{\theta_k}(\cdot\,|\,s)}{\pi_{0}(\cdot\,|\,s)}\biggl) \biggr\ra \cdot \nu^*(s)\ud s\biggr|\notag\\
	&\qquad = \biggl|\int_{\cS\times\cA}\bigl(\tau_{k+1}^{-1}f_{\theta_{k+1}}(s,a)- (\beta_k^{-1}(1 + 2\overline{\lambda}_k\overline{y}_k)Q_{q_k}(s,a) - \beta_k^{-1}\overline{\lambda}_kW_{\omega_k}(s,a) +\tau_k^{-1}f_{\theta_k}(s,a)) \bigr)  \notag\\ 
	&\qquad\qquad \cdot \biggl(\frac{\pi^*(\cdot\,|\,s)}{\pi_{0}(\cdot\,|\,s)}-\frac{\pi_{\theta_k}(\cdot\,|\,s)}{\pi_{0}(\cdot\,|\,s)}\biggr) \ud {\sigma}_k(s,a)\biggr| .
	\$}
	By Cauchy-Schwarz inequality, we further have
	{\small\# \label{812946}
	&\bigl|\EE_{\nu^*}[\la  \tau_{k+1}^{-1}f_{\theta_{k+1}}(s,\cdot)- (\beta_k^{-1}(1 + 2\overline{\lambda}_k\overline{y}_k)Q_{q_k}(s,\cdot) - \beta_k^{-1}\overline{\lambda}_kW_{\omega_k}(s,\cdot)  +\tau_k^{-1}f_{\theta_k}(s,\cdot)),\pi^*(\cdot\,|\,s)-\pi_{\theta_k}(\cdot\,|\,s) \ra ]\bigr|\notag\\
	&\qquad \le \EE_{{\sigma}_k}[\bigl(\tau_{k+1}^{-1}f_{\theta_{k+1}}(s,a)- (\beta_k^{-1}(1 + 2\overline{\lambda}_k\overline{y}_k)Q_{q_k}(s,a) - \beta_k^{-1}\overline{\lambda}_kW_{\omega_k}(s,a)+\tau_k^{-1}f_{\theta_k}(s,a))\bigr)^2]^{1/2}  \notag\\ 
	&\qquad\qquad 
	\cdot \EE_{{\sigma}_k}\biggl[\biggl| \biggl(\frac{\ud \pi^*}{\ud \pi_{0}}-\frac{ \ud \pi_{\theta_k}}{\ud \pi_{0}}\biggr)\biggr|^2\biggr]^{1/2}\notag\\
	&\qquad \le \tau_{k+1}^{-1}\epsilon_{k+1}\cdot\phi^*_{k},
	\#}
	where in the last inequality holds by  \eqref{eq:error-kl1} and the definition of $\phi^*_k$ in \eqref{eq:w0w}.
	
	\vskip4pt
	\noindent{\bf Upper Bounding (ii):} By the updating rule of  $\lambda$ and $y$, we have $|\overline{\lambda}_k |\le N$ and $| \overline{y}_k | \le M$. Thus, we have
	{\small\$
	&|(1 + 2\overline{\lambda}_k\overline{y}_k)\cdot \EE_{\nu^*}[\la \beta_k^{-1}Q_{q_k}(s,\cdot)-\beta_k^{-1}Q^{\pi_{\theta_k}}(s,\cdot), \pi^*(\cdot\,|\,s)-\pi_{\theta_k}(\cdot\,|\,s) \ra]|\notag\\
	&\qquad \leq (1 + 2MN)\cdot |\EE_{\nu^*}[\la \beta_k^{-1}Q_{q_k}(s,\cdot)-\beta_k^{-1}Q^{\pi_{\theta_k}}(s,\cdot) , \pi^*(\cdot\,|\,s)-\pi_{\theta_k}(\cdot\,|\,s) \ra]|\notag\\
	& \qquad= (1 + 2MN)\cdot\biggl|\int_{\cS\times\cA}( \beta_k^{-1}Q_{q_k}(s, a)-\beta_k^{-1}Q^{\pi_{\theta_k}}(s, a))\cdot\biggl(\frac{\pi^*(a\,|\,s)}{\pi_{\theta_k}(a\,|\,s)}-\frac{\pi_{\theta_k}(a\,|\,s)}{\pi_{\theta_k}(a\,|\,s)}\biggr) \cdot\frac{\nu^*(s)}{\nu_{k}(s)}\ud \sigma_k(s,a)\biggr| \notag .
	\$}
	By Cauchy-Schwartz inequality, we further have
	\#\label{812947}
	&|(1 + 2\overline{\lambda}_k\overline{y}_k)\cdot \EE_{\nu^*}[\la \beta_k^{-1}Q_{q_k}(s,\cdot)-\beta_k^{-1}Q^{\pi_{\theta_k}}(s,\cdot) , \pi^*(\cdot\,|\,s)-\pi_{\theta_k}(\cdot\,|\,s) \ra]|\notag\\
	&\qquad\le (1 + 2MN)\cdot\EE_{\sigma_k}[(\beta_k^{-1}Q_{q_k}(s,a)-\beta_k^{-1}Q^{\pi_{\theta_k}}(s,a))^2]^{1/2}
	\cdot\EE_{\sigma_k}\biggl[\biggl|\frac{\ud\sigma^*}{\ud\sigma_{k}} - \frac{\ud\nu^*}{\ud\nu_{k}}\biggr|^2\biggr]^{1/2}\notag\\
	& \qquad\le (1 + 2MN)\cdot\beta_k^{-1}\epsilon_{k}'\cdot\psi^*_{k},
	\#
	where  the last inequality holds by the error bound \eqref{eq:error-kl2} and the definition of $\psi^*_k$  \eqref{eq:w0w}.

	\vskip4pt
	\noindent{\bf Upper Bounding (iii):} By the updating rule of  $\lambda$, we have $|\overline{\lambda}_k |\le N$. Thus, we have
	\$
	&|\overline{\lambda}_k\cdot \EE_{\nu^*}[\la \beta_k^{-1}W^{\pi_{\theta_k}}(s,\cdot)- \beta_k^{-1}W_{\omega_k}(s,\cdot), \pi^*(\cdot\,|\,s)-\pi_{\theta_k}(\cdot\,|\,s) \ra]|\notag\\
	&\qquad \leq N\cdot |\EE_{\nu^*}[\la \beta_k^{-1}W^{\pi_{\theta_k}}(s,\cdot)- \beta_k^{-1}W_{\omega_k}(s,\cdot) , \pi^*(\cdot\,|\,s)-\pi_{\theta_k}(\cdot\,|\,s) \ra]|\notag\\
	& \qquad= N\cdot\biggl|\int_{\cS\times\cA}(\beta_k^{-1}W^{\pi_{\theta_k}}(s,a)- \beta_k^{-1}W_{\omega_k}(s,a) )\cdot\biggl(\frac{\pi^*(a\,|\,s)}{\pi_{\theta_k}(a\,|\,s)}-\frac{\pi_{\theta_k}(a\,|\,s)}{\pi_{\theta_k}(a\,|\,s)}\biggr) \cdot\frac{\nu^*(s)}{\nu_{k}(s)}\ud \sigma_k(s,a)\biggr| \notag .
	\$
		By Cauchy-Schwartz inequality, we further have
	\#\label{812948}
	&|\overline{\lambda}_k\cdot \EE_{\nu^*}[\la \beta_k^{-1}W^{\pi_{\theta_k}}(s,\cdot)- \beta_k^{-1}W_{\omega_k}(s,\cdot), \pi^*(\cdot\,|\,s)-\pi_{\theta_k}(\cdot\,|\,s) \ra]|\notag\\
	&\qquad\le N\cdot\EE_{\sigma_k}[(\beta_k^{-1}W^{\pi_{\theta_k}}(s,a)- \beta_k^{-1}W_{\omega_k}(s,a))^2]^{1/2}
	\cdot\EE_{\sigma_k}\biggl[\biggl|\frac{\ud\sigma^*}{\ud\sigma_{k}} - \frac{\ud\nu^*}{\ud\nu_{k}}\biggr|^2\biggr]^{1/2}\notag\\
	& \qquad\le N\cdot\beta_k^{-1}\epsilon_{k}''\cdot\psi^*_{k},
	\#
	Finally, combining \eqref{eq:l04}, \eqref{812951}, \eqref{812946}, \eqref{812947} and \eqref{812948}, we have 
	\$
	&|\EE_{\nu^*}[\la \log\pi_{\theta_{k+1}}(\cdot\,|\,s)-\log\pi_{k+1}(\cdot\,|\,s), \pi^*(\cdot\,|\,s)-\pi_{\theta_k}(\cdot\,|\,s) \ra]| \\
	&\qquad\le \tau_{k+1}^{-1}\epsilon_{k+1}\cdot\phi^*_{k}+(1+2MN)\cdot\beta_k^{-1}\epsilon_{k}'\cdot\psi^*_{k} +N\cdot\beta_k^{-1}\epsilon_{k}''\cdot\psi^*_{k} ,
	\$
	which concludes the proof.
\end{proof}

\begin{proof}[Proof of Lemma \ref{lem:stepwise-energy}]
	By the triangle inequality, we have
	\#\label{812424}
	&\| \tau_{k+1}^{-1} f_{\theta_{k+1}}(s,\cdot) - \tau_k^{-1}f_{\theta_k}(s,\cdot)  \|_{\infty}^2 \notag\\
	&\qquad\le
	2\| \tau_{k+1}^{-1} f_{\theta_{k+1}}(s,\cdot) - \tau_k^{-1}f_{\theta_k}(s,\cdot)-\beta_k^{-1}(1+2\overline{\lambda}_k\overline{y}_k)Q_{q_k}(s,\cdot) + \beta^{-1}\overline{\lambda}_kW_{\omega_k}  \|_{\infty}^2  \notag\\
	&\qquad\qquad +2 \| \beta_k^{-1}(1+2\overline{\lambda}_k\overline{y}_k)Q_{q_k}(s,\cdot) - \beta^{-1}\overline{\lambda}_kW_{\omega_k} \|_{\infty}^2 \notag\\
	&\qquad\le
	2\| \tau_{k+1}^{-1} f_{\theta_{k+1}}(s,\cdot) - \tau_k^{-1}f_{\theta_k}(s,\cdot)-\beta_k^{-1}(1+2\overline{\lambda}_k\overline{y}_k)Q_{q_k}(s,\cdot) + \beta^{-1}\overline{\lambda}_kW_{\omega_k}  \|_{\infty}^2  \notag\\
	&\qquad\qquad +4 \| \beta_k^{-1}(1+2\overline{\lambda}_k\overline{y}_k)Q_{q_k}(s,\cdot)\|_{\infty}^2 +4 \| \beta^{-1}\overline{\lambda}_kW_{\omega_k} \|_{\infty}^2.
	\#
	For the first term on the right-hand side of \eqref{812424}, by Lemma \ref{lem:error-kl}, we have
	\#\label{812432}
	\EE_{\nu^*}[\| \tau_{k+1}^{-1} f_{\theta_{k+1}}(s,\cdot) - \tau_k^{-1}f_{\theta_k}(s,\cdot)-\beta_k^{-1}(1+2\overline{\lambda}_k\overline{y}_k)Q_{q_k}(s,\cdot) + \beta^{-1}\overline{\lambda}_kW_{\omega_k}   \|_{\infty}^2] \le |\cA|\cdot \tau_{k+1}^{-2}\epsilon_{k+1}^2.
	\#
	For the second term on the right-hand side of \eqref{812424}, we have
	\#\label{812433}
	\EE_{\nu^*}[\|  \beta_k^{-1}(1+2\overline{\lambda}_k\overline{y}_k)Q_{q_k}(s,\cdot) \|_{\infty}^2] \le \beta_k^{-2}\cdot (1+2MN)^2 \cdot \EE_{\nu^*}\Bigl[ \max_{a\in\cA} 2(Q_{q_0}(s,a))^2 + 2R_{\rm c}^2 \Bigr],
	\#
	where we use the $1$-Lipschitz continuity of $Q_{\omega}$ in $\omega$ and the constraint $\|\omega_k-\omega_0\|_2\le R_\omega$.
		For the third term on the right-hand side of \eqref{812424}, we have
	\#\label{812434}
	\EE_{\nu^*}[\|\beta^{-1}\overline{\lambda}_kW_{\omega_k}   \|_{\infty}^2] \le \beta_k^{-2}\cdot N^2 \cdot \EE_{\nu^*}\Bigl[ \max_{a\in\cA} 2(W_{\omega_0}(s,a))^2 + 2R_{\rm b}^2 \Bigr].
	\# 
	Then, taking expectation with respect to $s\sim\nu^*$ on  both sides of \eqref{812424} and plugging \eqref{812432}, \eqref{812433} and \eqref{812434} in, the result holds as desired.
\end{proof}

\section{Proofs of Section \ref{sec:main} }\label{appendix:main}

\begin{proof}[Proof of Lemma \ref{lem:v-diff}]
By the definition of $\rho(\pi)$ in \eqref{eq:rho}, we have
\#\label{eq:main-lemma-proof1}
\rho(\pi^*) - \rho(\pi) =\EE_{\nu^*\pi^*}[r(s,a)] - \rho(\pi) =\EE_{\nu^*\pi^*}[r(s,a) - \rho(\pi)] .
\#
By the Bellman equation that $Q^{\pi}(s,a) = r(s,a) - \rho(\pi) +V^{\pi}(s')$, we have 
\#\label{eq:main-lemma-proof6}
\EE_{\nu^*\pi^*}[r(s,a) - \rho(\pi)] =\EE_{\nu^*\pi^*}[Q^{\pi}(s,a) - V^{\pi}(s')]  =\EE_{\nu^*\pi^*}[Q^{\pi}(s,a) - V^{\pi}(s)] ,
\#
where the last equality follows from $(\cP^{\pi^*})^t\nu^* = \nu^*$. Finally, note that for any given $s\in \cS$,
\#\label{eq:main-lemma-proof2}
\EE_{\pi^*}[Q^\pi(s, a) - V^\pi(s)] &= \la Q^\pi(s,\cdot) , \pi^*(\cdot\,|\, s) \ra - \la Q^\pi(s,\cdot), \pi(\cdot\,|\,s)\ra\notag\\
 &= \la Q^\pi(s,\cdot), \pi^*(\cdot\,|\,s) - \pi(\cdot\,|\,s)\ra.
\#
Plugging \eqref{eq:main-lemma-proof6} and \eqref{eq:main-lemma-proof2} into \eqref{eq:main-lemma-proof1}, we obtain
\# \label{eq:main-lemma-proof3}
\rho(\pi^*) - \rho(\pi) = \EE_{\nu^*}[\la Q^{\pi}(s,\cdot), \pi^*(\cdot\,|\,s) - \pi(\cdot\,|\,s)\ra] .
\#
Similarly, by the definition of $\eta(\pi)$ in \eqref{eq:eta}, we obtain
\#\label{eq:main-lemma-proof4}
\eta(\pi^*) - \eta(\pi) =\EE_{\nu^*\pi^*}[r(s,a)^2] - \eta(\pi) =\EE_{\nu^*\pi^*}[r(s,a)^2 - \eta(\pi)] .
\#
By the equation $r(s,a)^2 - \eta(\pi) = W^{\pi}(s,a) - U^{\pi}(s')$, we further have
\#\label{eq:main-lemma-proof7}
\EE_{\nu^*\pi^*}[r(s,a)^2 - \eta(\pi)] 
=\EE_{\nu^*\pi^*}[W^{\pi}(s,a) - U^{\pi}(s')] 
=\EE_{\nu^*\pi^*}[W^{\pi}(s,a) - U^{\pi}(s)] ,
\#
where the last equality follows from $(\cP^{\pi^*})^t\nu^* = \nu^*$. In addition, note that for any given $s\in \cS$,
\#\label{eq:main-lemma-proof8}
\EE_{\nu^*\pi^*}[W^{\pi}(s,a) - U^{\pi}(s)]  &= \la W^\pi(s,\cdot) , \pi^*(\cdot\,|\, s) \ra - \la W^\pi(s,\cdot), \pi(\cdot\,|\,s)\ra\notag\\
&= \EE_{\nu^*}[\la W^{\pi}(s,\cdot), \pi^*(\cdot\,|\,s) - \pi(\cdot\,|\,s)\ra] .
\#
Plugging \eqref{eq:main-lemma-proof7} and \eqref{eq:main-lemma-proof8} into \eqref{eq:main-lemma-proof4}, we have
\#\label{eq:main-lemma-proof9}
\eta(\pi^*) - \eta(\pi) = \EE_{\nu^*}[\la W^{\pi}(s,\cdot), \pi^*(\cdot\,|\,s) - \pi(\cdot\,|\,s)\ra] .
\#
Combining \eqref{eq:main-lemma-proof3} and \eqref{eq:main-lemma-proof9}, and by the definition of $\cL(\pi)$ in \eqref{eq:def-L}, we obtain
\$%#\label{eq:main-lemma-proof5}
\cL(\lambda, \pi^*, y) - \cL(\lambda, \pi, y) &=(1 + 2\lambda y)(\rho(\pi^*) - \rho(\pi)) - \lambda(\eta(\pi^*) - \eta(\pi)) \notag\\
&=\EE_{\nu^*}[\la (1 + 2\lambda y)Q^{\pi}(s,\cdot) - \lambda W^{\pi}(s, \cdot), \pi^*(\cdot\,|\,s) - \pi(\cdot\,|\,s)\ra] ,
\$
which completes the proof.
\end{proof}

\begin{proof}[Proof of Lemma \ref{lem:main-descent}]
First, we have
{\small\#\label{eq:l03}
& {\rm KL}\bigl(\pi^*(\cdot\,|\,s)\,\|\,\pi_{\theta_{k}}(\cdot\,|\,s)\bigr) - {\rm KL}\bigl(\pi^*(\cdot\,|\,s)\,\|\,\pi_{\theta_{k+1}}(\cdot\,|\,s)\bigr)\notag\\
& \qquad= \la \log(\pi_{\theta_{k+1}}(\cdot\,|\,s)/\pi_{\theta_k}(\cdot\,|\,s)), \pi^*(\cdot\,|\,s)  \ra\notag\\
& \qquad= \la \log(\pi_{\theta_{k+1}}(\cdot\,|\,s)/\pi_{\theta_k}(\cdot\,|\,s)), \pi^*(\cdot\,|\,s) - \pi_{\theta_{k+1}}(\cdot\,|\,s)\ra + {\rm KL}\bigl(\pi_{\theta_{k+1}}(\cdot\,|\,s)\,\|\,\pi_{\theta_{k}}(\cdot\,|\,s)\bigr)\notag\\
& \qquad= \la \log(\pi_{\theta_{k+1}}(\cdot\,|\,s)/\pi_{\theta_k}(\cdot\,|\,s)) - \beta_k^{-1}(1+\overline{\lambda}_k\overline{y}_k)Q^{\pi_{\theta_k}}(s, \cdot) +\beta_k^{-1}\overline{\lambda}_kW^{\pi_{\theta_k}}(s, \cdot), \pi^*(\cdot\,|\,s) - \pi_{\theta_{k}}(\cdot\,|\,s)\ra\notag\\
&\qquad\qquad+ \beta_k^{-1}\cdot\la (1+\overline{\lambda}_k\overline{y}_k)Q^{\pi_{\theta_k}}(s, \cdot) - \overline{\lambda}_kW^{\pi_{\theta_k}}(s, \cdot), \pi^*(\cdot\,|\,s) - \pi_{\theta_{k}}(\cdot\,|\,s)\ra + {\rm KL}\bigl(\pi_{\theta_{k+1}}(\cdot\,|\,s)\,\|\,\pi_{\theta_{k}}(\cdot\,|\,s)\bigr)\notag\\
& \qquad\qquad + \la \log(\pi_{\theta_{k+1}}(\cdot\,|\,s)/\pi_{\theta_k}(\cdot\,|\,s)), \pi_{\theta_{k}}(\cdot\,|\,s) - \pi_{\theta_{k+1}}(\cdot\,|\,s)\ra.
\#}
Recall that $\pi_{k+1}\propto \exp(\tau_k^{-1}f_{\theta_k} +\beta_k^{-1}(1+\overline{\lambda}_k\overline{y}_k)Q^{\pi_{\theta_k}} -\beta_k^{-1}\overline{\lambda}_kW^{\pi_{\theta_k}})$ and $Z_{k+1}(s)$, and $Z_{\theta_k}(s)$ are defined in \eqref{eq:normalization}. Also recall that we have $ \la \log Z_{\theta_k}(s), \pi(\cdot\,|\,s) - \pi'(\cdot\,|\,s)\ra =  \la \log Z_{k}(s), \pi(\cdot\,|\,s) - \pi'(\cdot\,|\,s)\ra = 0$ for all $k$, $\pi$, and $\pi'$, which implies that, on the right-hand-side of \eqref{eq:l03}, 
\#\label{eq:l05}
& \la \log\pi_{\theta_k}(\cdot\,|\,s) + \beta_k^{-1}(1+\overline{\lambda}_k\overline{y}_k)Q^{\pi_{\theta_k}}(s, \cdot) -\beta_k^{-1}\overline{\lambda}_kW^{\pi_{\theta_k}}(s, \cdot), \pi^*(\cdot\,|\,s) - \pi_{\theta_{k}}(\cdot\,|\,s)\ra\notag\\
& \qquad= \la \tau_k^{-1}f_{\theta_k}(s, \cdot) + \beta_k^{-1}(1+\overline{\lambda}_k\overline{y}_k)Q^{\pi_{\theta_k}}(s, \cdot) -\beta_k^{-1}\overline{\lambda}_kW^{\pi_{\theta_k}}(s, \cdot), \pi^*(\cdot\,|\,s) - \pi_{\theta_{k}}(\cdot\,|\,s)\ra  \notag\\
& \qquad\qquad - \la \log Z_{\theta_{k}}(s), \pi^*(\cdot\,|\,s) - \pi_{\theta_k}(\cdot\,|\,s)\ra\notag\\
& \qquad= \la \tau_k^{-1}f_{\theta_k}(s, \cdot) + \beta_k^{-1}(1+\overline{\lambda}_k\overline{y}_k)Q^{\pi_{\theta_k}}(s, \cdot) -\beta_k^{-1}\overline{\lambda}_kW^{\pi_{\theta_k}}(s, \cdot), \pi^*(\cdot\,|\,s) - \pi_{\theta_{k}}(\cdot\,|\,s)\ra  \notag\\
& \qquad\qquad - \la \log Z_{k+1}(s), \pi^*(\cdot\,|\,s) - \pi_{\theta_k}(\cdot\,|\,s)\ra\notag\\
& \qquad =  \la \log\pi_{k+1}(\cdot\,|\,s), \pi^*(\cdot\,|\,s) - \pi_{\theta_{k}}(\cdot\,|\,s)\ra,
\#
and
\#\label{eq:l06}
&\la \log(\pi_{\theta_{k+1}}(\cdot\,|\,s)/\pi_{\theta_k}(\cdot\,|\,s)), \pi_{\theta_{k}}(\cdot\,|\,s) - \pi_{\theta_{k+1}}(\cdot\,|\,s) \ra\notag\\
& \qquad = \la \tau_{k+1}^{-1}f_{\theta_{k+1}}(s, \cdot) - \tau_k^{-1}f_{\theta_k}(s, \cdot), \pi_{\theta_{k}}(\cdot\,|\,s) - \pi_{\theta_{k+1}}(\cdot\,|\,s) \ra\notag\\
&\qquad\qquad - \la \log Z_{\theta_{k+1}}(s), \pi_{\theta_k}(\cdot\,|\,s) - \pi_{\theta_{k+1}}(\cdot\,|\,s)\ra + \la \log Z_{\theta_{k}}(s), \pi_{\theta_k}(\cdot\,|\,s) - \pi_{\theta_{k+1}}(\cdot\,|\,s)\ra\notag\\
&\qquad= \la \tau_{k+1}^{-1}f_{\theta_{k+1}}(s, \cdot) - \tau_k^{-1}f_{\theta_k}(s, \cdot), \pi_{\theta_{k}}(\cdot\,|\,s) - \pi_{\theta_{k+1}}(\cdot\,|\,s) \ra.
\#
Plugging \eqref{eq:l05} and \eqref{eq:l06} into \eqref{eq:l03}, we obtain
\#\label{eq:l07}
& {\rm KL}\bigl(\pi^*(\cdot\,|\,s)\,\|\,\pi_{\theta_{k}}(\cdot\,|\,s)\bigr) - {\rm KL}\bigl(\pi^*(\cdot\,|\,s)\,\|\,\pi_{\theta_{k+1}}(\cdot\,|\,s)\bigr)\\
&\qquad = \la \log(\pi_{\theta_{k+1}}(\cdot\,|\,s)/\pi_{k+1}(\cdot\,|\,s)), \pi^*(\cdot\,|\,s) - \pi_{\theta_{k}}(\cdot\,|\,s)\ra \notag\\
&\qquad\qquad+ \beta_k^{-1}\cdot\la(1+2\overline{y}_k\overline{\lambda}_k)Q^{\pi_{\theta_k}}(s, \cdot) - \overline{\lambda}_kW^{\pi_{\theta_k}}(s, \cdot), \pi^*(\cdot\,|\,s) - \pi_{\theta_{k}}(\cdot\,|\,s)\ra\notag\\
&\qquad\qquad+\la \tau_{k+1}^{-1}f_{\theta_{k+1}}(s, \cdot) - \tau_k^{-1}f_{\theta_k}(s, \cdot), \pi_{\theta_{k}}(\cdot\,|\,s) - \pi_{\theta_{k+1}}(\cdot\,|\,s) \ra + {\rm KL}\bigl(\pi_{\theta_{k+1}}(\cdot\,|\,s)\,\|\,\pi_{\theta_{k}}(\cdot\,|\,s)\bigr)\notag\\
&\qquad = \la \log(\pi_{\theta_{k+1}}(\cdot\,|\,s)/\pi_{k+1}(\cdot\,|\,s)), \pi^*(\cdot\,|\,s) - \pi_{\theta_{k}}(\cdot\,|\,s) \ra \notag\\
&\qquad\qquad+ \beta_k^{-1}\cdot\la(1+2y^*\overline{\lambda}_k)Q^{\pi_{\theta_k}}(s, \cdot) - \overline{\lambda}_kW^{\pi_{\theta_k}}(s, \cdot), \pi^*(\cdot\,|\,s) - \pi_{\theta_{k}}(\cdot\,|\,s)\ra\notag\\
&\qquad\qquad + \beta_k^{-1}\cdot\la(2(\overline{y}_k - y^*)\overline{\lambda}_k)Q^{\pi_{\theta_k}}(s, \cdot), \pi^*(\cdot\,|\,s) - \pi_{\theta_{k}}(\cdot\,|\,s)\ra\notag\\
&\qquad\qquad+\la \tau_{k+1}^{-1}f_{\theta_{k+1}}(s, \cdot) - \tau_k^{-1}f_{\theta_k}(s, \cdot), \pi_{\theta_{k}}(\cdot\,|\,s) - \pi_{\theta_{k+1}}(\cdot\,|\,s) \ra + {\rm KL}\bigl(\pi_{\theta_{k+1}}(\cdot\,|\,s)\,\|\,\pi_{\theta_{k}}(\cdot\,|\,s)\bigr)\notag\\
& \qquad\geq \la \log(\pi_{\theta_{k+1}}(\cdot\,|\,s)/\pi_{k+1}(\cdot\,|\,s)), \pi^*(\cdot\,|\,s) - \pi_{\theta_{k}}(\cdot\,|\,s) \ra \notag\\
&\qquad\qquad+ \beta_k^{-1}\cdot\la(1+2y^*\overline{\lambda}_k)Q^{\pi_{\theta_k}}(s, \cdot) - \overline{\lambda}_kW^{\pi_{\theta_k}}(s, \cdot), \pi^*(\cdot\,|\,s) - \pi_{\theta_{k}}(\cdot\,|\,s)\ra\notag\\
&\qquad\qquad + \beta_k^{-1}\cdot\la(2(\overline{y}_k - y^*)\overline{\lambda}_k)Q^{\pi_{\theta_k}}(s, \cdot), \pi^*(\cdot\,|\,s) - \pi_{\theta_{k}}(\cdot\,|\,s)\ra\notag\\
&\qquad\qquad+\la \tau_{k+1}^{-1}f_{\theta_{k+1}}(s, \cdot) - \tau_k^{-1}f_{\theta_k}(s, \cdot), \pi_{\theta_{k}}(\cdot\,|\,s) - \pi_{\theta_{k+1}}(\cdot\,|\,s) \ra \notag\\
&\qquad\qquad+ \frac{1}{2}\cdot\|\pi_{\theta_{k+1}}(\cdot\,|\,s) - \pi_{\theta_{k}}(\cdot\,|\,s)\|_1^2,\notag
\#
where in the last inequality holds by the Pinsker's inequality. Rearranging the terms in \eqref{eq:l07}, we conclude the proof.
\end{proof}

\begin{proof}[Proof of Corollary  \ref{coro:main-bound}]
	By Lemmas \ref{thm:ac-error}, \ref{thm:td} and \ref{thm:td2}, it holds with probability at least $1- \exp( - \Omega(R_{\rm a}^{2/3} m_{\rm a}^{2/3}H_a )) - \exp( - \Omega(R_{\rm b}^{2/3} m_{\rm b}^{2/3}H_b )) - \exp( - \Omega(R_{\rm c}^{2/3} m_{\rm c}^{2/3}H_c ))$ that
	\$
	&\epsilon_{k+1} = O ( R_{\rm a}^2 T^{-1/2} + R_{\rm a}^{8/3} m_{\rm a}^{-1/6} H_{\rm a}^{7} \log m_{\rm a} ), \\ 
	&\epsilon'_{k} = O ( R_{\rm c}^2 T^{-1/2} + R_{\rm c}^{8/3} m_{\rm c}^{-1/6} H_{\rm c}^{7} \log m_{\rm c} ), \\ 
	&\epsilon''_{k} = O ( R_{\rm b}^2 T^{-1/2} + R_{\rm b}^{8/3} m_{\rm b}^{-1/6} H_{\rm b}^{7} \log m_{\rm b} ).
	\$ 
	By our choice of of the parameters that
	\$
	& R_{\rm a}  = R_{\rm a} = R_{\rm c} = O\bigl(m_{\rm a}^{1/2} H_{\rm a}^{-6}(\log m_{\rm a})^{-3}\bigr),\notag\\
	& T = \Omega\bigl( K^{3}(\phi^*_{k}+\psi^*_{k})^2 |\cA| R_{\rm a}^4 H_am_a^{2/3}\bigr), \\
	& m_{\rm a} =m_{\rm b} = m_{\rm c} = \Omega\bigl( d^{3/2}K^{9} R_{\rm a}^{16}(\phi^*_{k}+\psi^*_{k})^6|\cA|^3 H_{\rm a}^{42} \log^6 m_{\rm a} \bigr) .
	\$
	Thus, it holds with probability at least $1 - 3\exp( - \Omega(R_{\rm a}^{2/3} m_{\rm a}^{2/3}H_a ))$ that
	$\epsilon_{k} \le \cO(K^{-3/2}(\phi^*_k + \psi^*_k)^{-1} |\cA|^{-1/2}),\  \epsilon_{k}' \le \cO(K^{-3/2}(\phi^*_k + \psi^*_k)^{-1} |\cA|^{-1/2}) \text{\ and\ }  \epsilon_{k}'' \le \cO(K^{-3/2}(\phi^*_k + \psi^*_k)^{-1} |\cA|^{-1/2}).$ 
	Recall that we set the temperature parameter $\tau_{k+1} = \beta\sqrt{K}/(k+1)$ and the penalty parameter $\beta_k = \beta\sqrt{K}$. For $\varepsilon_k$ defined in Lemma \ref{lem:error-kl}, we have
	\# \label{eq:cor-1}
	&\varepsilon_k= \tau_{k+1}^{-1}\epsilon_{k+1}\cdot\phi^*_{k}+(1+2MN)\cdot\beta_k^{-1}\epsilon_{k}'\cdot\psi^*_{k} +N\cdot\beta_k^{-1}\epsilon_{k}''\cdot\psi^*_{k} \le \cO(1/K).
	\#
	For $\varepsilon_k'$ defined in Lemma \ref{lem:stepwise-energy}, we have
	\#\label{eq:cor-2}
	&\varepsilon_k' =|\cA|\cdot \tau_{k+1}^{-2}\epsilon_{k+1}^2 \le  \cO(1/K).
	\#
	By Lemma \ref{lem:average-error}, we have 
	\$
	c_k \le O\bigl(T^{-1/2}\log(4K/p)^{1/2}\bigr), \quad d_k \le O\bigl(T^{-1/2}\log(4K/p)^{1/2}\bigr) .
	\$ 
	The parameters we set ensure that $T=\Omega(K\log(4K/p))$ and $p=\exp( - \Omega(R_{\rm a}^{2/3} m_{\rm a}^{2/3}H_a ))$, which further implies that 
	\# \label{eq:cor-3}
	c_k \le \cO(1/\sqrt{K}), \qquad d_k \le  \cO(1/\sqrt{K})
	\#
	with probability at least $1-\exp( - \Omega(R_{\rm a}^{2/3} m_{\rm a}^{2/3}H_a ))$.
	Plugging \eqref{eq:cor-1}, \eqref{eq:cor-2} and \eqref{eq:cor-3} into Theorem \ref{thm:main}, we have, with probability at least $1-4\exp( - \Omega(R_{\rm a}^{2/3} m_{\rm a}^{2/3}H_a ))$, 
	\$
	-\cO(1/\sqrt{K}) \le \frac{1}{K}\sum_{k=0}^{K-1} \bigl( \cL(\lambda^*,\pi^*, y^*) - \cL(\overline{\lambda}_k, \pi_{\theta_k}, \overline{y}_k) \bigr) \le \cO(1/\sqrt{K}) ,
	\$
	which concludes the proof.
\end{proof}

\begin{proof}[Proof of Theorem  \ref{thm:duality-gap}]
	The negativity of duality gap holds by the definition. We only need to show the upper bound. For the optimal solution $(\lambda^*,\pi^*,y^*)$, we have
	\$
	&\frac{1}{K}\sum_{k=0}^{K-1} \bigl(\cL(\overline{\lambda}_k,\pi^*, y^*) - \cL(\lambda^*, \pi_{\theta_k}, \overline{y}_k)\bigr) \\
	&\qquad = \frac{1}{K}\sum_{k=0}^{K-1} \bigl(\cL(\overline{\lambda}_k,\pi^*, y^*) -\cL(\overline{\lambda}_k, \pi_{\theta_k}, \overline{y}_k) +\cL(\overline{\lambda}_k, \pi_{\theta_k}, \overline{y}_k) - \cL(\lambda^*, \pi_{\theta_k}, \overline{y}_k) \bigr).
	\$
	By  \eqref{eq:result-part1} and \eqref{eq:result-part2} in Section \ref{sec:main}, we obtain
	\$ 
	\frac{1}{K}\sum_{k=0}^{K-1} \bigl( \cL(\overline{\lambda}_k, \pi_{\theta_k}, \overline{y}_k) - \cL(\lambda^*, \pi_{\theta_k}, \overline{y}_k) \bigr)  \le  \sum_{k=0}^{K-1} (c_k + d_k)\cdot \cO(1/K)  + \cO(1/\sqrt{K}) ,
	\$
	and
	\$
	\frac{1}{K}\sum_{k=0}^{K-1} \bigl(\cL(\overline{\lambda}_k,\pi^*, y^*) -\cL(\overline{\lambda}_k, \pi_{\theta_k}, \overline{y}_k)\bigr) \le\sum_{k=0}^{K-1}c_k \cdot \cO(1/K) +   \sum_{k=0}^{K-1}(  \varepsilon_k + \varepsilon_k') \cdot  \cO(1/\sqrt{K}) .
	\$
	Thus, we have
	\$
	&\frac{1}{K}\sum_{k=0}^{K-1} \bigl(\cL(\overline{\lambda}_k,\pi^*, y^*) - \cL(\lambda^*, \pi_{\theta_k}, \overline{y}_k)\bigr) \\
	& \qquad  \le  \sum_{k=0}^{K-1} (c_k + d_k)\cdot \cO(1/K)  +  \sum_{k=0}^{K-1}(  \varepsilon_k + \varepsilon_k') \cdot  \cO(1/\sqrt{K}) +  \cO(1/\sqrt{K}) .
	\$
	Moreover, by setting the parameters same as Corollary \ref{coro:main-bound}, together with \eqref{eq:cor-1}, \eqref{eq:cor-2} and \eqref{eq:cor-3}, we have
	\$
	\frac{1}{K}\sum_{k=0}^{K-1} \bigl(\cL(\overline{\lambda}_k,\pi^*, y^*) - \cL(\lambda^*, \pi_{\theta_k}, \overline{y}_k)\bigr) \le \cO(1/\sqrt{K}),
	\$
	which concludes the proof.     
	\end{proof}

	\section{Algorithms in Section \ref{sec:linear}} \label{appendix:alg:linear}

	\begin{algorithm}[H]
		\caption{Variance-Constrained Actor-Critic with Linear Function Approximation}
		\begin{algorithmic}[1]\label{alg:risac2}
		\REQUIRE MDP $(\cS, \cA, \cP, r)$, projection radii $R$, penalty parameter $\beta$, number of SGD and TD iterations $T$ and number of VARAC iterations $K$
		\STATE Initialize with uniform policy: $\tau_0\leftarrow 1$, $f_{\theta_0} \leftarrow 0$, $ \pi_{\theta_0} \leftarrow \pi_0 \propto \exp(\tau_0^{-1}f_{\theta_0})$
		\STATE Sample $\{(s_t, a_t, a^0_t, s_t', a_t')\}^{T}_{t = 1}$ with $(s_t, a_t) \sim \sigma_0$, $a^0_t\sim \pi_0(\cdot\,|\,s_t)$, $s_t'\sim\cP(\cdot\,|\,s_t, a_t)$ and $a_t' \sim \pi_{\theta_0}(\cdot\,|\,s_t')$
		\STATE Estimate $\rho(\pi_{\theta_{0}})$ and $\eta(\pi_{\theta_{0}})$ by $\overline{\rho}(\pi_{\theta_{0}}) = \frac{1}{T}\cdot\sum^{T}_{t=1}r(s_t, a_t)$ and $\overline{\eta}(\pi_{\theta_{0}}) = \frac{1}{T}\cdot\sum^{T}_{t=1}r(s_t, a_t)^2$
		\FOR{$k = 0, \dots, K-1$}

		\STATE Set temperature parameter $\tau_{k+1} \leftarrow \beta\sqrt{K}/(k+1)$ and penalty parameter $\beta_k \leftarrow \beta\sqrt{K}$
		
		\STATE Solve $Q_{q_{k}}(s, a) = q_k^\top \varphi(s, a)$ using the TD update in \eqref{eq:td-update3} (Algorithm \ref{alg:td3})\label{line:td:q2}
		
		\STATE Solve $W_{\omega_{k}}(s, a) = \omega_{k}^\top \varphi(s, a)$ using the TD update in \eqref{eq:td-update4} (Algorithm~\ref{alg:td4})\label{line:td:w2}
		
		\STATE Update $\lambda$ : $\overline{\lambda}_{k+1} =  \Pi_{[0,N]}\bigl(\overline{\lambda}_k - \frac{1}{2\gamma_k}(\alpha + 2\overline{y}_k\overline{\rho}(\pi_{\theta_k}) + \overline{\eta}(\pi_{\theta_k}) -\overline{y}_k^2) \bigr)$

		\STATE Update $\theta_{k+1}$ using \eqref{eq:update:theta2} and calculate $f_{\theta_{k+1}} = \theta_{k + 1}^\top \varphi$ using \label{line:sgd2}
		
		\STATE Update policy: $\pi_{\theta_{k+1}} \propto \exp(\tau_{k+1}^{-1}f_{\theta_{k+1}})$ \label{line:c2}
		\STATE Sample $\{(s_t, a_t, a^0_t, s_t', a_t')\}^{T}_{t = 1}$ with $(s_t, a_t) \sim \sigma_{k+1}$, $a^0_t\sim \pi_0(\cdot\,|\,s_t)$, $s_t'\sim\cP(\cdot\,|\,s_t, a_t)$ and $a_t' \sim \pi_{\theta_{k+1}}(\cdot\,|\,s_t')$
		\STATE Estimate $\rho(\pi_{\theta_{k+1}})$ and $\eta(\pi_{\theta_{k+1}})$ by $\overline{\rho}(\pi_{\theta_{k+1}}) = \frac{1}{T}\cdot\sum^{T}_{t=1}r(s_t, a_t)$ and $\overline{\eta}(\pi_{\theta_{k+1}}) = \frac{1}{T}\cdot\sum^{T}_{t=1}r(s_t, a_t)^2$\label{line:estimate2}
		
		\STATE Update $y$ : $\overline{y}_{k+1} = \overline{\rho}(\pi_{\theta_{k+1}}) $

		\ENDFOR
		\end{algorithmic}
		\end{algorithm}

\begin{algorithm}[H]
	\caption{Update $q$ via TD(0)}
	\begin{algorithmic}[1]\label{alg:td3}
	\STATE {\bf Require:} MDP $(\cS, \cA, \cP, r)$, number of iterations $T$, sample $\{(s_t, a_t, s_t', a_t')\}^{T}_{t = 1}$
	\STATE {\bf Initialization:} $q(0) \leftarrow 0$
	\STATE Set stepsize $\delta \leftarrow T^{-1/2}$
	\FOR{$t = 0, \dots,T - 1$}
	\STATE Sample $(s, a, s', a')$ $\leftarrow (s_{t+1}, a_{t+1}, s'_{t+1}, a'_{t+1})$
	\STATE $q(t+1) \leftarrow \Pi_{\cB{(0, R)}}\bigl( q(t) - \delta\cdot \bigl(Q_{q(t)}(s,a) - r(s, a) + \overline{\rho}(\pi_{\theta_{k}}) - Q_{q(t)}(s', a')\bigr)\cdot\varphi(s,a)\bigr)$
	\ENDFOR
	\STATE Average over path $\overline{q} \leftarrow 1/T \cdot \sum^{T - 1}_{t = 0}q(t)$
	\STATE {\bf Output:} $Q_{\overline{q}}$
	\end{algorithmic}
	\end{algorithm}
	
	\begin{algorithm}[H]
		\caption{Update $\omega$ via TD(0)}
		\begin{algorithmic}[1]\label{alg:td4}
			\STATE {\bf Require:} MDP $(\cS, \cA, \cP, r, \gamma)$, number of iterations $T$, sample $\{(s_t, a_t, s_t', a_t')\}^{T}_{t = 1}$
			\STATE {\bf Initialization:} $\omega(0) \leftarrow 0$
			\STATE Set stepsize $\delta \leftarrow T^{-1/2}$
			\FOR{$t = 0, \dots,T-1$}
			\STATE Sample $(s, a, s', a') \leftarrow (s_{t+1}, a_{t+1}, s'_{t+1}, a'_{t+1})$
			\STATE $\omega(t+ 1) \leftarrow \Pi_{\cB(\omega_0, R_b)}\bigl(\omega(t) - \delta\cdot \bigl(W_{\omega(t)}(s,a) - r(s, a) + \overline{\eta}(\pi_{\theta_{k}}) - W_{\omega(t)}(s', a')\bigr)\cdot\varphi(s,a)\bigr)$
			\ENDFOR
			\STATE Average over path $\overline{\omega} \leftarrow 1/T \cdot \sum^{T-1}_{t = 0}\omega(t)$
			\STATE {\bf Output:} $W_{\overline{\omega}}$
		\end{algorithmic}
	\end{algorithm}

\section{Proof of Lemma \ref{lemma:q:error:linear}} \label{appendix:td:linear}

\begin{proof}
	For notational simplicity, we omit the dependence of $k$ and use $\theta$ to denote $\theta_k$.
	With slight abuse of notation, we denote by 
	\# \label{eq:local:linear2}
		& g_t = \bigl( Q_{q(t)}(s,a) -  Q_{q(t)}(s', a') -  r_0  + \overline{\rho}(\pi_\theta) \bigr)\cdot \varphi(s,a), && g_t^e = \EE_{\pi_\theta}[ g_n ], \notag\\
		& g_* = \bigl( Q_{q_*}(s,a) -  Q_{q_*}(s',a') - r_0 + \overline{\rho}(\pi_\theta)  \bigr)\cdot \varphi(s,a), && g_*^e = \EE_{\pi_\theta}[ g_* ],%\notag\\
		%& \overline g_* = \bigl(  Q_{q_*}(s,a) - Q_{q(t)}(s',a') - r_0  + \overline{\rho}(\pi_\theta) \bigr)\cdot \nabla_q Q_{q_0}(s,a), && \overline g_*^e = \EE_{\pi_\theta}[ \overline g_* ],
		\#
		where $q_*$ satisfies that
		\$
		q_* = \Pi_{\cB(0, R)}(q_* - \delta \cdot  g_*^e ). 
		\$
		Here the expectation $\EE_{\pi_\theta}[\cdot]$ is taken following $(s,a)\sim \rho_{\pi_\theta}(\cdot)$, $s'\sim P(\cdot\given s,a)$, $a'\sim \pi_{\theta}(\cdot\given s_1)$, and $r_0 = \cR(s,a)$.  By Algorithm \ref{alg:td}, we have that
		\$
		q(t+1) = \Pi_{\cB(0, R)}\big(q(t) - \delta \cdot  g_t \big). 
		\$
		Then, we have
		\# \label{eq:66001}
		& \EE_{\pi_\theta}\bigl[ \| q(t+1) - q_* \|_2^2 \given q(t) \bigr]\notag\\
		& \qquad = \EE_{\pi_\theta}\bigl[ \| \Pi_{\cB(q_0, R_{\rm c})}(q(t) - \delta \cdot  g_t ) - \Pi_{\cB(q_0, R_{\rm c})}(q_* - \delta \cdot  g_*^e ) \|_2^2 \given q(t) \bigr]\notag\\
		& \qquad \leq \EE_{\pi_\theta}\bigl[ \| (q(t) - \delta \cdot  g_t ) - (q_* - \delta \cdot  g_*^e ) \|_2^2 \given q(t) \bigr]\notag\\
		& \qquad = \underbrace{\|q(t) - q_*\|_2^2}_{\rm (i)}  + 2\delta \cdot  \underbrace{\la q_* - q(t), g_t^e -  g_*^e \ra|}_{\rm (ii)} + \delta^2 \cdot \underbrace{\EE_{\pi_\theta} \bigl[\|g_t - g_*^e\|_2^2\given q(t)\bigr]}_{\rm (iii)}. 
		\#
		We upper bound the second term on the right hand side of \eqref{eq:66001} in the sequel. 
		By the definitions in \eqref{eq:local:linear2}, we further obtain
		\# \label{eq:66002}
		{\rm (ii)} &= \la q_* - q(t),  g_t^e -  g_*^e \ra \notag\\
		&  = \EE_{\pi_\theta} \bigl[ \bigl(( Q_{q(t)}(s,a) -  Q_{q_*}(s,a)) - ( Q_{q(t)}(s',a') - Q_{q_*}(s',a')) \bigr) \cdot \la q_* - q(t), \varphi(s,a) \ra \bigr]\notag\\ 
		&  = \EE_{\pi_\theta} \bigl[ \bigl(( Q_{q(t)}(s,a) -  Q_{q_*}(s,a)) - ( Q_{q(t)}(s',a') -  Q_{q_*}(s',a')) \bigr) \cdot ( Q_{q_*}(s,a) -  Q_{q(t)}(s,a)) \bigr]\notag\\
		&  = \EE_{\pi_\theta} \bigl[ ( Q_{q(t)}(s,a) -  Q_{q_*}(s,a)) \cdot ( Q_{q(t)}(s',a') -  Q_{q_*}(s',a')) \bigr] \notag \\
		&\qquad - \EE_{\pi_\theta} \bigl[ ( Q_{q(t)}(s,a) -  Q_{q_*}(s,a))^2 \bigr].
		\# 
	    By Cauchy-Schwartz inequality, we further have 
		\#  \label{eq:66003}
		&\EE_{\pi_\theta} \bigl[ ( Q_{q(t)}(s,a) -  Q_{q_*}(s,a)) \cdot ( Q_{q(t)}(s',a') -  Q_{q_*}(s',a')) \bigr]  \notag \\
		&\qquad \le  [\EE_{\pi_\theta} \bigl[ ( Q_{q(t)}(s,a) -  Q_{q_*}(s,a))^2 \bigr] ]^{\frac{1}{2}} \cdot  [ \EE_{\pi_\theta} \bigl[ ( Q_{q(t)}(s',a') -  Q_{q_*}(s',a'))^2 \bigr] ]^{\frac{1}{2}} \notag \\
		& \qquad \le \beta_{\pi_\theta} \cdot \EE_{\pi_\theta} \bigl[ ( Q_{q(t)}(s,a) -  Q_{q_*}(s,a))^2 \bigr],
		\#
		where the last inequality holds by Assumption \ref{assumption:contraction}. Combining \eqref{eq:66002} and \eqref{eq:66003}, we have
		\# \label{eq:66004}
		\la q_* - q(t),  g_t^e - g_*^e \ra \le  (\beta_{\pi_\theta} - 1) \EE_{\rho_{\pi_\theta}} \bigl[ ( Q_{q(t)} - Q_{q_*})^2 \bigr]. 
		\#
		Moreover, by Cauchy-Schwarz inequality, we have 
		\# \label{eq:66005}
		{\rm (iii)} &= \EE_{\pi_\theta} \bigl[\|g_t - g_*^e\|_2^2\given q(t)\bigr] \notag\\
		&\le 2\underbrace{\EE_{\pi_\theta} \bigl[\|g_t - g_t^e\|_2^2\given q(t)\bigr]}_{\rm (iii.1)} + 2\underbrace{\|g_t^e - g_*^e\|_2^2}_{\rm (iii.2)}
		\# 
		By the definitions of $g_t$ and $g_t^*$ in \eqref{eq:local:linear2}, we can upper bound Term ${\rm (iii.1)}$ by 
		\# \label{eq:66006}
		{\rm (iii.1)} = \EE_{\pi_\theta} \bigl[\|g_t\|_2^2 - \|g_t^e\|_2^2\given q(t)\bigr] \le \EE_{\pi_\theta} \bigl[\|g_t\|_2^2 \given q(t)\bigr].
		\#
		Meanwhile, by the definition of $g_n$ in \eqref{eq:local:linear2}, we have 
		\$
		\|g_t\|_2^2 &= \bigl( Q_{q(t)}(s,a) -  Q_{q(t)}(s', a') -  r_0  + \overline{\rho}(\pi_\theta) \bigr)^2 \cdot \|\varphi(s,a)\|_2^2 \\
		& \le 4(R + M)^2,
		\$
		where the last inequality follows from the facts that $\|q(t)\|_2^2 \le R$, $\|\varphi(\cdot, \cdot)\|_2 \le$, $|r_0| \le M$ and $|\overline{\rho}(\pi_\theta)| \le M$. Then we upper bound Term ${\rm (iii.2)}$ by 
		\# \label{eq:66007}
		{\rm (iii.2)} &= \|\EE_{\sigma_{\pi_\theta}}[(Q_{q(n)} - Q_{q_*}) \varphi]\|_2^2 \notag\\
		&\le \EE_{\sigma_{\pi_\theta}}[(Q_{q(n)} - Q_{q_*})^2 \|\varphi\|_2^2] \notag\\
		& \le \EE_{\sigma_{\pi_\theta}}[(Q_{q(n)} - Q_{q_*})^2],
		\#
		where the first inequality uses the definitions of $g_t^e$ and $g_*^e$ in \eqref{eq:local:linear2}, the first inequality follows from Cauchy-Schwarz inequality, and the last inequality is obtained by the assumption that $\|\varphi(\cdot, \cdot)\|_2 \le 1$. 
		Combining \eqref{eq:66006} and \eqref{eq:66007}, we have 
		\# \label{eq:66008}
		{\rm (iii)}  \le 8(R + M)^2 + 2\EE_{\sigma_{\pi_\theta}}[(Q_{q(n)} - Q_{q_*})^2].
		\#
		Plugging \eqref{eq:66004} and \eqref{eq:66008} into \eqref{eq:66001}, we have 
		\# \label{eq:66009}
		& \EE_{\pi_\theta}\bigl[ \| q(t+1) - q_* \|_2^2 \given q(t) \bigr]\notag\\
		& \qquad \le \|q(t) - q_*\|_2^2 + 2\delta \cdot (\beta_{\pi_\theta} - 1) \EE_{\rho_{\pi_\theta}} [ ( Q_{q(t)} - Q_{q_*})^2 ] \notag\\
		&\qquad \qquad + \delta^2 \cdot  \bigl( 8(R + M)^2 + 2\EE_{\rho_{\pi_\theta}}  [ ( Q_{q(t)} - Q_{q_*})^2 ] \bigr).
		\#
		Rearranging \eqref{eq:66009} gives that
		\# \label{eq:66010}
		&\bigl(2\delta\beta_{\pi_\theta} - 2\delta^2\bigr) \cdot \EE_{\rho_{\pi_\theta}}  [ ( Q_{q(t)} - Q_{q_*})^2 ] \notag\\
		& \qquad \le \|q(t) - q_*\|_2^2 - \EE_{\pi_\theta}\bigl[ \| q(t+1) - q_* \|_2^2 \given q(t) \bigr] + 8\delta^2(R + M)^2.
		\#
		Here $\beta_{\pi_\theta}$ and $M$ are constants. Telescoping \eqref{eq:66010} and using Jensen's inequality, together with the fact that $\delta = T^{-1/2}$, we obtain
		\$
	    \EE_{\rho_{\pi_\theta}}\bigl[\bigl(Q_{\overline{q}}(s, a) - Q_{q_*}(s, a)\bigr)^2 \bigr] &\le \frac{1}{T}\cdot \sum_{t = 0}^{T - 1} \EE_{\sigma_{\pi_\theta}} \bigl[\bigl(Q_{q(t)} - Q_{q_*} \bigr)^2 \bigr] \\
		& \le \cO( R^2T^{-1/2}),
		\$
		which concludes the proof of Lemma \ref{lemma:q:error:linear}.
\end{proof}

    \section{Supporting Lemma}

    \begin{lemma} \label{lem:constraint:violation}
	Suppose Assumption \ref{assumption:slater} hold. Let $\lambda^*$ be the optimal Lagrangian dual variable and assuming that Let $\lambda' \ge 2\lambda^*$. Suppose that 
	\$
	\rho(\pi^*) - \rho(\pi) + \lambda' \cdot [\Lambda(\pi) - \alpha]_+ \le \delta.
	\$
	Then, it holds that 
	\$
	[\Lambda(\pi) - \alpha]_+ \le 2\delta/\lambda'.
	\$
   \end{lemma}

   \begin{proof}
	See \citet{efroni2020exploration} for a detailed proof.         
   \end{proof}

\section{Implementation Details} 
\label{app:hyper}

\begin{table}[h]
\centering

\begin{tabular}{lc}
\hline
\textbf{Hyperparameter} & \textbf{Value} \\
\hline
Optimizer & Adam \\
Learning rate & 1e-4 \\
Replay Buffer Size (Pendulum) & 1e5 \\
Replay Buffer Size (BipedalWalker) & 1e6 \\
Batch Size & 256 \\
Decay Rate & 0.99 \\
Policy noise & 0.2 \\
Policy noise clipping & (-0.5, 0.5) \\
Initial $\lambda$ & 0.5 \\
\hline
Number of Layers for Actor Network & 2 \\
Number of Layers for Critic Network & 2 \\
Hidden dim & 128 \\
Activation function & ReLU \\
\hline
\end{tabular}
\caption{Hyper-parameters sheet.}
\label{tab:hyperparameters}
\end{table}

\end{appendix}

\end{document}